\documentclass[10pt,letterpaper]{article}
\PassOptionsToPackage{table}{xcolor}
\usepackage[textwidth=7.05in,textheight=9in,centering]{geometry}
\usepackage[T1]{fontenc}
\usepackage{amsmath,amssymb,amsthm}
\usepackage{amsmath,amsfonts,bm}

\def\eqref#1{equation~\ref{#1}}
\def\1{\bm{1}}

\DeclareMathAlphabet{\mathsfit}{\encodingdefault}{\sfdefault}{m}{sl}
\SetMathAlphabet{\mathsfit}{bold}{\encodingdefault}{\sfdefault}{bx}{n}

\newtheorem{mimicxproposition}{Proposition}[section]
\newtheorem{mimicxcorollary}[mimicxproposition]{Corollary}
\usepackage{xcolor,graphicx,booktabs,tabularx,multirow,makecell,array}
\usepackage{tikz,etoolbox}
\usepackage{algorithm,algpseudocode}
\usepackage{times,placeins,needspace,multicol,type1cm}
\usepackage[expansion=false]{microtype}
\usepackage[numbers,sort&compress]{natbib}
\definecolor{FixedGray}{HTML}{7A7F87}
\definecolor{CurriculumBlue}{HTML}{3B78A8}
\definecolor{TaskTeal}{HTML}{23866B}
\definecolor{MimicCoral}{HTML}{D45B4C}
\definecolor{MimicLinkRed}{HTML}{EF3B2D}
\definecolor{SonicPurple}{HTML}{7A6FAF}
\usepackage[breaklinks,colorlinks,allcolors=MimicLinkRed]{hyperref}
\usepackage{caption,subcaption,fancyhdr}
\newcommand{\mimicx}{MimicX}
\newcommand{\hloop}{MimicX-HLoop}
\newcommand{\HalfPlate}{0.485\textwidth}
\graphicspath{{figures/quantitative/}{figures/assets/}{figures/}}
\definecolor{TableGray}{HTML}{EFF0F2}
\definecolor{TableRedLight}{HTML}{FFF6F3}
\definecolor{TableRedMedium}{HTML}{FAE8E4}
\definecolor{TableRedDeep}{HTML}{EDBFB5}
\colorlet{tablehead}{TableRedMedium}
\colorlet{tablegroup}{TableGray}
\colorlet{basebg}{TableGray}
\colorlet{curriculumbg}{TableRedLight}
\colorlet{taskbg}{TableRedMedium}
\colorlet{oursbg}{TableRedDeep}
\colorlet{sonicbg}{TableRedLight}
\definecolor{BadgeAccept}{HTML}{C5DAD2}
\definecolor{BadgeProtect}{HTML}{DED9CF}
\definecolor{BadgeReject}{HTML}{E6BFB3}
\newcommand{\papertablestyle}{%
  \fontsize{7.0}{9.25}\selectfont
  \setlength{\tabcolsep}{3pt}%
  \renewcommand{\arraystretch}{1.13}%
  \renewcommand{\tabularxcolumn}[1]{m{##1}}%
  \setlength{\aboverulesep}{0pt}%
  \setlength{\belowrulesep}{0pt}%
  \setlength{\abovetopsep}{0pt}%
  \setlength{\belowbottomsep}{0pt}%
  \setlength{\cmidrulesep}{0pt}}
\newcommand{\papertablenote}{\fontsize{6.6}{8.35}\selectfont}
\newcommand{\tabletag}[2]{%
  \begingroup
  \ifstrequal{#2}{Accept}{\def\tagfill{BadgeAccept}}{%
    \ifstrequal{#2}{Reject}{\def\tagfill{BadgeReject}}%
      {\def\tagfill{BadgeProtect}}}%
  \tikz[baseline=(tag.base)]\node[inner xsep=1.4pt,inner ysep=.3pt,
    rounded corners=2pt,fill=\tagfill,text=black,draw=black,line width=.3pt,dash pattern=on 1.2pt off .8pt] (tag)
    {\fontsize{5.8}{6.5}\selectfont\strut #2};%
  \endgroup}
\newcommand{\pairtablecaption}[2]{%
  \caption{#1}\label{#2}}

\newsavebox{\evidenceartbox}
\newcommand{\evidenceart}[7]{%
  \sbox{\evidenceartbox}{#2}%
  \noindent\begin{minipage}[t][\dimexpr\ht\evidenceartbox+\dp\evidenceartbox\relax][s]{\dimexpr\textwidth-#1\textwidth-0.025\textwidth\relax}%
    \vspace{0pt}#3\par
    \vfill
    \captionsetup{type=figure,format=plain,justification=justified,singlelinecheck=false,skip=3pt}%
    \caption{\textbf{#5} #6}\label{#4}#7%
  \end{minipage}\hfill%
  \begin{minipage}[t]{#1\textwidth}%
    \vspace{0pt}\usebox{\evidenceartbox}%
  \end{minipage}\par}

\newcommand{\evidencefigure}[3]{%
  \begin{figure}[!htbp]
    \centering#1
    \caption{#3}\label{#2}
  \end{figure}}

\colorlet{MimicHeading}{MimicCoral!82!black}
\newcommand{\MimicWebsiteIcon}{%
  \makebox[12pt][c]{\raisebox{-1.5pt}{%
    \includegraphics[height=10.2pt]{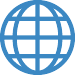}}}}
\newcommand{\MimicGitHubIcon}{%
  \makebox[12pt][c]{\raisebox{-1.6pt}{%
    \includegraphics[height=10.4pt]{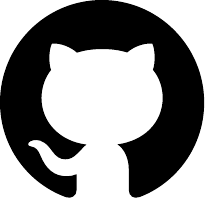}}}}
\newcommand{\MimicHFIcon}{%
  \makebox[12pt][c]{\raisebox{-1.8pt}{%
    \includegraphics[width=12pt]{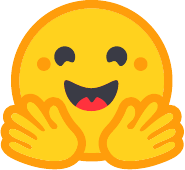}}}}

\newcommand{\MimicXResourceLink}[3]{%
  #1\hspace{5pt}\href{#3}{\textbf{#2}}}
\newcommand{\MimicXResourceLinks}{%
  \par\vspace{8pt}\begingroup\normalfont\normalsize\urlstyle{same}%
  \noindent\begin{tabular*}{\linewidth}{@{\extracolsep{\fill}}llll@{}}
    \MimicXResourceLink{\MimicWebsiteIcon}{Project Page}{https://nebulis-lab.com/MimicX} &
    \MimicXResourceLink{\MimicGitHubIcon}{GitHub Code}{https://github.com/NEBULIS-Lab/MimicX} &
    \MimicXResourceLink{\MimicHFIcon}{Hugging Face Models}{https://huggingface.co/Shuaijun/MimicX-Policies} &
    \MimicXResourceLink{\MimicHFIcon}{Hugging Face Dataset}{https://huggingface.co/datasets/Shuaijun/MimicX-Assets}
  \end{tabular*}\par\endgroup}
\newcommand{\MimicXPreprintTitle}{%
  \textcolor{MimicCoral}{MimicX}: Policy-in-the-Loop Supervision Refinement\\
  for Video-Driven Humanoid Motion Tracking}
\fancypagestyle{mimicxfirst}{%
  \fancyhf{}\fancyfoot[C]{\normalfont\normalsize\thepage}%
  }
\makeatletter
\renewcommand{\@maketitle}{%
  \thispagestyle{mimicxfirst}%
  \vbox{\hsize\textwidth\parskip=0pt\centering
    {\color{MimicHeading}\hrule height .8pt}\vskip 11pt
    {\normalfont\fontsize{17}{20}\selectfont\bfseries\upshape\@title\par}%
    \vskip 9pt{\color{MimicCoral!55}\hrule height .35pt}\vskip 9pt
    {\normalfont\normalsize\@author\par}\vskip 9pt}}
\patchcmd{\section}{\normalfont\Large\bfseries}
  {\normalfont\Large\bfseries\color{MimicHeading}}{}{}
\pretocmd{\section}{\Needspace{5\baselineskip}}{}{}
\patchcmd{\subsection}{\normalfont\large\bfseries}
  {\normalfont\large\bfseries\color{MimicHeading}}{}{}
\patchcmd{\subsubsection}{\normalfont\normalsize\bfseries}
  {\normalfont\normalsize\bfseries\color{MimicHeading}}{}{}
\makeatother
\newsavebox{\MimicAbstractBox}
\renewenvironment{abstract}{%
  \par\vspace{7pt}\begin{lrbox}{\MimicAbstractBox}%
  \begin{minipage}{\dimexpr\textwidth-24.7pt\relax}%
  \normalfont\normalsize
  {\centering\large\bfseries\color{MimicHeading}Abstract\par}\vspace{5pt}%
}{%
  \MimicXResourceLinks
  \par\end{minipage}\end{lrbox}%
  \noindent\begin{tikzpicture}
    \node[draw=MimicCoral!35,fill=MimicCoral!4,line width=.35pt,
      rounded corners=3pt,inner sep=12pt]{\usebox{\MimicAbstractBox}};
  \end{tikzpicture}\par\clearpage}
\hypersetup{pdfauthor={Shuaijun Liu, Chenglong Zhang, Xuhao Liu, Feiyang You, Yifan Liao, Shuyang Hao, Chaozhe Zhang, Chengyu Wu, Zhen Sun, Ningxin Su}}
\author{\parbox[t]{\textwidth}{%
\centering\normalfont\normalsize
\textbf{Shuaijun Liu, Chenglong Zhang, Xuhao Liu, Feiyang You, Yifan Liao}\\[2pt]
\textbf{Shuyang Hao, Chaozhe Zhang, Chengyu Wu, Zhen Sun, Ningxin Su\textsuperscript{*}}\\[6pt]
The Hong Kong University of Science and Technology (Guangzhou)\\[3pt]
\textsuperscript{*}Corresponding author:
\href{mailto:ningxinsu@hkust-gz.edu.cn}{ningxinsu@hkust-gz.edu.cn}}}

\title{\MimicXPreprintTitle}
\hypersetup{pdftitle={MimicX: Policy-in-the-Loop Supervision Refinement for Video-Driven Humanoid Motion Tracking}}

\begin{document}
\maketitle
\noindent\begin{minipage}{\textwidth}
\includegraphics[width=\linewidth]{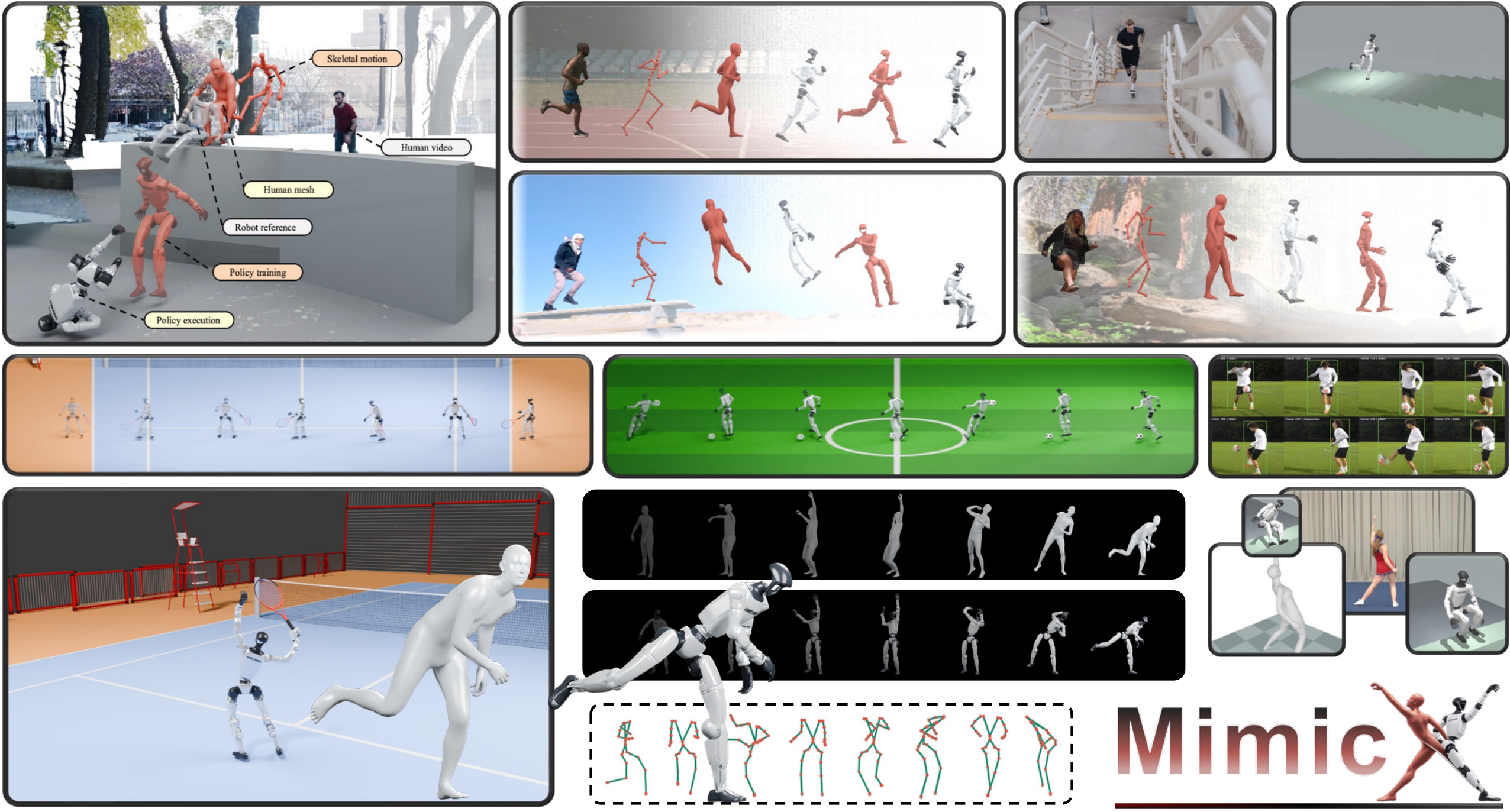}\par
\captionsetup{hypcap=false}
\captionof{figure}{\textbf{From human video to humanoid motion.}
Input videos, human reconstruction, robot references, and learned execution
span athletic skills and collision-scene tasks.}
\label{fig:teaser}
\end{minipage}\par
\begin{abstract}
Human videos provide motion references for humanoid control, yet tracking may fail
at brief transitions such as takeoff, landing, and support changes. These
failures reveal where fixed objectives and practice distributions need revision. We
present \mbox{\mimicx{}}, a policy-in-the-loop framework that uses execution feedback
to refine supervision and select the policy from which training continues.
Rollout diagnosis identifies difficult motion intervals and affected body regions,
coordinating tracking rewards with reset-state sampling. Repeated verification evaluates
intermediate checkpoints for uninterrupted execution progress and both world-space and
root-local tracking fidelity. Across twelve motion tasks with matched warmstarts and
training budgets, \mbox{\mimicx{}} improves held-out completion from 83.3\% to 91.7\%
over an existing motion-tracking baseline. At matched completion, it further
reduces mean world-space body tracking error by 30.8\% on Broad Jump and 18.4\% on
Heel-and-Toe Dance. On five video-driven collision-scene tasks, full-motion completion
increases from 60\% to 80\%. These results demonstrate how execution feedback can guide
targeted supervision refinement and verified policy continuation for more reliable
humanoid motion tracking.

\end{abstract}
\section{Introduction}
\label{sec:main-introduction}

Human video provides whole-body motion without requiring robot demonstrations.
World-grounded reconstruction and retargeting translate it into humanoid
motion references~\citep{gvhmr2024,gmr2025}, and physics-based tracking learns
to execute them~\citep{beyondmimic2025,sonic2025}. In dynamic motions,
loss of support at takeoff, landing or an obstacle can terminate an otherwise
well-tracked rollout. Earlier frames can dominate a summary error. A
policy can therefore track the approach closely without ever attempting
the subsequent phases of the demonstrated motion.

The question is which supervision helps the robot complete the transition
where its execution fails. Tracking objectives and reset-state sampling
interact: better global alignment may worsen articulation, while higher
reward may still accompany earlier termination. Crossing a difficult
transition may require both more practice and different tracking priorities.
Execution feedback must
determine both what to practice next and which checkpoint to retain.

We introduce \mimicx{}, a policy-in-the-loop supervision framework
(Figure~\ref{fig:main-overview}). Its central mechanism couples
execution-conditioned supervision with protected continuation. Diagnosis
targets the failed transition, the affected bodies and their dominant
tracking errors. Repeated rollouts assess execution progress and tracking
fidelity in world and root coordinates, selecting a policy for the next
stage. Feedback thus changes both the learning objective and the policy
state from which optimization proceeds.
Across twelve prepared motions with shared warmstarts and matched
continuation budgets, MimicX improves execution over the BeyondMimic baseline and
precision over static supervision search at similar measured cost.
Held-out rollouts assess the resulting policies alongside released SONIC.
Five paired scene tasks starting from video evaluate the complete
video-to-policy loop in collision geometry.

Our contributions are \textbf{(1)} execution-conditioned supervision coupling
task-critical tracking with practice at difficult transitions;
\textbf{(2)} protected policy continuation that verifies execution progress
and tracking fidelity before retaining a checkpoint;
\textbf{(3)} a video-to-policy loop for continuous whole-body skills, with
paired execution in collision scenes; and \textbf{(4)} \hloop{}, a heterogeneous
schedule that overlaps rollout and diagnosis while preserving selection.

\begin{figure}[!htbp]
\evidenceart{0.62}%
{\includegraphics[width=0.62\textwidth]{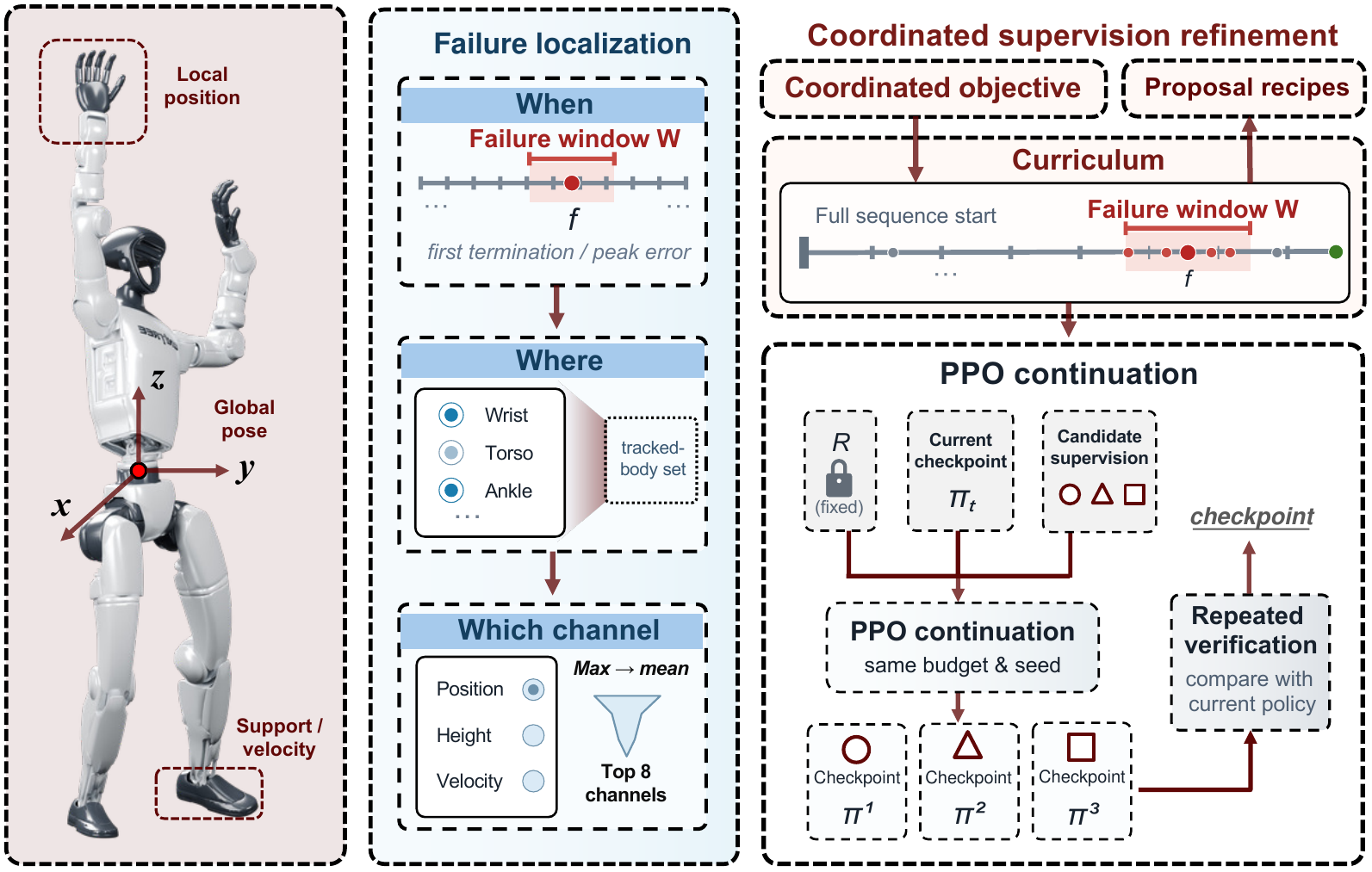}}%
{\includegraphics[width=\linewidth]{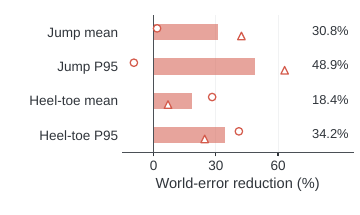}}%
{fig:main-overview}
{Execution feedback coordinates supervision.}
{\textbf{Left:} world-error reductions over Static Search on two dynamic motions. Symbols show
two paired training seeds with three repeats each. \textbf{Right:} diagnosis
sets tracking and practice for one training trajectory per stage.
Verification selects the next training parent.}
{\label{fig:refinement}}
\end{figure}

\section{Related Work}
\label{sec:main-related-work}

\paragraph{Video reconstruction and motion tracking.}
SMPL-X, GVHMR and GMR provide human geometry, world-grounded motion and robot
references~\citep{smplx2019,gvhmr2024,gmr2025}. Physics-based imitation and
motion priors support tracking, recovery and partial-motion control
\citep{deepmimic2018,amp2021,phc2023,maskedmimic2024}. BeyondMimic emphasizes
difficult segments and guided diffusion; SONIC scales generalist tracking
\citep{beyondmimic2025,sonic2025}. MeshMimic couples reconstructed scene geometry with contact-aware motion
retargeting~\citep{meshmimic2026}. MimicX uses observed execution failures
to revise tracking objectives and practice states, then verifies candidate
policies through repeated rollouts.

\paragraph{Adaptive whole-body execution.}
Contact-aware tracking, timing adaptation and dynamics alignment motivate
coordinated local, global and dynamic supervision
\citep{kungfubot2025,adamimic2025,asap2025}. Recent methods separate motor
learning from physical refinement, combine references with task goals, and
compose long-horizon or end-effector skills
\citep{omnixtreme2026,referencegoal2026,parkour2026,hiwet2026}.
MimicX adapts stage supervision and the continuation parent from execution
evidence, with a joint criterion for accepting updates.
Static Search and BeyondMimic provide matched-budget
controls for this decision process. Appendix~\ref{sec:extended-related-work}
compares these methods.

\begin{figure}[!htbp]
\evidenceart{0.58}%
{\includegraphics[width=0.58\textwidth]{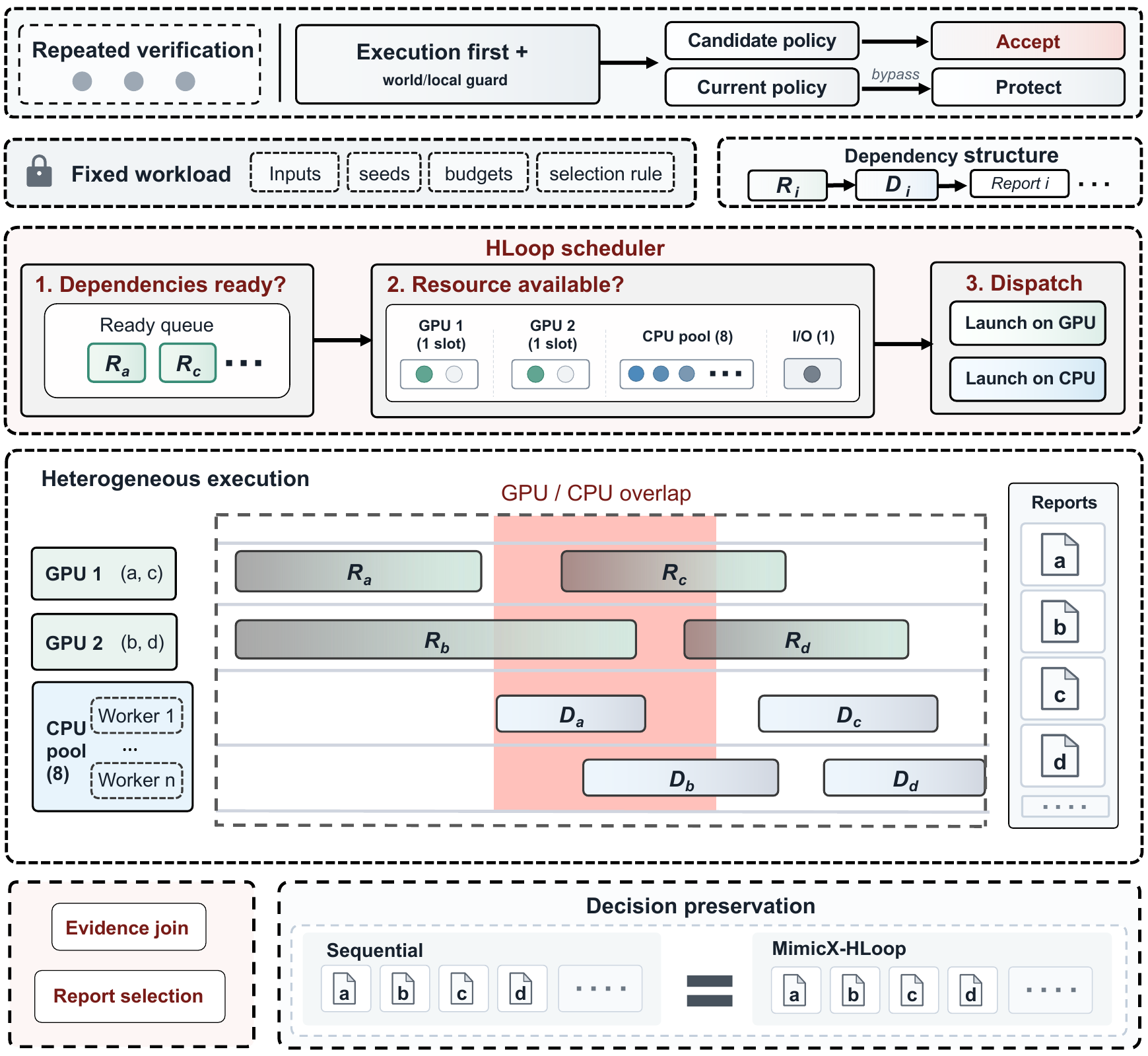}}%
{\includegraphics[width=\linewidth]{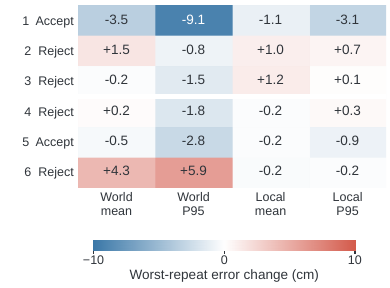}}%
{fig:verification}
{Verified updates with heterogeneous execution.}
{\textbf{Left:} six Broad Jump checkpoints across two training stages. Cells show the
worst signed error change across three paired repeats; negative is better.
World and local checks admit checkpoints 1 and 5 and reject regressions.
Comparison uses common valid intervals and checks the retained policy and
original warmstart. \textbf{Right:} the gate checks execution progress and
tracking fidelity. HLoop overlaps GPU rollout and CPU diagnosis;
selection waits for all evaluation reports.}
{}
\end{figure}

\section{Policy-in-the-Loop Supervision Refinement}
\label{sec:main-method}

GVHMR recovers SMPL-X motion from video; GMR retargets it to the
humanoid~\citep{gvhmr2024,smplx2019,gmr2025}, producing phase-indexed pose
and velocity targets $R$. Prepared motions share this interface.
A proximal policy optimization (PPO) tracker~\citep{ppo2017} learns joint-position
commands in MuJoCo/MjLab~\citep{mujoco2012,mjlab2026}; MimicX refines its supervision
(Figures~\ref{fig:main-overview} and~\ref{fig:registered_stages}).

\subsection{Execution-Conditioned Supervision}

At stage $k$, rollouts of the retained policy $\theta_k$ produce diagnostics
$D_k$: completion, difficult intervals, affected bodies and tracking errors.
The proposal rule combines these with selection history $\mathcal H_k$:
\begin{equation}
 U_k=(W_k,C_k)=\mathcal P(D_k,\mathcal H_k).
 \label{eq:main-supervision}
\end{equation}
Here $W_k$ sets tracking rewards and $C_k$ sets reset sampling. Candidates
pair failure windows with dominant tracking errors; history excludes tested
settings. The first failure identifies the transition to revisit, and
body and error-channel rankings guide reward adjustments. After uninterrupted
execution, practice shifts to full-motion precision.

The tracking objective combines local articulation, global trajectory and
dynamic consistency:
\begin{equation}
\begin{aligned}
 r_t(W_k)={}&\sum_{\ell\in\mathcal L}w_\ell\phi_\ell(t)
 +\sum_{g\in\mathcal G}w_g\phi_g(t)\\
 &+\sum_{d\in\mathcal D}w_d\phi_d(t)-\mathcal P_t .
\end{aligned}
\label{eq:main-objective}
\end{equation}
Local terms track limbs and end effectors; global terms align the torso
anchor and core. Dynamic terms track velocities and reference-conditioned foot motion.
$\mathcal P_t$ penalizes action changes, joint limits and self-collision.
Positional rewards use exponential squared error with bounded weights and
scales, retaining whole-body supervision.

The practice distribution mixes motion-start initialization and
reference-state initialization within a window $\mathcal W_k$:
\begin{equation}
 \mu_k=(1-\rho_k)\mu_{\rm start}
       +\rho_k\mu_{\mathcal W_k}.
 \label{eq:main-curriculum}
\end{equation}
Reset sampling changes the motion phases practiced, while rewards
emphasize different tracking errors
(Appendix~\ref{sec:theory-supervision}).

\subsection{Protected Policy Continuation}

In the matched study, each stage starts from the verified incumbent under
fixed $U_k$, saving three checkpoints along one PPO trajectory:
\begin{equation}
 \widetilde\theta_{k,j}
 =\mathcal A_{b_j}(\theta_k;R,U_k,\xi_k),
 \quad b_j\in\{500,1000,1500\}.
 \label{eq:main-continuation}
\end{equation}
$\mathcal A$ is PPO continuation and $\xi_k$ training randomness. Ordered
verification selects an intermediate checkpoint or retains the incumbent
for the next stage. Two stages consume 3,000 updates with six checkpoint
opportunities.

Repeated global and local tradeoffs trigger native-objective replay in stage two
only if the incumbent is unchanged; otherwise another supervision update is proposed.
Algorithm~\ref{alg:main-mimicx} gives the update rule.

\subsection{Joint Repeated-Execution Verification}

On the same reference and selection seed $s$, verification records two
execution horizons $F_s(\theta)$: valid tracker steps and valid steps before a
common physical failure. World-space body error measures trajectory alignment.
Root-local forward kinematics (FK) fixes the root pose to isolate articulation.
Temporal means and 95th percentiles (P95) define four checks $e_{s,c}(\theta)$
over the shared valid pre-reset interval, excluding initialization.
Joint verification requires componentwise
\begin{equation}
\begin{gathered}
 F_s(\widetilde\theta)\ge F_s(\theta_k),\quad\forall s,\\
 e_{s,c}(\widetilde\theta)\le e_{s,c}(\theta_k)+\varepsilon,\quad\forall s,c .
\end{gathered}
\label{eq:main-joint-guard}
\end{equation}
Tracker validity conditions also apply. A second joint check protects the
initial warmstart. Passing policies are accepted for a longer valid frontier
or a reduction exceeding $2\%$ in summed world and local mean errors:
\begin{equation}
 \exists s:\ F_s(\widetilde\theta)\ne F_s(\theta_k)
 \quad\text{or}\quad
 E(\widetilde\theta)<0.98E(\theta_k).
 \label{eq:main-progress}
\end{equation}
$E$ sums distances in metres over selection repeats. Acceptance prioritizes
execution while verifying both world-space and root-local fidelity.
If no checkpoint passes, the incumbent remains. Appendix~\ref{sec:theory-gate}
derives these properties; held-out rollouts evaluate the policy.

\paragraph{Heterogeneous feedback.}
\hloop{} schedules rollout, diagnosis and selection by their dependencies.
GPU simulations overlap CPU diagnosis of completed rollouts on separate
resources without a stage barrier, while selection
waits for the complete ordered evidence (Figure~\ref{fig:verification}). With
fixed node outputs, dependency-respecting schedules yield the same decision;
Appendix~\ref{sec:theory-hloop} gives the invariance argument and runtime bound.

\begin{algorithm*}[tp]
\caption{Execution-conditioned supervision with protected continuation}
\label{alg:main-mimicx}\label{alg:mimicx}
\begin{algorithmic}[1]
\Require Reference $R$, warmstart $\theta_0$, fixed selection seeds $S$, two stage budgets of 1,500 PPO updates
\State $\theta\gets\theta_0$; $E\gets\operatorname{Evaluate}(\theta,R,S)$; $E_{\rm donor}\gets E$; $\mathcal H\gets\emptyset$
\For{$k=1,2$}
  \State $D\gets\operatorname{Diagnose}(E)$; $U_k\gets\operatorname{Propose}(D,\mathcal H)$
  \State Train once from protected parent $\theta$ under $U_k$, saving $\widetilde\theta_{k,j}$ at $b_j\in\{500,1000,1500\}$
  \For{checkpoints $j=1,2,3$ in order}
    \State $\widetilde E\gets\operatorname{Evaluate}(\widetilde\theta_{k,j},R,S)$
    \State $g\gets\operatorname{NativeSafe}(E,\widetilde E)\wedge\operatorname{JointGuard}(E,\widetilde E)$
    \State $g\gets g\wedge\operatorname{JointGuard}(E_{\rm donor},\widetilde E)\wedge\operatorname{Progress}(E,\widetilde E)$
    \State Append the proposal, checkpoint evidence and decision $g$ to $\mathcal H$
    \If{$g$}
      \State $\theta\gets\widetilde\theta_{k,j}$; $E\gets\widetilde E$
    \EndIf
  \EndFor
\EndFor
\State Freeze $\theta$; evaluate on disjoint nominal and perturbed seeds
\State \Return $\theta$
\end{algorithmic}
\end{algorithm*}

\section{Experimental Setup}
\label{sec:main-setup}

\paragraph{Matched motion study.}
Twelve references cover locomotion, articulation, dance, jumping and recovery:
ten from BONES, one AMASS and one LaFAN1~\citep{amass2019,lafan2020}.
The three per-motion PPO methods share a 29-DoF Unitree G1, inputs and
a 3,000-update warmstart for each motion and seed.
Continuation uses 3,000 updates, 4,096 environments, 24 steps per
update and six checkpoint opportunities. Each of 24 endpoints per
method is evaluated with three held-out initializations in each of two
conditions, for 144 executions. Evaluation seeds are disjoint from the
selection seeds. Appendix~\ref{sec:current-protocol} details the protocol.

\paragraph{Comparisons.}
\emph{BeyondMimic (MjLab)} continues its tracking objective for 3,000 updates;
\emph{Static Search} trains two predetermined 1,500-update branches from the
warmstart; \emph{MimicX} adapts supervision and continues from its verified checkpoint.
All share the joint checkpoint-verification rule and the same optimization and selection budgets.
\emph{Released SONIC}~\citep{sonic2025} executes the same references without
per-task PPO training: three repeats per condition give 72 executions.

\paragraph{Execution, fidelity and cost.}
Completion requires uninterrupted execution of the scored interval; neither
the starting pose nor states recorded after a reset extend the horizon.
World-space body and root-local FK errors share a G1 skeleton and source clock;
task summaries average per-rollout means and temporal P95 over all predefined repeats.
Full-horizon errors require a completed cell; completion retains every task.
Prefix metrics compare common valid intervals. Joint initialization noise
is $\pm0.005$\,rad nominally and $\pm0.01$\,rad under perturbation.
Assigned GPU phase time includes training, selection and evaluation; shared
warmstart cost is counted once.

\paragraph{Video and feedback studies.}
Five collision-scene videos use paired 3,000-update continuations from shared
warmstarts, one training seed and three evaluation repeats. The reference
and collision geometry stay fixed within each pair
(Figure~\ref{fig:registered_stages}). The four-video component study uses
three continuation seeds and three evaluations, with reference and budget
settings in Appendix~\ref{sec:setup}. HLoop compares three schedules over
five paired repetitions of eight rollouts, their diagnoses and one selection,
with identical policies and decision criteria.

\begin{table*}[t]\centering
\caption{\textbf{Full-horizon world tracking for motions with defined errors.} N and P: nominal and perturbed starts; mean and temporal P95 in cm. Completion includes all twelve motions.}\label{tab:full-main}
\begingroup\papertablestyle
\begin{tabularx}{\textwidth}{>{\raggedright\arraybackslash}p{.17\textwidth}*{8}{>{\centering\arraybackslash}X}}\toprule
\rowcolor{tablehead}
\textbf{Task} & \multicolumn{2}{c}{\textbf{BeyondMimic (MjLab)}} & \multicolumn{2}{c}{\textbf{Static Search}} & \multicolumn{2}{c}{\cellcolor{oursbg}\textbf{MimicX}} & \multicolumn{2}{c}{\textbf{Released SONIC}} \\
\rowcolor{tablehead}
 & N & P & N & P & N & P & N & P \\
\midrule
\rowcolor{tablegroup}
\multicolumn{9}{l}{\textit{BONES}} \\
Turn and Jog & \textbf{8.49} / \textbf{13.80} & \textbf{4.89} / \textbf{9.15} & 10.23 / 18.48 & 5.82 / 12.11 & \cellcolor{oursbg}9.38 / 14.40 & \cellcolor{oursbg}5.26 / 9.31 & 12.69 / 18.07 & 12.53 / 17.99 \\
Forward Jog & 4.94 / 7.06 & 5.09 / 7.23 & \textbf{4.60} / \textbf{6.59} & 5.78 / 8.75 & \cellcolor{oursbg}5.42 / 7.52 & \cellcolor{oursbg}\textbf{4.73} / \textbf{6.65} & 4.87 / 7.46 & 4.84 / 7.26 \\
Rubbing Eyes & 1.64 / \textbf{2.19} & 1.43 / \textbf{1.82} & 1.64 / \textbf{2.19} & 1.43 / 1.83 & \cellcolor{oursbg}\textbf{1.63} / \textbf{2.19} & \cellcolor{oursbg}\textbf{1.42} / \textbf{1.82} & 2.77 / 3.76 & 2.85 / 4.23 \\
Praying/Salute & 3.46 / 5.90 & 3.56 / 5.65 & \textbf{3.27} / \textbf{5.14} & \textbf{3.31} / \textbf{4.96} & \cellcolor{oursbg}3.34 / 5.71 & \cellcolor{oursbg}3.36 / 5.39 & 7.63 / 11.05 & 6.73 / 9.94 \\
Heel-and-Toe Dance & 7.31 / 12.50 & 6.57 / 9.86 & 7.91 / 16.71 & 7.66 / 13.21 & \cellcolor{oursbg}\textbf{6.45} / \textbf{10.99} & \cellcolor{oursbg}\textbf{5.55} / \textbf{8.85} & 22.19 / 38.94 & 18.41 / 38.67 \\
Crouch Transition & 2.67 / 3.62 & 2.91 / 3.83 & 2.11 / 3.04 & 2.14 / \textbf{2.96} & \cellcolor{oursbg}\textbf{2.06} / \textbf{2.90} & \cellcolor{oursbg}\textbf{2.08} / 2.97 & 10.40 / 12.94 & 11.09 / 13.70 \\
Pointing Down & \textbf{2.96} / \textbf{4.94} & \textbf{3.05} / \textbf{5.24} & 3.05 / 5.47 & 3.15 / 5.91 & \cellcolor{oursbg}3.07 / 5.42 & \cellcolor{oursbg}3.14 / 5.90 & 3.14 / 4.40 & 3.31 / 4.73 \\
Dodge Backward & 5.93 / 9.31 & 5.17 / 9.01 & 7.66 / 11.54 & 6.69 / 10.93 & \cellcolor{oursbg}\textbf{5.78} / \textbf{9.09} & \cellcolor{oursbg}\textbf{4.90} / \textbf{8.81} & 25.55 / 36.33 & 25.04 / 35.60 \\
Injured-Torso Walk & 2.82 / 4.39 & 3.02 / 4.63 & 2.71 / 4.19 & 2.96 / 4.49 & \cellcolor{oursbg}\textbf{2.70} / \textbf{4.08} & \cellcolor{oursbg}\textbf{2.80} / \textbf{4.27} & 6.96 / 12.05 & 7.66 / 13.50 \\
Broad Jump & NC & NC & 8.04 / 21.78 & 8.00 / 19.33 & \cellcolor{oursbg}\textbf{5.56} / \textbf{11.13} & \cellcolor{oursbg}\textbf{5.39} / \textbf{10.23} & NC & NC \\
\rowcolor{tablegroup}
\multicolumn{9}{l}{\textit{AMASS}} \\
Side Step & 7.11 / 8.37 & 7.18 / 8.65 & \textbf{5.26} / \textbf{7.53} & \textbf{5.21} / \textbf{7.03} & \cellcolor{oursbg}7.11 / 8.37 & \cellcolor{oursbg}7.18 / 8.65 & 18.80 / 39.10 & 19.13 / 40.38 \\
\midrule
\rowcolor{tablegroup}
\textit{Complete executions} & 60/72 & 60/72 & \textbf{66/72} & \textbf{66/72} & \textbf{66/72} & \textbf{66/72} & 30/36 & 30/36 \\
\textit{Continuation GPU-h} & \multicolumn{2}{c}{42.50} & \multicolumn{2}{c}{42.91} & \multicolumn{2}{c}{42.71} & \multicolumn{2}{c}{Execution only} \\
\bottomrule\end{tabularx}
\par\smallskip\begin{minipage}{\textwidth}\papertablenote Each native cell averages two training seeds and three held-out executions; SONIC uses three executions of its released controller. P95 is computed within each rollout before averaging. NC: at least one required execution is incomplete. Each native method uses 72k continuation updates; GPU-h excludes the shared warmstart pool. Bold: best native error or completion, including ties at displayed precision. Root-local errors and full sampling details: Tables~\ref{tab:current-full-nominal} and~\ref{tab:current-full-perturbed}.
\par\smallskip\hrule height .4pt\end{minipage}
\endgroup\end{table*}

\section{Results}
\label{sec:main-experiments}

\subsection{Matched Motion Tracking}

On matched tests, MimicX achieved 91.7\% completion versus 83.3\% for
BeyondMimic (MjLab), completing eleven of twelve motions rather than ten
(Table~\ref{tab:full-main}). The gain held under both nominal and perturbed
initialization: both MimicX training seeds completed Broad Jump in every
held-out repeat. Static Search matched MimicX's completion count,
and released SONIC completed 60/72 executions.

\medskip
\noindent\begin{minipage}{\linewidth}
\noindent\begin{minipage}[t]{\HalfPlate}
\vspace{0pt}
\includegraphics[width=\linewidth]{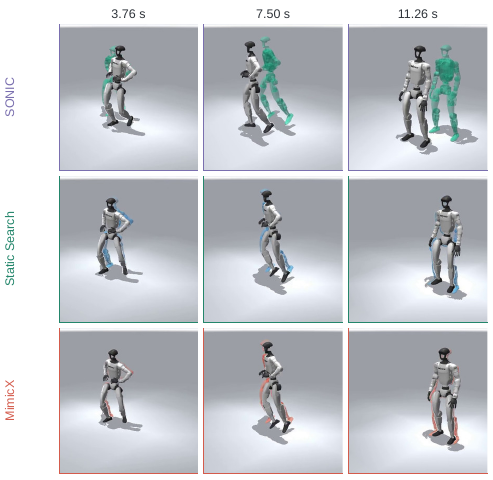}
\end{minipage}\hfill
\begin{minipage}[t]{\HalfPlate}
\vspace{0pt}
\captionsetup{hypcap=false}
\captionof{table}{\textbf{World tracking after complete execution.}
Reductions relative to Static Search (\%).}
\label{tab:dynamic-precision}
\begingroup\papertablestyle
\begin{tabularx}{\linewidth}{@{}>{\raggedright\arraybackslash}Xlcc@{}}
\toprule
\rowcolor{tablehead}
\textbf{Motion} & \textbf{Condition} & \textbf{Mean} & \textbf{P95} \\
\midrule
Broad Jump & Nominal & 30.8 & \cellcolor{oursbg}48.9 \\
 & Perturbed & 32.6 & \cellcolor{oursbg}47.1 \\
\rowcolor{curriculumbg}
Heel-and-Toe & Nominal & 18.4 & \cellcolor{oursbg}34.2 \\
\rowcolor{curriculumbg}
 & Perturbed & 27.5 & \cellcolor{oursbg}33.1 \\
\bottomrule
\end{tabularx}
\par\smallskip
{\papertablenote Two training seeds and three held-out repeats per
condition. Both methods completed every repeat in these cells.\par}
\smallskip\hrule height .4pt
\endgroup

\smallskip
At matched completion and training budgets, MimicX improved world-space
tracking on Broad Jump and Heel-and-Toe under both initialization conditions.
On Heel-and-Toe, root-local mean and P95 errors fell by 5.1\% and 9.3\%
nominally, and by 3.3\% and 8.9\% under perturbation. The recorded states in
Figure~\ref{fig:heel-qualitative} show the corresponding world alignment.

\smallskip
At matched budgets, MimicX had the lowest mean world-space error among
the three PPO methods on six motions nominally and seven under perturbation.
Tables~\ref{tab:current-full-nominal} and~\ref{tab:current-full-perturbed}
detail world-space and root-local errors.
\end{minipage}\par\vspace{4pt}
\captionsetup{hypcap=false}
\captionof{figure}{\textbf{World alignment on Heel-and-Toe.}
Released SONIC, Static Search and MimicX follow the same reference at three
matched times. Solid robots show execution; translucent robots show the
reference. Native training seed 101, evaluation seed 21001.}
\label{fig:heel-qualitative}

\end{minipage}\par\medskip

\begin{figure}[H]
  \centering
  \includegraphics[width=\linewidth]{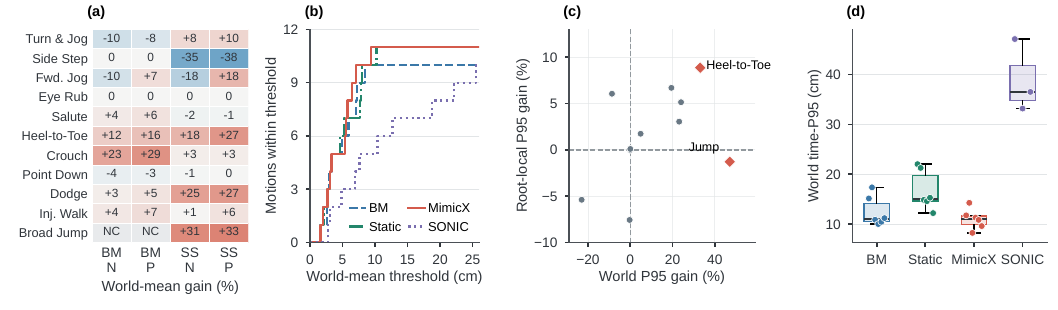}
  \caption{\textbf{Tracking quality across tasks and conditions.}
  (a) Relative world-mean error reduction over BeyondMimic (BM) and Static
  Search (SS) for motions with defined comparisons, nominal (N) and
  perturbed (P); NC indicates incomplete
  execution. (b) Nominal motions completed with mean error at or below each
  threshold (out of twelve). (c) Perturbed world and root-local P95
  reductions versus Static Search on eleven completed motions.
  (d) Heel-and-Toe nominal world P95 (six rollouts per native method;
  three for released SONIC); boxes span the middle half.}
  \label{fig:current-quality}
\end{figure}

Under perturbed initialization, Broad Jump world P95 fell from 19.33\,cm
for Static Search to 10.23\,cm for MimicX, with both methods completing
every run. In Heel-and-Toe, the world and root-local gains indicate closer
trajectory and articulation tracking through the same sequence.

\clearpage
\subsection{Protected Continuation in Practice}

Broad Jump initially reached 39--40 of 336 steps across selection repeats.
Diagnosis concentrated practice on frames 0--82 (reset probability 0.45)
and coordinated core-position and velocity tracking. The 500-update
checkpoint completed every selection repeat and became the second-stage
parent. Once the transition was executable, full-motion practice targeted
lower-limb position and height alongside arm and velocity tracking.
The stage-two 1,000-update checkpoint was retained after joint world and local
verification (Figures~\ref{fig:verification} and~\ref{fig:current-mechanism}).

\begin{figure}[H]
\centering
\includegraphics[width=\linewidth]{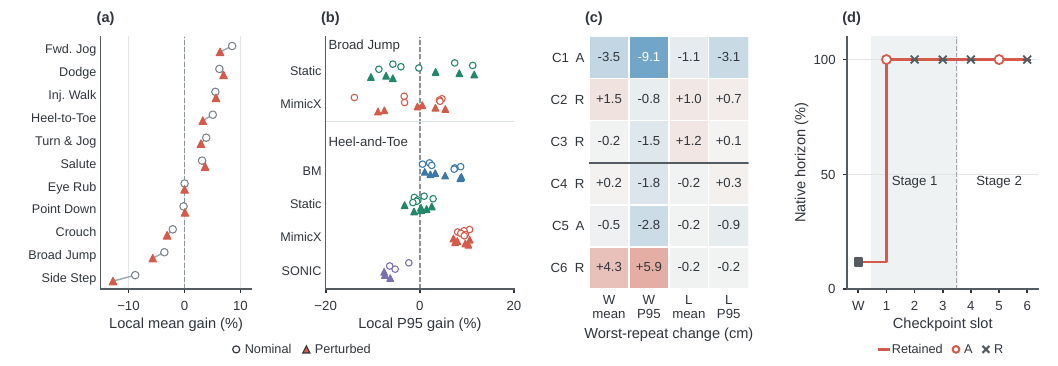}
\caption{\textbf{Local tracking and protected continuation.}
(a) Root-local mean error reduction over Static Search on eleven completed
motions. (b) Per-rollout root-local P95 relative to the Static Search mean
for the same motion and condition; BeyondMimic and SONIC did not complete
Broad Jump. Positive values in (a,b) indicate lower error.
(c) Broad Jump checkpoints (seed 101): worst of three paired signed
candidate-minus-incumbent changes in world (W) and local (L) errors on
common valid intervals; negative values favor the candidate.
(d) Native frontiers; whiskers show retained ranges across three repeats.
W: warmstart; A: accepted; R: rejected.}
\label{fig:current-local-fidelity}\label{fig:current-mechanism}
\end{figure}

The same gate can retain the incumbent despite a higher-reward alternative.
Selection considers repeated execution rather than the reward used to train
a candidate. The four-video component study accepted all six
Tennis and Football refinements and protected all six Dance and Kung Fu
parents. In a separate repeated Kung Fu check, both higher-reward proposals
shortened execution. The gate retained the incumbent with the longer
frontier (Tables~\ref{tab:gate} and~\ref{tab:rejection}).

\begin{table}[H]
\begin{minipage}[t]{\HalfPlate}\vspace{0pt}
% Generated by scripts/rebuild_table_layouts.py
\centering\pairtablecaption{\textbf{Repeated verification.} Decisions for four tasks and three continuation seeds yield six accepted refinements and six protected incumbent policies.}{tab:gate}
\begingroup\papertablestyle
\begin{tabularx}{\linewidth}{l*{3}{>{\centering\arraybackslash}X}}
\toprule
\rowcolor{tablehead}
Task & \makecell{Seed\\101} & \makecell{Seed\\202} & \makecell{Seed\\303} \\
\midrule
\rowcolor{curriculumbg}
Tennis & \tabletag{oursbg}{Accept}$^{B,K}$ & \tabletag{oursbg}{Accept}$^{B,K}$ & \tabletag{oursbg}{Accept}$^{B,K}$ \\
\rowcolor{curriculumbg}
Football & \tabletag{oursbg}{Accept}$^{B,K}$ & \tabletag{oursbg}{Accept}$^{B,K}$ & \tabletag{oursbg}{Accept}$^{B,K}$ \\
\rowcolor{curriculumbg}
Dance & \tabletag{basebg}{Protect}$^{B}$ & \tabletag{basebg}{Protect}$^{B}$ & \tabletag{basebg}{Protect}$^{B}$ \\
\rowcolor{curriculumbg}
Kung Fu & \tabletag{basebg}{Protect}$^{B,K}$ & \tabletag{basebg}{Protect}$^{B}$ & \tabletag{basebg}{Protect}$^{B,K}$ \\
\bottomrule\end{tabularx}
\par\smallskip{\papertablenote\raggedright Accept selects the refined policy; Protect retains the current policy. B: body guard; K: execution ordering. Full protocols appear in Appendix~\ref{sec:gate-results}.\par}
\par\smallskip\hrule height .4pt
\endgroup

\end{minipage}\hfill
\begin{minipage}[t]{\HalfPlate}\vspace{0pt}
% Generated by scripts/rebuild_table_layouts.py
\centering\pairtablecaption{\textbf{Execution overrides reward.} Two Kung Fu proposals improve training reward but shorten the repeated-evaluation horizon; both are rejected.}{tab:rejection}
\begingroup\papertablestyle
\begin{tabularx}{\linewidth}{l*{3}{>{\centering\arraybackslash}X}}
\toprule\rowcolor{tablehead}
Metric & Current & \makecell{Proposal\\1} & \makecell{Proposal\\2} \\
\midrule
\rowcolor{curriculumbg}
Reward & 0.5707 & 0.6275 & 0.5838 \\
$\Delta$ reward & +0.0000 & +0.0568 & +0.0131 \\
\rowcolor{curriculumbg}
Horizon & \textbf{879} & 677 & 689 \\
Gate & \tabletag{basebg}{Retain} & \tabletag{basebg}{Reject} & \tabletag{basebg}{Reject} \\
\bottomrule

\end{tabularx}
\par\smallskip{\papertablenote\raggedright Historical case: horizon is the worst repeated-evaluation failure step. Proposals lose 202/190 steps; reward changes use the current policy.\par}
\par\smallskip\hrule height .4pt
\endgroup

\end{minipage}
\end{table}

Across recorded verification decisions, retained policies never shortened
the tested execution horizons. On Broad Jump, C1 restored full execution,
while C5 improved tracking without extending the horizon.
Each joint error comparison uses the
prefix executed by both policies, so a longer rollout can expose new
tracking errors. World-space tracking measures trajectory alignment, while
root-local tracking measures articulation independently of global
displacement. The direct check against the warmstart applies across stages,
even after an intermediate policy becomes the new parent
(Appendix~\ref{sec:theory-gate}).

\clearpage
\subsection{Complementary Supervision Components}

\begin{figure}[H]
\centering
\includegraphics[width=\linewidth]{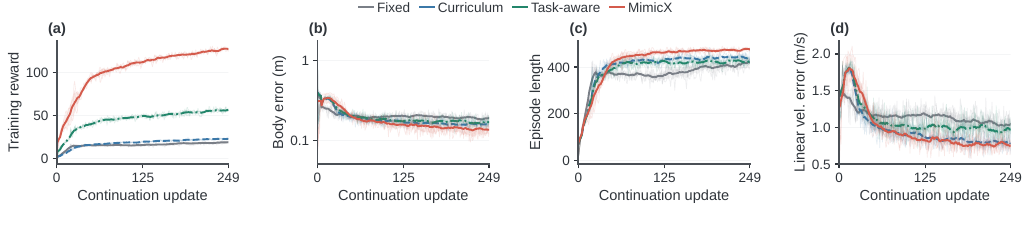}
\caption{\textbf{Coordinated supervision changes the learning trajectory.}
(a) Reward, (b) body error, (c) episode length and (d) linear-velocity error under the Tennis video
protocol, with three seeds per setting and 250 updates per
continuation. Thin traces show individual runs; curves and bands
show mean and sample standard deviation of 11-update-smoothed seed
traces. MimicX shows its selected continuation.}
\label{fig:main-learning}\label{fig:learning}
\end{figure}

\begin{table}[H]\centering
\caption{\textbf{Four-video continuation comparison.} Execution horizon and complementary tracking channels, mean $\pm$ sample SD across three continuation seeds.}\label{tab:video-main}
\begingroup\papertablestyle
\begin{tabularx}{\textwidth}{>{\raggedright\arraybackslash}p{.21\textwidth}*{6}{>{\centering\arraybackslash}X}}\toprule
\rowcolor{tablehead}
\textbf{Supervision} & \makecell{Horizon $\uparrow$\\(steps)} & \makecell{Body mean $\downarrow$\\(m)} & \makecell{Body P95 $\downarrow$\\(m)} & \makecell{Anchor $\downarrow$\\(m)} & \makecell{Linear vel. $\downarrow$\\(m/s)} & \makecell{Angular vel. $\downarrow$\\(rad/s)} \\
\midrule
\rowcolor{tablegroup}
\multicolumn{7}{l}{\textit{Tennis Swing ($T=800$)}} \\
\rowcolor{basebg}
Fixed Reference & 322.0 $\pm$ 302.2 & 0.241 $\pm$ 0.020 & 0.518 $\pm$ 0.125 & 0.301 $\pm$ 0.023 & 0.566 $\pm$ 0.053 & 2.287 $\pm$ 0.031 \\
\rowcolor{curriculumbg}
Failure Curriculum & \textbf{801.0} $\pm$ 0.0 & 0.184 $\pm$ 0.002 & 0.282 $\pm$ 0.007 & 0.208 $\pm$ 0.014 & \textbf{0.426} $\pm$ 0.004 & \textbf{1.798} $\pm$ 0.012 \\
\rowcolor{taskbg}
Task-Aware Refinement & 768.0 $\pm$ 57.2 & 0.178 $\pm$ 0.011 & 0.293 $\pm$ 0.012 & \textbf{0.177} $\pm$ 0.023 & 0.454 $\pm$ 0.003 & 2.074 $\pm$ 0.027 \\
\rowcolor{oursbg}
MimicX & \textbf{801.0} $\pm$ 0.0 & \textbf{0.157} $\pm$ 0.008 & \textbf{0.278} $\pm$ 0.009 & 0.276 $\pm$ 0.006 & 0.439 $\pm$ 0.017 & 1.979 $\pm$ 0.064 \\
\rowcolor{tablegroup}
\multicolumn{7}{l}{\textit{Football Juggling ($T=455$)}} \\
\rowcolor{basebg}
Fixed Reference & 53.7 $\pm$ 10.5 & 0.286 $\pm$ 0.017 & 0.612 $\pm$ 0.096 & 0.130 $\pm$ 0.009 & 0.840 $\pm$ 0.056 & 3.669 $\pm$ 0.241 \\
\rowcolor{curriculumbg}
Failure Curriculum & 264.3 $\pm$ 2.1 & \textbf{0.237} $\pm$ 0.007 & \textbf{0.395} $\pm$ 0.005 & 0.093 $\pm$ 0.006 & \textbf{0.535} $\pm$ 0.006 & \textbf{2.051} $\pm$ 0.025 \\
\rowcolor{taskbg}
Task-Aware Refinement & 360.0 $\pm$ 39.8 & 0.246 $\pm$ 0.007 & 0.448 $\pm$ 0.002 & \textbf{0.083} $\pm$ 0.014 & 0.599 $\pm$ 0.019 & 2.765 $\pm$ 0.076 \\
\rowcolor{oursbg}
MimicX & \textbf{427.3} $\pm$ 22.4 & 0.244 $\pm$ 0.017 & 0.504 $\pm$ 0.049 & 0.108 $\pm$ 0.026 & 0.568 $\pm$ 0.021 & 2.489 $\pm$ 0.032 \\
\rowcolor{tablegroup}
\multicolumn{7}{l}{\textit{Dance Sequence ($T=784$)}} \\
\rowcolor{basebg}
Fixed Reference & 193.3 $\pm$ 88.6 & 0.270 $\pm$ 0.008 & 0.502 $\pm$ 0.051 & 0.215 $\pm$ 0.042 & 0.811 $\pm$ 0.056 & 3.779 $\pm$ 0.216 \\
\rowcolor{curriculumbg}
Failure Curriculum & 506.3 $\pm$ 14.6 & \textbf{0.243} $\pm$ 0.006 & 0.488 $\pm$ 0.035 & 0.100 $\pm$ 0.010 & \textbf{0.622} $\pm$ 0.007 & \textbf{2.892} $\pm$ 0.025 \\
\rowcolor{taskbg}
Task-Aware Refinement & \textbf{522.3} $\pm$ 0.6 & 0.256 $\pm$ 0.008 & 0.482 $\pm$ 0.016 & \textbf{0.077} $\pm$ 0.003 & 0.631 $\pm$ 0.003 & 3.484 $\pm$ 0.032 \\
\rowcolor{oursbg}
MimicX & 341.3 $\pm$ 1.2 & 0.244 $\pm$ 0.001 & \textbf{0.467} $\pm$ 0.000 & 0.133 $\pm$ 0.000 & 0.686 $\pm$ 0.004 & 3.447 $\pm$ 0.010 \\
\rowcolor{tablegroup}
\multicolumn{7}{l}{\textit{Kung Fu ($T=1062$)}} \\
\rowcolor{basebg}
Fixed Reference & 97.7 $\pm$ 59.4 & 0.249 $\pm$ 0.005 & 0.634 $\pm$ 0.189 & 0.229 $\pm$ 0.115 & 0.548 $\pm$ 0.062 & 2.084 $\pm$ 0.038 \\
\rowcolor{curriculumbg}
Failure Curriculum & 298.7 $\pm$ 144.1 & 0.194 $\pm$ 0.044 & 0.362 $\pm$ 0.084 & 0.140 $\pm$ 0.037 & \textbf{0.220} $\pm$ 0.041 & \textbf{0.998} $\pm$ 0.155 \\
\rowcolor{taskbg}
Task-Aware Refinement & \textbf{415.0} $\pm$ 94.0 & 0.209 $\pm$ 0.011 & 0.358 $\pm$ 0.029 & 0.108 $\pm$ 0.007 & 0.255 $\pm$ 0.006 & 1.249 $\pm$ 0.093 \\
\rowcolor{oursbg}
MimicX & 196.0 $\pm$ 0.0 & \textbf{0.141} $\pm$ 0.000 & \textbf{0.296} $\pm$ 0.000 & \textbf{0.070} $\pm$ 0.000 & 0.230 $\pm$ 0.000 & 1.038 $\pm$ 0.000 \\
\bottomrule\end{tabularx}
\par\smallskip\begin{minipage}{\textwidth}\papertablenote Three evaluations per seed; horizon is their earliest termination, with $T+1$ for none. Body mean/P95 pool aligned worst-body errors; other channels use world coordinates. MimicX uses three 250-update candidates and repaired Dance/Kung Fu references; others use 250 updates. Errors use each configuration\textquotesingle s reference. Bold: best task mean. Details: Appendix~\ref{sec:setup} and Table~\ref{tab:core}.
\par\smallskip\hrule height .4pt\end{minipage}
\endgroup\end{table}

In the four-video study, failure-window practice gave the lowest
task-averaged velocity errors, while task-aware supervision gave the lowest
anchor error across tasks. The verified loop yielded the lowest
task-averaged body error. Relative to Fixed Reference, body error fell by
35.0\% on Tennis, 14.7\% on Football, 9.9\% on Dance, and 43.3\% on Kung Fu.
Tennis completed all nine 800-step evaluations, up from one; Football's
worst-repeat horizon rose from 53.7 to 427.3 steps.
The Tennis learning traces show how these differences develop over each
250-update continuation. The selected MimicX continuations lowered logged
body error to 0.143\,m versus 0.189\,m for Fixed Reference at the final update.
Checkpoint evaluations extend this comparison beyond the training traces.
Across 48 continuations, five checkpoints per run track how much of the
motion the policy executes. On Football, MimicX led all four configurations
in integrated execution horizon (0.744 versus 0.170 for Fixed Reference)
and at the final checkpoint (Table~\ref{tab:milestones}).
Appendix~\ref{sec:setup} specifies the component study's reference and
candidate budgets. Figure~\ref{fig:main-learning} compares the selected
Tennis continuations on a common update axis; the main tracking study
instead matches the total update budget and checkpoint opportunities
(Section~\ref{sec:main-setup}).

\clearpage
\subsection{Continuous Skills from Video}

The video pathway links input video, human reconstruction, robot reference,
and recorded execution (Figure~\ref{fig:registered_stages}).
The paired scene study holds reference motion and collision geometry fixed
while changing policy supervision. Both methods continue from the same
checkpoint for 3,000 PPO updates per scene. Three paired evaluations per
scene measure full-length execution and tracking through common motion phases.

\medskip
\noindent\begin{minipage}[t]{\HalfPlate}
\vspace{0pt}
\captionsetup{hypcap=false}
\captionof{table}{\textbf{Paired video-scene tracking.} Fixed Reference, then MimicX (cm).}
\label{tab:current-scenes}
\begingroup
\papertablestyle
\begin{tabularx}{\linewidth}{@{}>{\raggedright\arraybackslash}Xcc@{}}
\toprule
\rowcolor{tablehead}
\textbf{Scene} & \textbf{Body P95} & \textbf{Root P95} \\
\midrule
Track Running & 24.78 / \textbf{11.45} & 10.82 / \textbf{8.12} \\
Stair Ascent & 21.06 / \textbf{18.56} & 20.70 / \textbf{19.07} \\
Forest Traversal & 83.10 / \textbf{17.32} & 78.29 / \textbf{8.58} \\
Platform Jump & 30.56 / \textbf{28.79} & 35.22 / \textbf{17.31} \\
Parkour & 60.94 / \textbf{59.70} & 18.50 / \textbf{18.39} \\
\bottomrule
\end{tabularx}
\endgroup

\end{minipage}\hfill
\begin{minipage}[t]{\HalfPlate}
\vspace{0pt}
Across five scenes, MimicX raised completion from 60\% to 80\%.
Forest Traversal improved from 0/3 to 3/3 complete 404-step runs; body
and root P95 fell 79.2\% and 89.0\% on the shared pre-failure interval.
Track Running, Stair Ascent and Platform Jump maintained completion while
improving tracking. Table~\ref{tab:current-scenes} averages per-rollout P95
over three paired pre-failure intervals.
\end{minipage}\par\medskip

\begin{figure}[H]
\centering
\begin{subfigure}[t]{.32\linewidth}
\vspace{0pt}
\includegraphics[width=\linewidth]{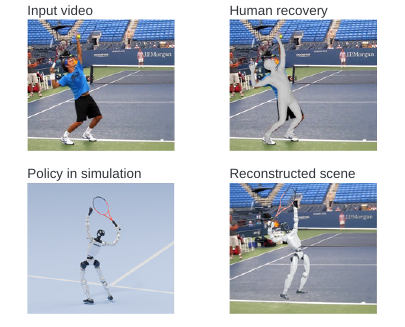}
\caption{Tennis: recovered scene context}
\end{subfigure}\hfill
\begin{subfigure}[t]{.32\linewidth}
\vspace{0pt}
\includegraphics[width=\linewidth]{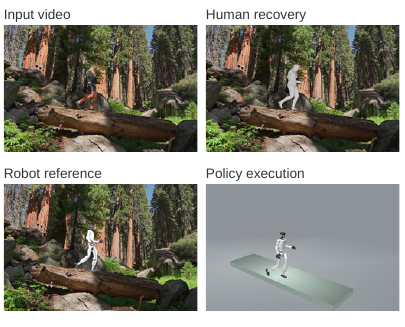}
\caption{Forest: collision-scene execution}
\end{subfigure}\hfill
\begin{subfigure}[t]{.32\linewidth}
\vspace{0pt}
\includegraphics[width=\linewidth]{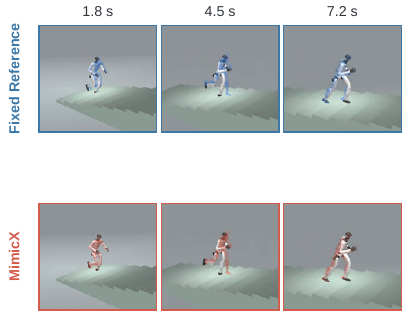}
\caption{Stair Ascent: matched policy states}
\label{fig:collision-extension}\label{fig:stairs-matched-states}
\end{subfigure}
\caption{\textbf{From video observations to humanoid execution.}
Tennis shows the input video, human reconstruction, simulation and registered scene
at a recovered source phase. Forest shows the input video, human reconstruction,
robot reference and recorded execution. Stair Ascent compares Fixed Reference and
MimicX at 20/50/80\% of the motion (seed 2002, matched crops).}
\label{fig:registered_stages}\label{fig:main-pipeline}
\end{figure}

\begin{figure}[H]
\centering
\includegraphics[width=\linewidth]{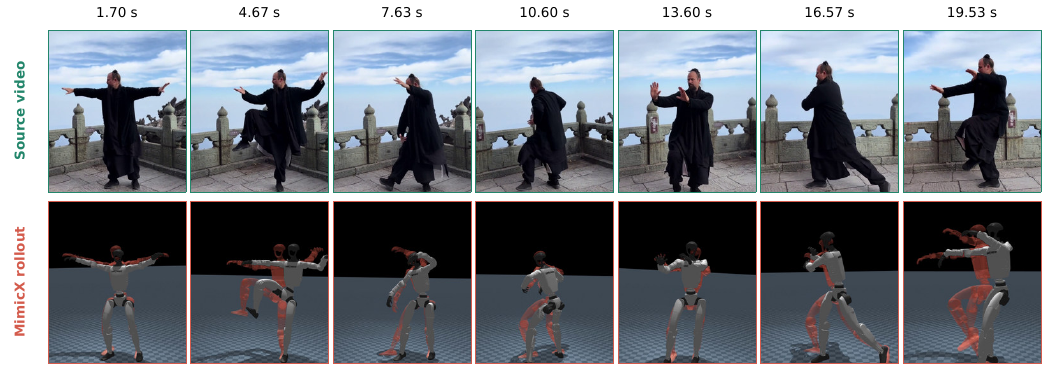}
\caption{\textbf{Kung Fu from video to recorded execution.} Seven paired
source and simulation states trace the updated MimicX rollout on the input
clock. The coral overlay shows the reference.}
\label{fig:kungfu-video-sequence}
\end{figure}

\clearpage
\subsection{Source-Aligned Policy Execution}

The updated Tennis and Football trackers share a source clock within each
motion (Figure~\ref{fig:main-qualitative}), revealing differences in root
displacement and raised-leg placement. Figure~\ref{fig:procession_tennis}
connects Tennis execution to the reconstructed motion and registered court.

\begingroup
\setlength{\intextsep}{5pt plus 1pt minus 1pt}
\begin{figure}[H]
\centering
\begin{subfigure}[t]{\HalfPlate}
\includegraphics[width=\linewidth]{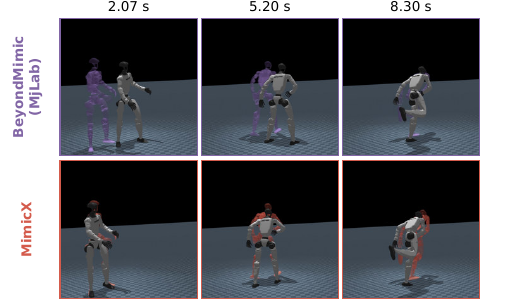}
\caption{Tennis Swing}
\end{subfigure}\hfill
\begin{subfigure}[t]{\HalfPlate}
\includegraphics[width=\linewidth]{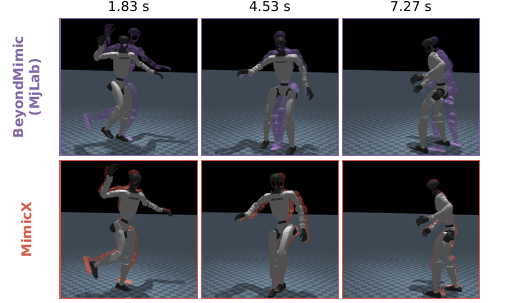}
\caption{Football Juggling}
\end{subfigure}
\caption{\textbf{Execution differences on the same motion timeline.}
BeyondMimic (MjLab) and MimicX rollouts at 20, 50 and
80\% of each source interval. Pairs share their clock, camera and
crop. Solid robots show execution; purple and coral show references.}
\label{fig:main-qualitative}
\end{figure}

\begin{figure}[H]
\centering
\includegraphics[width=\linewidth]{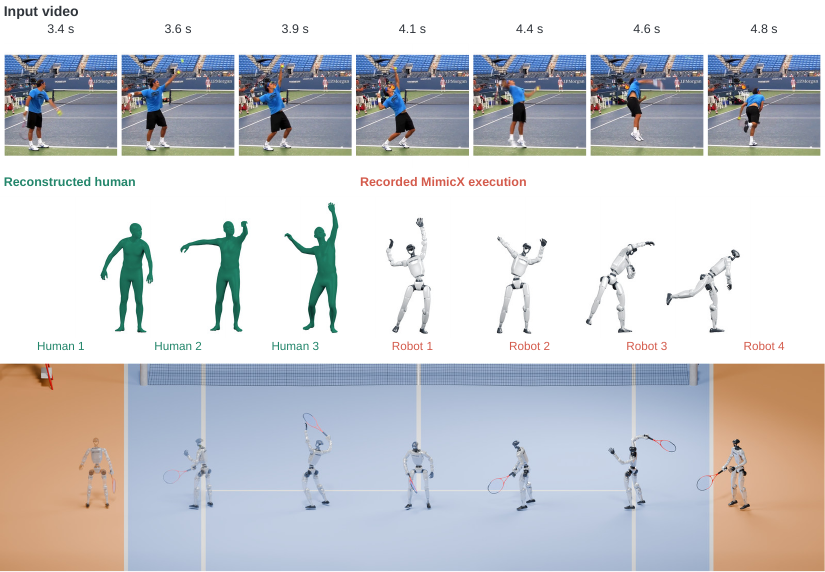}
\caption{\textbf{Tennis across time, embodiment and scene.}
Top: seven input frames, three reconstructed human states and four policy
states from the original video study at recovered source phases.
Bottom: seven recorded MimicX states in the registered court, separated
spatially with increasing opacity over time.}
\label{fig:procession_tennis}\label{fig:tennis_panorama}
\end{figure}

\endgroup

\clearpage
\subsection{Additional Videos and Motion References}

Four additional videos evaluate the loop on inputs beyond the component
study (Table~\ref{tab:additional}). Each task and method uses three
continuation seeds and three repeated evaluations. On Soccer, MimicX
reduced mean body error by 11.0\% and increased full-horizon success from
0\% to 33.3\%. The two dance videos yielded reductions of 21.0\% and
36.7\%. Figures~\ref{fig:soccer_events} and~\ref{fig:uniandes_events}
show the resulting policy execution on Soccer and Uniandes Dance.

% Table design v3
% Generated by scripts/build_tables.py; source hashes: tables/source_manifest.json
\begin{table}[H]
\centering
\caption{Additional video-task outcomes. Body error is reported in meters ($\downarrow$); bold marks the lower error within each task. Values are recorded task aggregates. Success change is in percentage points (pp).}
\label{tab:additional}
\begingroup
\papertablestyle
\begin{tabularx}{\linewidth}{>{\raggedright\arraybackslash}p{.31\linewidth}*{4}{>{\centering\arraybackslash}X}}\toprule
\rowcolor{tablehead}
\textbf{Video task} & \textbf{Fixed Reference} & \textbf{\mimicx} & \textbf{\makecell{Error reduction\\(\%, $\uparrow$)}} & \textbf{\makecell{$\Delta$ success\\(pp, $\uparrow$)}}\\
\midrule
Basketball & \textbf{0.204} & 0.217 & -6.22 & +0.0 \\
Soccer & 0.279 & \textbf{0.248} & +11.03 & +33.3 \\
Michael Jackson Dance & 0.413 & \textbf{0.326} & +21.04 & +0.0 \\
Uniandes Dance & 0.289 & \textbf{0.183} & +36.68 & +0.0 \\

\bottomrule
\end{tabularx}
\par\smallskip\noindent\begin{minipage}{\linewidth}\papertablenote Body error is the temporal mean of the maximum aligned tracked-body error, pooled over rollout samples.\par\smallskip\hrule height 0.4pt\relax\end{minipage}
\endgroup
\end{table}

\begin{figure}[H]
\begin{minipage}[t]{\HalfPlate}
\vspace{0pt}
\includegraphics[width=\linewidth]{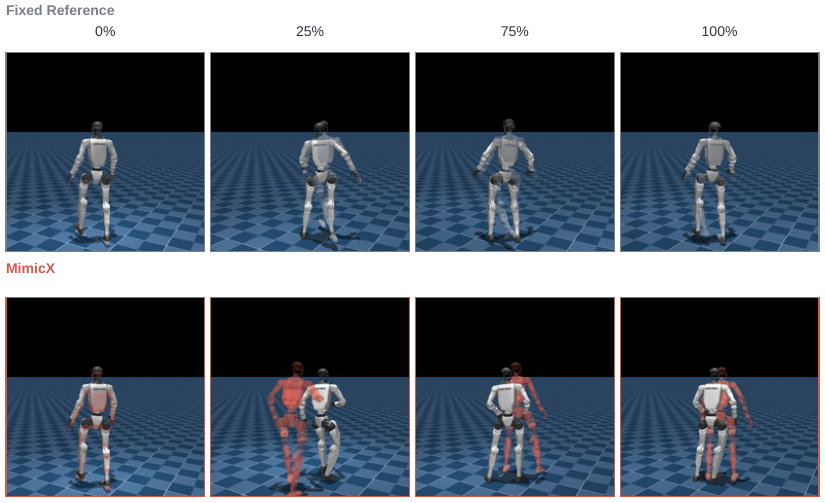}
\captionsetup{hypcap=false}
\caption{\textbf{Soccer.} Fixed Reference and MimicX at four phases.}
\label{fig:soccer_events}
\end{minipage}\hfill
\begin{minipage}[t]{\HalfPlate}
\vspace{0pt}
\includegraphics[width=\linewidth]{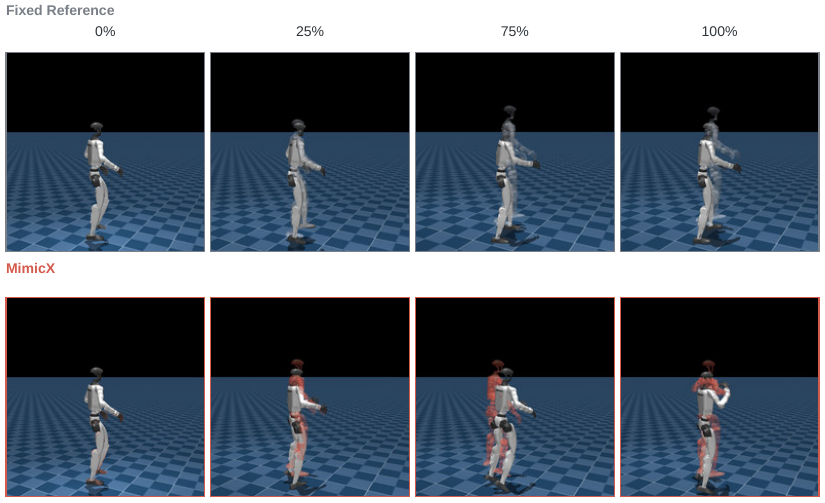}
\captionsetup{hypcap=false}
\caption{\textbf{Uniandes Dance.} Fixed Reference and MimicX.}
\label{fig:uniandes_events}
\end{minipage}
\end{figure}

Prepared robot references enter through the same policy-training interface.
Across eleven LaFAN1 references, the mean per-reference body-error
reduction was 15.8\%; across three AMASS/ACCAD references, it was 19.8\%
(Table~\ref{tab:breadth}). These experiments train a policy for each reference.
The source motions span transitions in locomotion, jumps, dance, and
articulated skills.

% Table design v3
% Generated by scripts/build_tables.py; source hashes: tables/source_manifest.json
\begin{table}[H]
\centering
\caption{Motion-reference breadth: fourteen evaluated references grouped by source dataset. Body error is reported in meters ($\downarrow$); bold compares methods within a row. Cohort reductions average per-reference relative reductions rather than taking the ratio of cohort-mean errors.}
\label{tab:breadth}
\begingroup
\papertablestyle
\begin{tabularx}{\linewidth}{>{\raggedright\arraybackslash}p{.22\linewidth}*{3}{>{\centering\arraybackslash}X}>{\raggedright\arraybackslash}p{.22\linewidth}*{3}{>{\centering\arraybackslash}X}}\toprule

\rowcolor{tablegroup}
\multicolumn{8}{l}{\textit{LaFAN1: 11 evaluated references}}\\\rowcolor{tablehead}
\textbf{Motion reference} & \textbf{Fixed} & \textbf{\mimicx} & \textbf{Red. (\%)} & \textbf{Motion reference} & \textbf{Fixed} & \textbf{\mimicx} & \textbf{Red. (\%)}\\
\midrule
Dance 1 & 0.2866 & \textbf{0.2499} & +12.81 & Jumps 1 & 0.2871 & \textbf{0.1606} & +44.06 \\
Dance 2 & 0.2961 & \textbf{0.2885} & +2.55 & Run 1 & 0.2523 & \textbf{0.1985} & +21.30 \\
Fall and Get Up 1 & \textbf{0.3746} & 0.3972 & -6.03 & Sprint 1 & 0.2861 & \textbf{0.2171} & +24.12 \\
Fall and Get Up 2 & \textbf{0.3540} & 0.3544 & -0.13 & Walk 1 & 0.2462 & \textbf{0.1419} & +42.36 \\
Fight 1 & 0.2988 & \textbf{0.2385} & +20.18 & Walk 4 & \textbf{0.2341} & 0.3008 & -28.51 \\
Fight and Sports 1 & 0.3099 & \textbf{0.1821} & +41.24 & Reference-weighted mean & 0.2932 & \textbf{0.2481} & +15.81 \\
\midrule
\rowcolor{tablegroup}
\multicolumn{8}{l}{\textit{AMASS/ACCAD: 3 evaluated references}}\\\rowcolor{tablehead}
\textbf{Motion reference} & \textbf{Fixed} & \textbf{\mimicx} & \textbf{Red. (\%)} & \textbf{Motion reference} & \textbf{Fixed} & \textbf{\mimicx} & \textbf{Red. (\%)}\\
\midrule
Form 1 (provided retargeting) & 0.2140 & \textbf{0.1858} & +13.17 & Run, Change Direction & 0.4376 & \textbf{0.4313} & +1.43 \\
Form 1 (new retargeting) & 0.3728 & \textbf{0.2054} & +44.90 & Reference-weighted mean & 0.3415 & \textbf{0.2742} & +19.84 \\
\bottomrule
\end{tabularx}
\par\smallskip\noindent\begin{minipage}{\linewidth}\papertablenote Body error is the temporal mean of the maximum aligned tracked-body error, pooled over rollout samples. Form 1 originates from AMASS/ACCAD. AMASS/ACCAD contains three references from two source motions, including two Form 1 retargetings. References are equally weighted.\par\smallskip\hrule height 0.4pt\relax\end{minipage}
\endgroup
\end{table}

\clearpage
\subsection{Training Cost and Feedback Efficiency}

The twelve-motion controllers had the same continuation budgets.
Assigned training, selection, and evaluation time totaled 42.71 GPU-h for
MimicX, versus 42.50 and 42.91 GPU-h for BeyondMimic and Static Search.
These totals exclude the warmstarts shared among all three methods.
Figure~\ref{fig:current-efficiency}(a) shows the distribution across
the 24 continuations per method; panel (b) compares the same task and seed.

\begin{figure}[H]
  \centering
  \includegraphics[width=\linewidth]{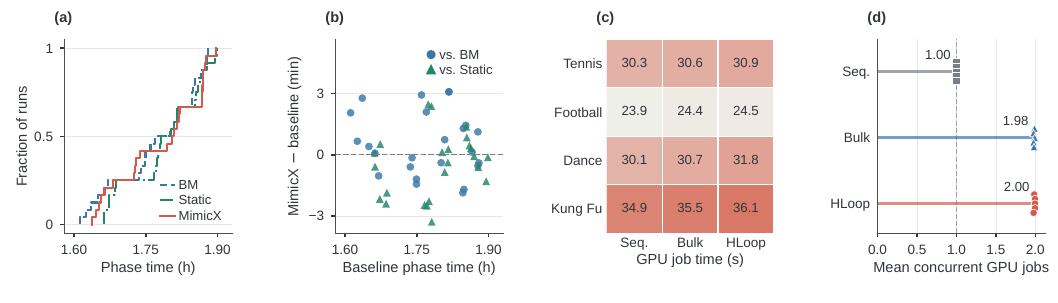}
  \caption{\textbf{Continuation cost and feedback latency.}
  (a) Continuation-phase wall time, excluding the shared warmstart.
  (b) MimicX minus baseline cost for each of the 24 matched task--seed pairs
  per comparison. BM: BeyondMimic.
  (c) Median GPU rollout job times by motion and scheduler, ten jobs per cell.
  (d) Mean concurrent GPU jobs (summed rollout job-seconds divided by
  end-to-end wall time): points show the five repetitions and labels show medians.}
  \label{fig:current-efficiency}
\end{figure}

\noindent\begin{minipage}[t]{\HalfPlate}
\vspace{0pt}
\captionsetup{hypcap=false}
\captionof{table}{\textbf{Feedback latency with fixed policies.}}
\label{tab:current-hloop}\label{tab:hloop}
\begingroup\papertablestyle
\setbox0=\hbox{%
\begin{tabularx}{\columnwidth}{>{\raggedright\arraybackslash}Xccc}
\toprule
\rowcolor{tablehead}
\textbf{Schedule} & \makecell{\textbf{Median}\\\textbf{(s, $\downarrow$)}} & \makecell{\textbf{Speedup}\\\textbf{($\uparrow$)}} & \textbf{Agreement}\\
\midrule
\rowcolor{basebg} Sequential & 241.22 & 1.000$\times$ & 5/5\\
Bulk-synchronous & \textbf{122.26} & \textbf{1.973}$\times$ & 5/5\\
\rowcolor{oursbg} MimicX-HLoop & 124.02 & 1.945$\times$ & 5/5\\
\bottomrule
\end{tabularx}%
}
\ifdefined\CurrentTableAudit
\typeout{CURRENTTABLE|hloop_current|\the\wd0|\the\dimexpr\ht0+\dp0\relax|\the\columnwidth}
\fi
\usebox0
\par\smallskip\noindent\begin{minipage}{\columnwidth}\papertablenote
Agreement counts repetitions that select the same reports as sequential
execution. Speedup is relative to sequential.
\par\smallskip\hrule height 0.4pt\relax\end{minipage}
\endgroup

\par\smallskip
HLoop preserved selection and reduced median feedback latency from
241.22 to 124.02\,s, a 1.945$\times$ speedup. The paired measurements
are listed in Table~\ref{tab:hloop-raw}.
\end{minipage}\hfill
\begin{minipage}[t]{\HalfPlate}
\vspace{0pt}
\captionsetup{type=table,hypcap=false}
% Generated by scripts/build_tables.py
\centering
\pairtablecaption{\textbf{Paired feedback timings.} Per-run wall time in seconds and ratio to HLoop; bold marks the fastest executor in each repetition.}{tab:hloop-raw}
\begingroup\papertablestyle
\begin{tabularx}{\linewidth}{c*{5}{>{\centering\arraybackslash}X}}
\toprule\rowcolor{tablehead}
Run & \makecell{Seq.\\(s)} & \makecell{Bulk\\(s)} & \makecell{HLoop\\(s)} & \makecell{Seq. /\\HLoop} & \makecell{Bulk /\\HLoop} \\
\midrule
1 & 242.403 & \textbf{121.813} & 124.017 & 1.955 & 0.982 \\
\rowcolor{curriculumbg}
2 & 239.627 & 126.440 & \textbf{121.040} & 1.980 & 1.045 \\
3 & 238.569 & 120.474 & \textbf{119.924} & 1.989 & 1.005 \\
\rowcolor{curriculumbg}
4 & 241.221 & 126.131 & \textbf{125.624} & 1.920 & 1.004 \\
5 & 241.921 & \textbf{122.258} & 128.876 & 1.877 & 0.949 \\
\bottomrule

\end{tabularx}
\par\smallskip{\papertablenote\raggedright Each ratio divides an executor time by the HLoop time within each run; HLoop / HLoop is 1.000 in all five runs. These repetitions are summarized in Table~\ref{tab:hloop}. Headline speedup uses medians, not mean paired ratios.\par}
\par\smallskip\hrule height .4pt
\endgroup

\end{minipage}\par\medskip

Each fixed-policy benchmark repetition evaluates four motions with two
rollouts per motion, diagnoses all eight rollouts, and selects once.
HLoop starts each CPU diagnosis as soon as its rollout finishes, without
waiting for the remaining GPU jobs. The selector still consumes the same
ordered reports and makes the same decision (Appendix~\ref{sec:theory-hloop}).

HLoop saves 113.0--118.6\,s per run relative to sequential execution.
Per-motion GPU job durations are similar across schedulers
(Figure~\ref{fig:current-efficiency}(c)); the parallel schedules instead keep
nearly two rollout jobs active on average throughout the workload
(Figure~\ref{fig:current-efficiency}(d)).
These timings isolate feedback; continuation costs include training
and checkpoint evaluation.

\Needspace{4\baselineskip}
\section{Conclusion}

\mimicx{} turns execution feedback into targeted supervision for continuous
humanoid motion tracking. Failure localization guides tracking objectives
and practice states, while repeated verification selects the checkpoint
for continued learning. Across prepared motions and video-derived collision
scenes, this coupling improves uninterrupted execution and world-space
tracking under matched training budgets. \hloop{} reduces feedback latency
while preserving checkpoint selection. Execution thus guides both
supervision refinement and policy continuation.

\clearpage
{\fontsize{9}{9.8}\selectfont
\setlength{\bibsep}{.4pt}
\begin{multicols}{2}
\bibliographystyle{ieeenat_fullname}
\bibliography{refs}

@inproceedings{gvhmr2024,
  author    = {Shen, Zehong and Pi, Huaijin and Xia, Yan and Cen, Zhi and Peng, Sida and Hu, Zechen and Bao, Hujun and Hu, Ruizhen and Zhou, Xiaowei},
  title     = {World-Grounded Human Motion Recovery via Gravity-View Coordinates},
  booktitle = {SIGGRAPH Asia 2024 Conference Papers},
  year      = {2024},
  doi       = {10.1145/3680528.3687565}
}

@article{gmr2025,
  author        = {Araujo, Joao Pedro and Ze, Yanjie and Xu, Pei and Wu, Jiajun and Liu, C. Karen},
  title         = {Retargeting Matters: General Motion Retargeting for Humanoid Motion Tracking},
  journal       = {arXiv preprint arXiv:2510.02252},
  year          = {2025},
  eprint        = {2510.02252v1},
  archivePrefix = {arXiv}
}

@article{beyondmimic2025,
  author        = {Liao, Qiayuan and Truong, Takara E. and Huang, Xiaoyu and Gao, Yuman and Tevet, Guy and Sreenath, Koushil and Liu, C. Karen},
  title         = {{BeyondMimic}: From Motion Tracking to Versatile Humanoid Control via Guided Diffusion},
  journal       = {arXiv preprint arXiv:2508.08241},
  year          = {2025},
  eprint        = {2508.08241v4},
  archivePrefix = {arXiv}
}

@article{ppo2017,
  author        = {Schulman, John and Wolski, Filip and Dhariwal, Prafulla and Radford, Alec and Klimov, Oleg},
  title         = {Proximal Policy Optimization Algorithms},
  journal       = {arXiv preprint arXiv:1707.06347},
  year          = {2017},
  eprint        = {1707.06347v2},
  archivePrefix = {arXiv}
}

@inproceedings{mujoco2012,
  author    = {Todorov, Emanuel and Erez, Tom and Tassa, Yuval},
  title     = {{MuJoCo}: A Physics Engine for Model-Based Control},
  booktitle = {2012 IEEE/RSJ International Conference on Intelligent Robots and Systems},
  pages     = {5026--5033},
  year      = {2012},
  doi       = {10.1109/IROS.2012.6386109}
}

@inproceedings{smplx2019,
  author    = {Pavlakos, Georgios and Choutas, Vasileios and Ghorbani, Nima and Bolkart, Timo and Osman, Ahmed A. A. and Tzionas, Dimitrios and Black, Michael J.},
  title     = {Expressive Body Capture: {3D} Hands, Face, and Body from a Single Image},
  booktitle = {Proceedings of the IEEE/CVF Conference on Computer Vision and Pattern Recognition},
  pages     = {10975--10985},
  year      = {2019}
}

@article{lafan2020,
  author  = {Harvey, F{\'e}lix G. and Yurick, Mike and Nowrouzezahrai, Derek and Pal, Christopher},
  title   = {Robust Motion In-Betweening},
  journal = {ACM Transactions on Graphics},
  volume  = {39},
  number  = {4},
  year    = {2020},
  doi     = {10.1145/3386569.3392480}
}

@inproceedings{amass2019,
  author    = {Mahmood, Naureen and Ghorbani, Nima and Troje, Nikolaus F. and Pons-Moll, Gerard and Black, Michael J.},
  title     = {{AMASS}: Archive of Motion Capture as Surface Shapes},
  booktitle = {Proceedings of the IEEE/CVF International Conference on Computer Vision},
  pages     = {5442--5451},
  year      = {2019}
}

@article{sonic2025,
  author  = {Luo, Zhengyi and Yuan, Ye and Wang, Tingwu and Li, Chenran and Casta{\~n}eda, Fernando and Chen, Sirui and Cao, Zi-Ang and Li, Jiefeng and Minor, David and Ben, Qingwei and Park, Jinhyung and Sami, David and Wang, Zi and Da, Xingye and Ding, Runyu and Hogg, Cyrus and Song, Lina and Lim, Edy and Jeong, Eugene and He, Tairan and Xue, Haoru and Xiao, Wenli and Yuen, Simon and Kautz, Jan and Chang, Yan and Iqbal, Umar and Fan, Linxi ``Jim'' and Zhu, Yuke},
  title   = {{SONIC}: Supersizing Motion Tracking for Natural Humanoid Whole-Body Control},
  journal = {Science Robotics},
  volume  = {11},
  number  = {117},
  pages   = {eaed4592},
  year    = {2026},
  doi     = {10.1126/scirobotics.aed4592}
}

@article{deepmimic2018,
  author  = {Peng, Xue Bin and Abbeel, Pieter and Levine, Sergey and van de Panne, Michiel},
  title   = {{DeepMimic}: Example-Guided Deep Reinforcement Learning of Physics-Based Character Skills},
  journal = {ACM Transactions on Graphics},
  volume  = {37},
  number  = {4},
  pages   = {143:1--143:14},
  year    = {2018},
  doi     = {10.1145/3197517.3201311}
}

@article{amp2021,
  author  = {Peng, Xue Bin and Ma, Ze and Abbeel, Pieter and Levine, Sergey and Kanazawa, Angjoo},
  title   = {{AMP}: Adversarial Motion Priors for Stylized Physics-Based Character Control},
  journal = {ACM Transactions on Graphics},
  volume  = {40},
  number  = {4},
  pages   = {144:1--144:20},
  year    = {2021},
  doi     = {10.1145/3450626.3459670}
}

@inproceedings{phc2023,
  author    = {Luo, Zhengyi and Cao, Jinkun and Winkler, Alexander and Kitani, Kris and Xu, Weipeng},
  title     = {Perpetual Humanoid Control for Real-Time Simulated Avatars},
  booktitle = {Proceedings of the IEEE/CVF International Conference on Computer Vision},
  pages     = {10895--10904},
  year      = {2023}
}

@inproceedings{kungfubot2025,
  author    = {Xie, Weiji and Han, Jinrui and Zheng, Jiakun and Li, Huanyu and Liu, Xinzhe and Shi, Jiyuan and Zhang, Weinan and Bai, Chenjia and Li, Xuelong},
  title     = {{KungfuBot}: Physics-Based Humanoid Whole-Body Control for Learning Highly-Dynamic Skills},
  booktitle = {Advances in Neural Information Processing Systems},
  volume    = {38},
  year      = {2025},
  doi       = {10.52202/085713-2089}
}

@article{adamimic2025,
  author        = {Huang, Tao and Wang, Huayi and Ren, Junli and Yin, Kangning and Wang, Zirui and Chen, Xiao and Jia, Feiyu and Zhang, Wentao and Long, Junfeng and Wang, Jingbo and Pang, Jiangmiao},
  title         = {Towards Adaptable Humanoid Control via Adaptive Motion Tracking},
  journal       = {arXiv preprint arXiv:2510.14454},
  year          = {2025},
  eprint        = {2510.14454v1},
  archivePrefix = {arXiv}
}

@inproceedings{asap2025,
  author    = {He, Tairan and Gao, Jiawei and Xiao, Wenli and Zhang, Yuanhang and Wang, Zi and Wang, Jiashun and Luo, Zhengyi and He, Guanqi and Sobanbabu, Nikhil and Pan, Chaoyi and Yi, Zeji and Qu, Guannan and Kitani, Kris and Hodgins, Jessica K. and Fan, Linxi and Zhu, Yuke and Liu, Changliu and Shi, Guanya},
  title     = {{ASAP}: Aligning Simulation and Real-World Physics for Learning Agile Humanoid Whole-Body Skills},
  booktitle = {Proceedings of Robotics: Science and Systems},
  year      = {2025},
  doi       = {10.15607/RSS.2025.XXI.066}
}

@article{maskedmimic2024,
  author  = {Tessler, Chen and Guo, Yunrong and Nabati, Ofir and Chechik, Gal and Peng, Xue Bin},
  title   = {{MaskedMimic}: Unified Physics-Based Character Control through Masked Motion Inpainting},
  journal = {ACM Transactions on Graphics},
  volume  = {43},
  number  = {6},
  year    = {2024},
  doi     = {10.1145/3687951}
}

@article{mjlab2026,
  author        = {Zakka, Kevin and Liao, Qiayuan and Yi, Brent and {Le Lay}, Louis and Sreenath, Koushil and Abbeel, Pieter},
  title         = {{mjlab}: A Lightweight Framework for {GPU}-Accelerated Robot Learning},
  journal       = {arXiv preprint arXiv:2601.22074},
  year          = {2026},
  eprint        = {2601.22074v2},
  archivePrefix = {arXiv}
}

@inproceedings{referencegoal2026,
  author    = {Wang, Jiashun and Mungai, M. Eva and Li, He and Sleiman, Jean Pierre and Hodgins, Jessica K. and Farshidian, Farbod},
  title     = {Generalizing from References using a Multi-Task Reference and Goal-Driven {RL} Framework},
  booktitle = {Proceedings of Robotics: Science and Systems},
  address   = {Sydney, Australia},
  month     = jul,
  year      = {2026},
  doi       = {10.15607/RSS.2026.XXII.026}
}

@inproceedings{omnixtreme2026,
  author    = {Wang, Yunshen and Zhu, Shaohang and Zhi, Peiyuan and Li, Yuhan and Li, Jiaxin and Li, Yong-Lu and Xiao, Yuchen and Wang, Xingxing and Jia, Baoxiong and Huang, Siyuan},
  title     = {{OmniXtreme}: Breaking the Generality Barrier in High-Dynamic Humanoid Control},
  booktitle = {Proceedings of Robotics: Science and Systems},
  address   = {Sydney, Australia},
  month     = jul,
  year      = {2026},
  doi       = {10.15607/RSS.2026.XXII.031}
}

@inproceedings{parkour2026,
  author    = {Wu, Zhen and Huang, Xiaoyu and Yang, Lujie and Zhang, Yuanhang and Chen, Xi and Abbeel, Pieter and Duan, Rocky and Kanazawa, Angjoo and Sferrazza, Carmelo and Shi, Guanya and Liu, C. Karen},
  title     = {Perceptive Humanoid Parkour: Chaining Dynamic Human Skills via Motion Matching},
  booktitle = {Proceedings of Robotics: Science and Systems},
  address   = {Sydney, Australia},
  month     = jul,
  year      = {2026},
  doi       = {10.15607/RSS.2026.XXII.020}
}

@inproceedings{hiwet2026,
  author    = {Cao, Zhanxiang and Yan, Liyun and Zhang, Yang and Chen, Sirui and Ma, Jianming and Zhan, Tianyue and Fu, Shengcheng and Jia, Yufei and Lu, Cewu and Gao, Yue},
  title     = {{HiWET}: Hierarchical World-Frame End-Effector Tracking for Long-Horizon Humanoid Loco-Manipulation},
  booktitle = {Proceedings of Robotics: Science and Systems},
  address   = {Sydney, Australia},
  month     = jul,
  year      = {2026},
  doi       = {10.15607/RSS.2026.XXII.030}
}

@article{meshmimic2026,
  title = {{MeshMimic}: Geometry-Aware Humanoid Motion Learning through {3D} Scene Reconstruction},
  author = {Zhang, Qiang and Ma, Jiahao and Liu, Peiran and Shi, Shuai and Su, Zeran and Wang, Zifan and Sun, Jingkai and Cui, Wei and Yu, Jialin and Han, Gang and Zhao, Wen and Sun, Pihai and Yin, Kangning and Wang, Jiaxu and Cao, Jiahang and Zhang, Lingfeng and Cheng, Hao and Hao, Xiaoshuai and Ji, Yiding and Liang, Junwei and Tang, Jian and Xu, Renjing and Guo, Yijie},
  journal = {arXiv preprint arXiv:2602.15733},
  year = {2026}
}

@inproceedings{franceschi2018bilevel,
  author = {Franceschi, Luca and Frasconi, Paolo and Salzo, Saverio and Grazzi, Riccardo and Pontil, Massimiliano},
  title = {Bilevel Programming for Hyperparameter Optimization and Meta-Learning},
  booktitle = {Proceedings of the 35th International Conference on Machine Learning},
  series = {Proceedings of Machine Learning Research},
  volume = {80},
  pages = {1568--1577},
  year = {2018}
}

@inproceedings{sutton1999policygradient,
  author = {Sutton, Richard S. and McAllester, David A. and Singh, Satinder P. and Mansour, Yishay},
  title = {Policy Gradient Methods for Reinforcement Learning with Function Approximation},
  booktitle = {Advances in Neural Information Processing Systems},
  volume = {12},
  year = {1999}
}

@article{supermimic2026,
  author = {Liao, Qiayuan and Shao, Yiyang and Long, Junfeng and Zakka, Kevin and
            Horowitz, Roberto and Abbeel, Pieter and Sreenath, Koushil},
  title = {{SuperMimic}: Scaling Humanoid Whole-Body Control with One
           Differentiable World Model},
  journal = {arXiv preprint arXiv:2610.28480},
  year = {2026}
}

@misc{supermimicreproduced,
  author = {{CuboiLeo}},
  title = {{SuperMimic-Reproduced}: Independent {MjLab} Implementation of {SuperMimic}},
  howpublished = {\href{https://github.com/CuboiLeo/SuperMimic-Reproduced}{Software repository}},
  year = {2026}
}
\end{multicols}
}

\clearpage
\appendix
\pdfbookmark[0]{Appendix}{appendix-start}
\section*{Appendix}
\setcounter{figure}{0}
\setcounter{table}{0}
\setcounter{algorithm}{0}
\renewcommand{\thefigure}{S\arabic{figure}}
\renewcommand{\thetable}{S\arabic{table}}
\renewcommand{\thealgorithm}{S\arabic{algorithm}}
\renewcommand{\theHfigure}{supp.\arabic{figure}}
\renewcommand{\theHtable}{supp.\arabic{table}}
\providecommand{\theHalgorithm}{}
\renewcommand{\theHalgorithm}{supp.\arabic{algorithm}}
\numberwithin{equation}{section}
\section{Algorithmic Details}
\label{sec:method}

The supervision loop holds the robot reference and evaluation task fixed
while changing the objective, reset distribution and continuation parent.
This section specifies the matched-budget configuration used in the main
study. The video-study configuration is given in
Appendix~\ref{sec:setup}.

\subsection{Reference and Tracking Interface}

The reference is
\begin{equation}
 R=\{\bar q_t,\dot{\bar q}_t,
 (\bar p_{t,b},\bar Q_{t,b},\bar v_{t,b},\bar\omega_{t,b})_{b\in\mathcal B}\}_{t=0}^{T_R-1},
 \label{eq:motion-reference}
\end{equation}
where bars indicate reference states and $\mathcal B$ the tracked bodies.
Video reconstruction and retargeting provide this representation; prepared
motions use the same named-joint and frame convention. The policy controls
joint-position targets. Its observations contain the reference command,
reference-anchor pose relative to the robot, base velocities, joint state
and previous action. Policy and optimizer states are resumed together,
including the adaptive learning-rate state.

For local reward terms, align reference articulation to the robot's anchor
translation and heading. With torso anchor $a$, yaw rotation
$\Delta Q_t$ from reference to robot, and
$d_t=(p^x_{t,a},p^y_{t,a},\bar p^z_{t,a})$, the aligned targets are
\begin{equation}
 \begin{aligned}
 \widetilde p_{t,b}&=d_t+\Delta Q_t(\bar p_{t,b}-\bar p_{t,a}),\\
 \widetilde Q_{t,b}&=\Delta Q_t\bar Q_{t,b}.
 \end{aligned}
 \label{eq:reference-alignment}
\end{equation}
Absolute anchor tracking remains a separate objective. For a selected
body group $G$, the positional reward is
\begin{equation}
 \phi_p(G,\sigma)=\exp\left[-\frac{\sum_{b\in G}
 \|p_{t,b}-\widetilde p_{t,b}\|_2^2}{|G|\sigma^2}\right].
 \label{eq:group-reward}
\end{equation}
Height terms retain the vertical component, orientation terms use angular
distance, and velocity terms compare reference and executed motion.
Equation~\ref{eq:main-objective} groups these terms by their supervision
role. The root-local FK evaluation in Equation~\ref{eq:current-local-error}
uses an identity root, distinct from the reward alignment.

\subsection{Diagnosis and Bounded Proposals}

For an incomplete incumbent, diagnosis uses the repeat with the shortest
valid pre-reset horizon, breaking ties by repeat order. Only its pre-reset
states supply body-specific tracking errors. The earliest failed step $f$ across
selection repeats determines the practice interval:
\begin{equation}
 \begin{gathered}
 r=\max(10,\min(90,\lfloor T/8\rfloor)),\\
 \mathcal W=[\max(0,f-r),\min(T,f+r)).
 \end{gathered}
 \label{eq:failure-window}
\end{equation}
The miner ranks body and error channels by their maximum and mean within the
diagnostic interval. Upper-limb, lower-limb, core and velocity channels map
to corresponding supervision families. If the incumbent fails on its first
step, general tracking supervision is proposed until valid body-error
measurements become available.

Admissible proposals specify reward families, target bodies, nonnegative reward
weights in $[0,10]$, exponential scales in $[0.01,5]$, and action-rate
coefficients in $[-1,0]$. Window proposals must overlap the diagnosed
interval. The bounded family provides window practice with probability
$0.45$, stronger practice with probability $0.65$ and dynamics emphasis, and
full-start consolidation. Duplicate supervision configurations are excluded.
Incomplete execution chooses the first untried diagnostic
update. Complete execution chooses from a fixed precision pair that combines
core, lower-body and upper-body tracking with full-motion practice.

For a chosen interval $[s,e)$, the reference start follows
\begin{equation}
 t_0\sim(1-\rho)\delta_0+\rho\operatorname{Unif}\{s,\ldots,e-1\}.
 \label{eq:reset-curriculum}
\end{equation}
Practice resets initialize the physical simulator at reference states and
collect new on-policy transitions under fixed reference coordinates.

\subsection{History-Conditioned Continuation}

Two stages each consume 1,500 PPO updates. A stage is trained once, with
checkpoints after 500, 1,000 and 1,500 updates. Selection evaluates those
checkpoints in chronological order against the current protected policy.
The second stage starts from the retained checkpoint, carrying its actor,
critic, optimizer and learning-rate state. Selecting an early checkpoint
does not change the budget spent on the completed stage.

The second stage switches from the next precision update to native-objective
full-motion practice only when all of the following hold: the first stage
used the first precision update; all three checkpoints and incumbent repeats
complete; the joint guard rejected all three checkpoints;
and the incumbent was unchanged. At every checkpoint, world mean and P95 must
not worsen for any matched repeat, while local mean must worsen for at least
one repeat. Every repeat must also pass the physical-frontier guard and
provide the full $T$ control steps in its paired valid interval. These
conditions identify a repeated local sacrifice rather than
a single fluctuating measurement. Full-motion practice then uses the native
objective with the first untried replay ratio. Otherwise, the next untried
precision or failure-conditioned update is retained. History uses only
selection evidence.

\subsection{Exact Verification Conditions}
\label{sec:exact-joint-gate}

The gate requires at least three matched repeated evaluations. For each
repeat it retains both the native valid-step count and the common physical
frontier. Neither may regress. The comparison interval is the minimum of
the two policies' valid-step counts and their common physical frontiers,
excluding initialization and reset observations. Within it, world-body and
root-local FK errors each supply a mean and P95. Every candidate statistic
must be at most its incumbent counterpart plus $10^{-12}$ metres.

Native checks additionally protect anchor position, aligned worst-body
position and anchor orientation on their common prefix. Their mean bounds
are $1.20$ times the incumbent plus $0.015$\,m for position and $0.03$\,rad
for orientation. When the incumbent has no valid prefix, the candidate must
have positive valid evidence; fixed startup bounds replace relative checks.
These mean/P95 bounds are $0.08/0.15$\,m for the world anchor,
$0.10/0.20$\,m for aligned body error and $0.20/0.35$\,rad for anchor rotation.

Subject to these checks and joint non-regression, progress means a strict
increase in at least one native or common physical frontier, or more than
a $2\%$ reduction in the sum of world and local mean errors over the matched
nonempty intervals. A final joint non-regression check compares the candidate
with the original warmstart. This donor check applies the joint guard,
not a second demand for $2\%$ progress over the donor. Missing or invalid
evidence cannot establish acceptance. Training reward does not enter this
joint decision.

\subsection{Heterogeneous Feedback Execution}

The feedback graph comprises GPU rollout jobs, dependent CPU diagnoses and
a selection node. A diagnosis starts when its rollout evidence is ready,
while selection waits for the complete ordered report set. Inputs, seeds
and decision rules are fixed across schedules. Appendix~\ref{sec:theory-hloop}
derives the scheduling property; the fixed-policy benchmark in
Appendix~\ref{sec:hloop-results} measures latency and selection agreement.

\section{Analysis of Policy-in-the-Loop Refinement}
\label{sec:theory}

This section characterizes how execution feedback governs policy continuation.
We first express refinement as a finite-budget two-level problem, then derive
the protected selection property, the complementary roles of tracking rewards
and practice distributions, and the scheduling invariance of heterogeneous
feedback execution. The analysis concerns the two-stage matched-budget
profile and the joint verification rule in
Equations~\ref{eq:main-joint-guard} and~\ref{eq:main-progress}.

\subsection{Finite-Budget Two-Level Optimization}
\label{sec:theory-bilevel}

Let $\theta_0$ be the shared 3,000-update warmstart. Stage $k\in\{1,2\}$
chooses one bounded supervision patch $U_k=(W_k,C_k)$ from the protected
policy's selection records and preceding decisions. Incomplete execution
invokes failure-conditioned proposals; completed execution invokes precision
proposals, with native-objective replay available under the prescribed
first-stage tradeoff predicate. Let $\mathcal A_b$ denote the prefix of a PPO
continuation trajectory, carrying the parent checkpoint's optimizer and
adaptive learning-rate state. The three saved candidates are
\begin{equation}
  \begin{gathered}
  \widetilde\theta_{k,j}=\mathcal A_{b_j}(p_k;R,W_k,C_k,\xi_k),\\
  (b_1,b_2,b_3)=(500,1000,1500),\qquad
  p_k=\theta_{3(k-1)}.
  \end{gathered}
  \label{eq:theory-continuation}
\end{equation}
Here $\xi_k$ specifies the stage's training randomness and $R$ is fixed.
The three checkpoints are prefixes of one 1,500-update trajectory, not
independent training branches. After that trajectory finishes, selection
processes its checkpoints in update order. Write
$c_{3(k-1)+j}=\widetilde\theta_{k,j}$, and let $\theta_i$ denote the incumbent
after slot $i$. Thus stage two starts from $\theta_3$, which may be an
intermediate first-stage checkpoint or the warmstart.

For two recorded policies $a,c$, let $J(a,c)$ denote the canonical-error and
common-physical-frontier checks in Equation~\ref{eq:main-joint-guard}, and
$P(a,c)$ the frontier-or-precision
progress condition in Equation~\ref{eq:main-progress}. Let $G_{\rm nat}(a,c)$
denote the native safety checks: positive candidate evidence, per-seed native
horizon preservation, and tracking checks with absolute startup bounds where
required.
The sequential outer decision is
\begin{equation}
  \begin{aligned}
  A_i={}&G_{\rm nat}(\theta_{i-1},c_i)\land J(\theta_{i-1},c_i)\\
       &\land J(\theta_0,c_i)\land P(\theta_{i-1},c_i),\\
  \theta_i={}&
  \begin{cases}
    c_i,& A_i\text{ holds},\\
    \theta_{i-1},&\text{otherwise}.
  \end{cases}
  \end{aligned}
  \label{eq:theory-outer}
\end{equation}
The donor check $J(\theta_0,c_i)$ applies the joint guard directly against the
warmstart; it does not require a second native progress test. All six slots
and both training stages are consumed, including when the incumbent is
retained. The finite inner trajectory and empirical outer acceptance rule
give a two-level formulation~\citep{franceschi2018bilevel}, without replacing
PPO by an exact inner optimum or selection by a lexicographic maximizer.

\subsection{Protected Joint Verification}
\label{sec:theory-gate}

Fix the reference, verification rules, and ordered selection seeds, and
condition on the recorded rollouts. To unpack the main text's frontier
notation, let $h_s(\theta)$ count valid native control steps before the first
reset for seed $s$, and $f_s(\theta)$ count common
physical pre-failure frames, including phase zero. For horizon $T$, the two
counts lie in $[0,T]$ and $[0,T+1]$, respectively. The joint comparison of
$a$ and $c$ uses control frames $1,\ldots,n_s(a,c)$, where
\[
 n_s(a,c)=\min\{h_s(a),h_s(c),[f_s(a)-1]_+,[f_s(c)-1]_+\}.
\]
Phase zero and reset states are excluded. On each nonempty interval, $J$
checks both the temporal mean and temporal P95 of world-body and root-local
FK body error for every seed, with numerical tolerance $\varepsilon=10^{-12}$.
An empty interval contributes no error statistic; native startup checks
govern recovery from a zero-length native prefix.

\begin{mimicxproposition}[Empirical acceptance and retention]
\label{prop:verified-monotonicity}
Every processed selection slot satisfies, for every matched seed $s$,
\begin{equation}
  \begin{aligned}
  h_s(\theta_i)&\ge h_s(\theta_{i-1}),\\
  f_s(\theta_i)&\ge f_s(\theta_{i-1}).
  \end{aligned}
  \label{eq:theory-monotone}
\end{equation}
An accepted candidate passes the native checks and both incumbent and donor
joint guards. Each of its four guarded error statistics is at most the
corresponding comparator statistic plus $\varepsilon$, on that pair's
nonempty comparison interval. Acceptance additionally requires a strict
increase in at least one native or physical frontier, or a greater-than-2\%
decrease in the sum of the two mean errors over matched seeds with nonempty
intervals. Otherwise the incumbent checkpoint and its records are retained.
\end{mimicxproposition}
\begin{proof}
If $A_i$ is false, Equation~\ref{eq:theory-outer} preserves the checkpoint
and its records, giving equality in Equation~\ref{eq:theory-monotone}.
If $A_i$ is true, $G_{\rm nat}$ preserves every native frontier, while
$J(\theta_{i-1},c_i)$ preserves every physical frontier and bounds the four
error statistics. The separate factor $J(\theta_0,c_i)$ gives the donor
bounds on the donor comparison intervals. Finally $P$ supplies the strict
frontier or aggregate-precision improvement.
\end{proof}

Frontier preservation composes across slots. Error comparisons instead use
pair-specific intervals: when execution extends, the next comparison can
include newly reached frames. Mean and P95 dominance on an extended interval
need not imply dominance on an earlier prefix. The direct donor guard checks
that earlier reference point explicitly, rather than inferring it by
transitivity. These are properties of the stored, matched-seed comparisons;
held-out seeds separately measure full-motion performance.

\begin{mimicxcorollary}[Reward-independent verification]
\label{cor:reward-priority}
For fixed execution records, changing training or evaluation reward values
alone cannot change an acceptance decision. In particular, increased reward
cannot override a native, joint, or donor guard violation.
\end{mimicxcorollary}
\begin{proof}
Every factor of $A_i$ depends on execution frontiers and tracking measurements,
not reward values. Holding those records fixed therefore leaves $A_i$
unchanged.
\end{proof}

\subsection{Complementarity of Objective and Curriculum}
\label{sec:theory-supervision}

\paragraph{Practice changes which states contribute to learning.}
Let $\mu_0$ be the reset-state distribution for a full-start rollout and
$\mu_{\mathcal W}$ the distribution induced by reference-state initialization
inside the practice window. Both include the prescribed reset randomization.
Conditioning on a stage's chosen $W$ and window, define
\begin{equation}
  \begin{gathered}
  J_{\mu,W}(\theta)=\mathbb E_{x_0\sim\mu,\,\tau\sim\pi_\theta}
  \left[\sum_{t=0}^{L-1}\gamma^t r_t(W)\right],
  \\
  \mu_\rho=(1-\rho)\mu_0+\rho\mu_{\mathcal W},
  \end{gathered}
  \label{eq:theory-mixture-return}
\end{equation}
where $L$ is the rollout budget, rewards after termination are zero, and
$\gamma\in[0,1]$. Episode time $t$ differs from reference phase.

\begin{mimicxproposition}[Mixture objective and window exposure]
\label{prop:mixture}
For a fixed policy and supervision, the expected return satisfies
\begin{equation}
  J_{\mu_\rho,W}=(1-\rho)J_{\mu_0,W}+\rho J_{\mu_{\mathcal W},W}.
  \label{eq:theory-return-decomposition}
\end{equation}
Where these objectives are differentiable and $\rho$, $W$, and the reset
distributions are held fixed, their gradients satisfy the same mixture
identity. If $p_{\mathcal W}$ is the probability that a full-start rollout
visits the practice window, counting initialization inside it as a visit, then
\begin{equation}
  p_\rho=\rho+(1-\rho)p_{\mathcal W},\qquad
  p_\rho-p_{\mathcal W}=\rho(1-p_{\mathcal W})\ge0.
  \label{eq:theory-window-exposure}
\end{equation}
\end{mimicxproposition}
\begin{proof}
Condition on the Bernoulli choice of reset distribution. The law of total
expectation gives Equation~\ref{eq:theory-return-decomposition}; differentiation
of this finite weighted sum gives the gradient identity. The visit probability is one under window initialization; under a full start,
it is $p_{\mathcal W}$, yielding Equation~\ref{eq:theory-window-exposure}.
\end{proof}

For a localized window, the exposure gain is largest when the current policy
seldom reaches the diagnosed transition. For example, with $p_{\mathcal W}=0.1$, the two proposal
probabilities $\rho=0.45$ and $0.65$ give visit probabilities $0.505$ and
$0.685$, respectively. This calculation concerns opportunities to practice;
successful execution through the window is measured by the subsequent
full-start verification. The gradient identity characterizes the objective
sampled by continuation, with PPO supplying its finite-budget update rule.
For full-motion replay, initialization changes the reference-phase distribution;
the entire-motion visit probability is already one. The identities condition
on the chosen patch rather than differentiating through diagnosis or selection.

\paragraph{Reward parameters change sensitivity within visited states.}
For one positional group, write
$E_G=|G|^{-1}\sum_{b\in G}\|e_b\|^2$ and
$r_G=w_G\exp(-E_G/\sigma_G^2)$, where
$e_b=p_b-\widetilde p_b$. Holding the error coordinates independent gives
\begin{equation}
  \frac{\partial r_G}{\partial e_b}
  =-\frac{2w_G}{|G|\sigma_G^2}
  \exp(-E_G/\sigma_G^2)e_b.
  \label{eq:theory-reward-sensitivity}
\end{equation}
This is the reward's sensitivity to tracking error. Policy optimization uses
sampled returns and policy gradients~\citep{sutton1999policygradient}, rather
than differentiating this expression through the simulator. Increasing $w_G$
scales the sensitivity, while changing $\sigma_G$ changes which error
magnitudes receive it. For fixed $E_G>0$, the scale-dependent factor
$u^{-1}\exp(-E_G/u)$, with $u=\sigma_G^2$, has derivative
$\exp(-E_G/u)(E_G-u)/u^3$ and is maximized at $u=E_G$. Thus reducing a scale
indefinitely does not indefinitely strengthen supervision at a nonzero error.
This motivates the bounded joint search over weights and scales in
Section~\ref{sec:method}.

For the nonnegative tracking terms indexed by
$\mathcal I=\mathcal L\cup\mathcal G\cup\mathcal D$, Taylor's theorem yields
the explicit quadratic approximation
\begin{equation}
  \begin{gathered}
  \sum_{i\in\mathcal I} w_i e^{-E_i/\sigma_i^2}
  =\sum_i w_i-\sum_i\frac{w_i}{\sigma_i^2}E_i+\mathcal R,
  \\
  0\le\mathcal R\le\frac12\sum_i w_i
           \left(\frac{E_i}{\sigma_i^2}\right)^2.
  \end{gathered}
  \label{eq:theory-quadratic}
\end{equation}
Indeed, $e^{-x}=1-x+e^{-\zeta}x^2/2$ for some $\zeta\in[0,x]$ when $x\ge0$;
summing with $w_i\ge0$ gives the bound. Locally, the tracking objective
therefore assigns effective coefficients $w_i/\sigma_i^2$ to local,
global, and dynamic squared discrepancies. Action smoothing and physical
penalties remain additional terms as in Equation~\ref{eq:main-objective}.
Curriculum modifies exposure to these discrepancies, while objective
refinement modifies their relative sensitivity. The two operations act on
different parts of the learning problem, providing a rationale for their
joint adaptation within a stage.

\subsection{Decision-Preserving Heterogeneous Execution}
\label{sec:theory-hloop}

Represent the fixed-checkpoint feedback workload by a directed acyclic graph
$\mathcal G=(V,E)$. Each node $v$ consumes immutable external inputs $a_v$,
outputs from its predecessors, and a node-local random stream $\xi_v$:
\begin{equation}
  y_v=F_v\!\left(a_v,\{y_u:u\in\operatorname{pred}(v)\},\xi_v\right).
  \label{eq:theory-node}
\end{equation}
Assume these functions produce schedule-independent semantic outputs for the
same inputs and random streams, with no shared mutable state between
independent nodes. Timing and log-order metadata are not part of $y_v$.
The selector consumes all required reports in the prescribed checkpoint order
and uses a deterministic acceptance rule. These conditions specify the reproducible workload
whose execution is rescheduled; no policy update is inserted into it.

\begin{samepage}
\begin{mimicxproposition}[Scheduling invariance]
\label{prop:schedule-invariance}
Two dependency-respecting schedules completing the same graph under these conditions produce identical semantic outputs and decisions.
\end{mimicxproposition}
\begin{proof}
Proceed by topological induction. Source nodes receive identical external
inputs and random streams, so their outputs agree. If every predecessor of a
node has matching outputs, Equation~\ref{eq:theory-node} supplies the same
arguments to its function in both schedules, and its output also agrees.
The graph is finite and acyclic, so this covers all nodes. In particular, the
selector receives identical ordered reports; deterministic selection agrees.
\end{proof}
\end{samepage}

The assumption concerns reproducible node outputs, not merely equal seed
integers. It can be checked independently of elapsed time by comparing
semantic outputs and decisions. Table~\ref{tab:hloop} reports identical selection outcomes for the measured feedback workload,
alongside a median $1.945\times$ speedup over sequential execution.

\paragraph{Where the time savings arise.}
For fixed node service times $d_v$, suppose each node occupies one slot of
its assigned resource class $r\in\{\mathrm{GPU},\mathrm{CPU},\mathrm{I/O}\}$,
with $m_r$ available slots. Write $W_r=\sum_{v:r(v)=r}d_v$, and let
$L_{\rm crit}=\max_{P}\sum_{v\in P}d_v$ be the longest dependency-path duration.
Every valid schedule has makespan
\begin{equation}
  T_{\rm sched}\ge
  \max\!\left\{L_{\rm crit},\max_r\frac{W_r}{m_r}\right\}.
  \label{eq:theory-runtime-bound}
\end{equation}
Nodes on one path must run in sequence, giving the first bound; a class with
$m_r$ slots can supply at most $m_rT_{\rm sched}$ slot-time, giving the second.
Sequential execution pays $\sum_vd_v$ before orchestration overhead, whereas
heterogeneous execution can overlap independent rollout and diagnosis work.
The bound identifies the critical-path and resource-load constraints; the
measured wall times include actual scheduling overhead. HLoop thus accelerates
feedback by changing when independent work executes while preserving what
information reaches selection.

\section{Matched-Budget Motion Study}
\label{sec:current-protocol}

\subsection{Motion Inputs and Sampling Units}

\begin{table*}[t]
\centering
\caption{\textbf{Execution, training cost and spatial precision.} Twelve-task coverage and two dynamic-motion comparisons.}\label{tab:precision-main}\label{tab:current-aggregate}
\begingroup\papertablestyle
\begin{tabularx}{\textwidth}{>{\hsize=1.8\hsize\linewidth=\hsize\raggedright\arraybackslash}X*{8}{>{\hsize=.9\hsize\linewidth=\hsize\centering\arraybackslash}X}}\toprule
\rowcolor{tablehead}
\textbf{Method} & \multicolumn{2}{c}{\textbf{All 12 motions}} & \multicolumn{3}{c}{\textbf{Broad Jump}} & \multicolumn{3}{c}{\textbf{Heel-and-Toe}} \\
\rowcolor{tablehead}
 & Complete & GPU-h & Complete & Mean (cm) & P95 (cm) & Complete & Mean (cm) & P95 (cm) \\
\midrule
\rowcolor{basebg}
BeyondMimic (MjLab) & 120/144 & 42.50 & 0/6 & NC & NC & 6/6 & 7.31 & 12.50 \\
Static Search & \textbf{132/144} & 42.91 & 6/6 & 8.04 & 21.78 & 6/6 & 7.91 & 16.71 \\
\rowcolor{oursbg}
MimicX & \textbf{132/144} & 42.71 & 6/6 & \textbf{5.56} & \textbf{11.13} & 6/6 & \textbf{6.45} & \textbf{10.99} \\
Released SONIC & 60/72 & n/a & 0/3 & NC & NC & 3/3 & 22.19 & 38.94 \\
\bottomrule\end{tabularx}
\par\smallskip\begin{minipage}{\textwidth}\papertablenote All-task completion combines nominal and perturbed conditions. Native methods: 72k continuation updates, two training seeds and three held-out repeats per condition; costs include selection/evaluation but exclude the shared warmstart (35.62 GPU-h). SONIC is execution-only. Dynamic-motion errors use nominal full-horizon world-space mean/time-P95, averaged over repeats; lower is better. NC: incomplete execution. Bold marks best native coverage and world errors. Full results: Appendix~\ref{sec:current-protocol}.\par\smallskip\hrule height .4pt\end{minipage}
\endgroup\end{table*}

The controlled study uses twelve prepared robot references from ten BONES
motions, one AMASS motion and one LaFAN1 motion. Each source contributes one
scored window, held fixed across methods. The tasks are Turn and Jog,
Side Step, Forward Jog, Rubbing Eyes, Praying/Salute, Heel-and-Toe Dance,
Crouch Transition, Pointing Down, Dodge Backward, Injured-Torso Walk, Broad
Jump and Get-up. Scored durations range from 2.64 to 16.50 seconds. Motion
phase advances at 50\,Hz through the final frame without wrapping. These
inputs isolate policy-side supervision after the common reference interface;
the five video-derived collision-scene pairs separately exercise reference
construction and execution.

Training seeds are 101 and 202. For each task and seed, the native methods
share a separately trained 3,000-update warmstart. Three native methods
produce 72 selected endpoints, each evaluated with three repeats under each
of two conditions. This gives 432 native final executions. Released SONIC
contributes 72 further executions, with three repeats per task and condition.
Repeated execution seeds measure variability of a given trained policy;
training seeds distinguish optimization runs. All methods completed every
repeat on ten motions under both conditions. On Broad Jump, Static Search
and MimicX completed all six repeats per condition, whereas BeyondMimic
and released SONIC did not complete; no method completed Get-up.

\subsection{Training and Selection Budgets}

Each PPO update collects 24 control steps from 4,096 environments. A
3,000-update continuation therefore consumes 294.9 million policy transitions.
BeyondMimic/MjLab follows one continuous optimization trajectory. Static
Search uses two predetermined 1,500-update supervision branches, both from
the warmstart. MimicX uses two 1,500-update stages with execution-conditioned
supervision and protected parent selection. All methods expose six checkpoints
at 500-update intervals and apply the same selection rule. An intermediate
checkpoint or the initial policy can be retained even though the entire
assigned training budget is consumed.

Selection uses seeds 1001, 1002 and 1003. Final nominal evaluations use
21001, 21002 and 21003; perturbed evaluations use 22001, 22002 and 22003.
The joint-position initialization ranges are respectively
$[-0.005,0.005]$ and $[-0.01,0.01]$ radians; selection uses the nominal range.
Checkpoint selection is finalized before either final-evaluation set is
executed. The same seeds and input definitions apply across compared methods.

Across the study, each native method consumes 72,000 continuation
updates. The shared warmstart pool consumes a further 72,000 updates,
counted once. Measured assigned GPU phase-hours are 35.6246 for this pool,
42.5017 for BeyondMimic, 42.9105 for Static Search and 42.7109 for MimicX.
These phase measurements include training, checkpoint selection and final
evaluation overhead within the assigned job. Released SONIC's 1.1358
assigned GPU hours measure the 72 execution jobs; its upstream pretraining
cost is outside this measurement.

\subsection{Common Physical and Kinematic Evaluation}

All controllers are measured against the original reference on a common
50\,Hz source clock and the same 29 named G1 joints and 30 bodies. Completion
combines the common physical failure predicate with valid recorded coverage;
native termination also truncates native execution. Recording later states
or restarting an episode does not extend the uninterrupted frontier. The
joint verification rule excludes the initialization sample and every sample
from the first reset onward.

The common failure predicate tests root-height error above $0.25$\,m,
root-relative ankle or wrist height error above $0.25$\,m, or root orientation
error above $1$\,rad. A violation persisting for five consecutive frames
sets the frontier to its onset. Native termination additionally uses its
fixed tracking conditions. The full interval is scored after support is
released, without per-method pose fitting or temporal warping.

For $B$ common bodies, instantaneous world error is
\begin{equation}
 e_t^{\rm world}=\frac{1}{B}\sum_{b=1}^B
 \|p_{t,b}-\bar p_{t,b}\|_2.
 \label{eq:current-world-error}
\end{equation}
For articulation, forward kinematics is recomputed from actual and reference
joint positions with both root poses set to the same identity transform:
\begin{equation}
 e_t^{\rm local}=\frac{1}{B}\sum_{b=1}^B
 \|\operatorname{FK}_b(q_t)-\operatorname{FK}_b(\bar q_t)\|_2.
 \label{eq:current-local-error}
\end{equation}
Local error measures configuration, while world error retains root
displacement and orientation. Temporal mean and P95 are computed per rollout,
then averaged over the declared repeats and training seeds. P95 denotes
the temporal 95th percentile. Relative improvements are computed from
task-level means across seeds.
Figure~\ref{fig:current-local-fidelity} compares mean local tracking across
motions and shows rollout-level local P95 for Broad Jump and Heel-and-Toe.
Figure~\ref{fig:current-quality} combines the
full-cohort comparisons with the per-rollout Heel-and-Toe distribution;
Table~\ref{tab:precision-main} details the two motions under both conditions.

The full-horizon cell includes all declared executions. If any execution
fails, its aggregate full-horizon error is unavailable; completion still
reports that cell. A separately declared common-prefix measurement compares
earlier source phases and has its own valid-coverage checks. Completion
counts include all twelve tasks; full-horizon tables list tasks with at least
one defined comparison cohort. Native
precision qualification additionally requires world-anchor mean/P95
at most $0.08/0.15$\,m, aligned worst-body mean/P95 at most $0.10/0.20$\,m,
and mean anchor-rotation error at most $0.20$\,rad in every repeat. These
native readouts are separate from the common-body FK metrics above.

\subsection{Selection Evidence and Retained Policies}

% Generated table inputs; include from the paper root.
% Generated by scripts/build_current_tables.py; do not edit by hand.
\begin{table*}[t]
\centering
\caption{\textbf{Measured work and phase-cost accounting.}}\label{tab:current-costs}
\begingroup
\papertablestyle
\setbox0=\hbox{%
\begin{tabularx}{\textwidth}{>{\hsize=1.7\hsize\linewidth=\hsize\raggedright\arraybackslash}X>{\hsize=0.55\hsize\linewidth=\hsize\centering\arraybackslash}X>{\hsize=0.9\hsize\linewidth=\hsize\centering\arraybackslash}X>{\hsize=1.35\hsize\linewidth=\hsize\centering\arraybackslash}X>{\hsize=0.8\hsize\linewidth=\hsize\centering\arraybackslash}X>{\hsize=0.8\hsize\linewidth=\hsize\centering\arraybackslash}X>{\hsize=0.9\hsize\linewidth=\hsize\centering\arraybackslash}X}\toprule
\rowcolor{tablehead}
\textbf{Phase / method} & \textbf{Jobs} & \textbf{Updates} & \textbf{Transitions} & \makecell{\textbf{Selection}\\\textbf{slots}} & \makecell{\textbf{Selection}\\\textbf{rollouts}} & \makecell{\textbf{Assigned}\\\textbf{GPU-h}} \\
\midrule
\rowcolor{tablegroup}
\multicolumn{7}{l}{\textit{Shared initialization}} \\
Shared warmstart & 24 & 72,000 & 7,077,888,000 & 0 & 0 & 35.6246 \\
\rowcolor{tablegroup}
\multicolumn{7}{l}{\textit{Matched native continuations}} \\
\rowcolor{basebg}
BeyondMimic (MjLab) & 24 & 72,000 & 7,077,888,000 & 144 & 432 & 42.5017 \\
Static Search & 24 & 72,000 & 7,077,888,000 & 144 & 432 & 42.9105 \\
\rowcolor{oursbg}
MimicX & 24 & 72,000 & 7,077,888,000 & 144 & 432 & 42.7109 \\
\midrule
Native campaign & 96 & 288,000 & 28,311,552,000 & 432 & 1,296 & 163.7477 \\
\midrule
\rowcolor{tablegroup}
\multicolumn{7}{l}{\textit{Released-policy execution}} \\
Released SONIC evaluation & 72 & n/a & n/a & n/a & n/a & 1.1358 \\
\bottomrule
\end{tabularx}%
}
\ifdefined\CurrentTableAudit
\typeout{CURRENTTABLE|s04_costs|\the\wd0|\the\dimexpr\ht0+\dp0\relax|\the\textwidth}
\fi
\usebox0
\par\vspace{2pt}
\begin{minipage}{\textwidth}\papertablenote Native jobs are task $\times$ training seed. Every phase uses 3,000 PPO updates/job; one update has $4,096\times24=98,304$ transitions. Shared warmstarts are counted once. Selection slots are checkpoints, not independent training runs; the 72 shared-donor selection rollouts are additional to the 1,296 candidate rollouts. Final testing comprises 432 native and 72 SONIC rollouts. Assigned phase wall time includes training, selection, and evaluation overhead; it is neither GPU kernel time nor end-to-end video-processing time. SONIC's 72 jobs are released-policy executions, not training; its upstream pretraining cost is unknown. SONIC prelaunch overhead (0.0064 h) is separate.\end{minipage}
\par\nointerlineskip\vspace{2pt}\noindent\rule{\textwidth}{0.4pt}
\endgroup
\end{table*}

% Generated by scripts/build_current_tables.py; do not edit by hand.
\begin{table*}[t]
\centering
\caption{\textbf{Full-horizon tracking, nominal.} Motions with at least one complete method cohort.}\label{tab:current-full-nominal}
\begingroup
\papertablestyle
\setbox0=\hbox{%
\begin{tabularx}{\textwidth}{>{\raggedright\arraybackslash}p{0.17\textwidth}>{\hsize=1\hsize\linewidth=\hsize\centering\arraybackslash}X>{\hsize=1\hsize\linewidth=\hsize\centering\arraybackslash}X>{\hsize=1\hsize\linewidth=\hsize\centering\arraybackslash}X>{\hsize=1\hsize\linewidth=\hsize\centering\arraybackslash}X>{\hsize=1\hsize\linewidth=\hsize\centering\arraybackslash}X>{\hsize=1\hsize\linewidth=\hsize\centering\arraybackslash}X>{\hsize=1\hsize\linewidth=\hsize\centering\arraybackslash}X>{\hsize=1\hsize\linewidth=\hsize\centering\arraybackslash}X}\toprule
\rowcolor{tablehead}
\textbf{Task} & \multicolumn{2}{c}{\cellcolor{tablehead}\textbf{BeyondMimic (MjLab)}} & \multicolumn{2}{c}{\cellcolor{tablehead}\textbf{Static Search}} & \multicolumn{2}{c}{\cellcolor{oursbg}\textbf{MimicX}} & \multicolumn{2}{c}{\cellcolor{tablehead}\textbf{Released SONIC}} \\
\cmidrule(lr){2-3}\cmidrule(lr){4-5}\cmidrule(lr){6-7}\cmidrule(lr){8-9}
\rowcolor{tablehead}
 & World $\downarrow$ & Root-local $\downarrow$ & World $\downarrow$ & Root-local $\downarrow$ & World $\downarrow$ & Root-local $\downarrow$ & World $\downarrow$ & Root-local $\downarrow$ \\
\midrule
\rowcolor{tablegroup}
\multicolumn{9}{l}{\textit{BONES}} \\
Turn and Jog & \textbf{8.49} / \textbf{13.80} & 3.12 / \textbf{4.60} & 10.23 / 18.48 & 3.20 / 4.77 & 9.38 / 14.40 & \textbf{3.07} / 4.64 & 12.69 / 18.07 & 2.23 / 2.89 \\
Forward Jog & 4.94 / 7.06 & \textbf{1.79} / \textbf{2.45} & \textbf{4.60} / \textbf{6.59} & 1.96 / 2.83 & 5.42 / 7.52 & 1.79 / 2.58 & 4.87 / 7.46 & 2.42 / 3.35 \\
Rubbing Eyes & 1.64 / 2.19 & 1.39 / 2.10 & 1.64 / 2.19 & 1.39 / 2.10 & \textbf{1.63} / \textbf{2.19} & \textbf{1.39} / \textbf{2.09} & 2.77 / 3.76 & 2.12 / 3.24 \\
Praying/Salute & 3.46 / 5.90 & \textbf{2.56} / \textbf{3.61} & \textbf{3.27} / \textbf{5.14} & 2.69 / 3.83 & 3.34 / 5.71 & 2.61 / 3.63 & 7.63 / 11.05 & 3.36 / 5.08 \\
Heel-and-Toe Dance & 7.31 / 12.50 & \textbf{3.17} / 4.59 & 7.91 / 16.71 & 3.40 / 4.82 & \textbf{6.45} / \textbf{10.99} & 3.23 / \textbf{4.37} & 22.19 / 38.94 & 3.62 / 5.04 \\
Crouch Transition & 2.67 / 3.62 & 3.31 / 3.77 & 2.11 / 3.04 & \textbf{2.68} / \textbf{3.31} & \textbf{2.06} / \textbf{2.90} & 2.74 / 3.47 & 10.40 / 12.94 & 6.26 / 6.86 \\
Pointing Down & \textbf{2.96} / \textbf{4.94} & \textbf{1.30} / 1.75 & 3.05 / 5.47 & 1.34 / 1.74 & 3.07 / 5.42 & 1.34 / \textbf{1.74} & 3.14 / 4.40 & 1.85 / 2.38 \\
Dodge Backward & 5.93 / 9.31 & \textbf{2.31} / \textbf{4.24} & 7.66 / 11.54 & 2.50 / 4.47 & \textbf{5.78} / \textbf{9.09} & 2.35 / 4.30 & 25.55 / 36.33 & 2.74 / 4.59 \\
Injured-Torso Walk & 2.82 / 4.39 & 1.57 / 2.19 & 2.71 / 4.19 & 1.62 / 2.23 & \textbf{2.70} / \textbf{4.08} & \textbf{1.53} / \textbf{2.15} & 6.96 / 12.05 & 3.02 / 3.74 \\
Broad Jump & n/a & n/a & 8.04 / 21.78 & \textbf{4.63} / \textbf{8.62} & \textbf{5.56} / \textbf{11.13} & 4.80 / 8.72 & n/a & n/a \\
\rowcolor{tablegroup}
\multicolumn{9}{l}{\textit{AMASS}} \\
Side Step & 7.11 / 8.37 & 2.31 / 4.24 & \textbf{5.26} / \textbf{7.53} & \textbf{2.12} / \textbf{3.92} & 7.11 / 8.37 & 2.31 / 4.24 & 18.80 / 39.10 & 2.23 / 3.28 \\
\bottomrule
\end{tabularx}%
}
\ifdefined\CurrentTableAudit
\typeout{CURRENTTABLE|s01_full_nominal|\the\wd0|\the\dimexpr\ht0+\dp0\relax|\the\textwidth}
\fi
\usebox0
\par\vspace{2pt}
\begin{minipage}{\textwidth}\papertablenote Errors in cm, mean / time-P95; P95 is computed per rollout over time and then averaged, not pooled. All predetermined repeats are retained: six per native method and three for released SONIC per task/condition. n/a denotes an undefined aggregate when any required repeat lacks the stated window; no complete-case filtering. World and root-local errors use the common 30-body canonical representation. Bold marks the lowest defined native error per task, channel, and statistic, including ties.\end{minipage}
\par\nointerlineskip\vspace{2pt}\noindent\rule{\textwidth}{0.4pt}
\endgroup
\end{table*}

% Generated by scripts/build_current_tables.py; do not edit by hand.
\begin{table*}[t]
\centering
\caption{\textbf{Full-horizon tracking, perturbed.} Motions with at least one complete method cohort.}\label{tab:current-full-perturbed}
\begingroup
\papertablestyle
\setbox0=\hbox{%
\begin{tabularx}{\textwidth}{>{\raggedright\arraybackslash}p{0.17\textwidth}>{\hsize=1\hsize\linewidth=\hsize\centering\arraybackslash}X>{\hsize=1\hsize\linewidth=\hsize\centering\arraybackslash}X>{\hsize=1\hsize\linewidth=\hsize\centering\arraybackslash}X>{\hsize=1\hsize\linewidth=\hsize\centering\arraybackslash}X>{\hsize=1\hsize\linewidth=\hsize\centering\arraybackslash}X>{\hsize=1\hsize\linewidth=\hsize\centering\arraybackslash}X>{\hsize=1\hsize\linewidth=\hsize\centering\arraybackslash}X>{\hsize=1\hsize\linewidth=\hsize\centering\arraybackslash}X}\toprule
\rowcolor{tablehead}
\textbf{Task} & \multicolumn{2}{c}{\cellcolor{tablehead}\textbf{BeyondMimic (MjLab)}} & \multicolumn{2}{c}{\cellcolor{tablehead}\textbf{Static Search}} & \multicolumn{2}{c}{\cellcolor{oursbg}\textbf{MimicX}} & \multicolumn{2}{c}{\cellcolor{tablehead}\textbf{Released SONIC}} \\
\cmidrule(lr){2-3}\cmidrule(lr){4-5}\cmidrule(lr){6-7}\cmidrule(lr){8-9}
\rowcolor{tablehead}
 & World $\downarrow$ & Root-local $\downarrow$ & World $\downarrow$ & Root-local $\downarrow$ & World $\downarrow$ & Root-local $\downarrow$ & World $\downarrow$ & Root-local $\downarrow$ \\
\midrule
\rowcolor{tablegroup}
\multicolumn{9}{l}{\textit{BONES}} \\
Turn and Jog & \textbf{4.89} / \textbf{9.15} & 3.20 / \textbf{4.58} & 5.82 / 12.11 & 3.25 / 4.79 & 5.26 / 9.31 & \textbf{3.15} / 4.64 & 12.53 / 17.99 & 2.23 / 2.89 \\
Forward Jog & 5.09 / 7.23 & \textbf{1.71} / \textbf{2.36} & 5.78 / 8.75 & 1.86 / 2.61 & \textbf{4.73} / \textbf{6.65} & 1.74 / 2.48 & 4.84 / 7.26 & 2.42 / 3.33 \\
Rubbing Eyes & 1.43 / 1.82 & 1.43 / 2.12 & 1.43 / 1.83 & 1.43 / \textbf{2.12} & \textbf{1.42} / \textbf{1.82} & \textbf{1.43} / 2.12 & 2.85 / 4.23 & 2.19 / 3.62 \\
Praying/Salute & 3.56 / 5.65 & \textbf{2.54} / \textbf{3.54} & \textbf{3.31} / \textbf{4.96} & 2.71 / 3.86 & 3.36 / 5.39 & 2.61 / 3.62 & 6.73 / 9.94 & 3.33 / 5.13 \\
Heel-and-Toe Dance & 6.57 / 9.86 & \textbf{3.16} / 4.53 & 7.66 / 13.21 & 3.36 / 4.76 & \textbf{5.55} / \textbf{8.85} & 3.25 / \textbf{4.34} & 18.41 / 38.67 & 3.63 / 5.10 \\
Crouch Transition & 2.91 / 3.83 & 3.42 / 3.91 & 2.14 / \textbf{2.96} & \textbf{2.80} / \textbf{3.40} & \textbf{2.08} / 2.97 & 2.89 / 3.66 & 11.09 / 13.70 & 6.26 / 6.93 \\
Pointing Down & \textbf{3.05} / \textbf{5.24} & \textbf{1.38} / \textbf{1.76} & 3.15 / 5.91 & 1.40 / 1.76 & 3.14 / 5.90 & 1.40 / 1.76 & 3.31 / 4.73 & 1.86 / 2.38 \\
Dodge Backward & 5.17 / 9.01 & \textbf{2.35} / 4.32 & 6.69 / 10.93 & 2.55 / 4.55 & \textbf{4.90} / \textbf{8.81} & 2.37 / \textbf{4.25} & 25.04 / 35.60 & 2.74 / 4.59 \\
Injured-Torso Walk & 3.02 / 4.63 & 1.50 / 2.14 & 2.96 / 4.49 & 1.56 / 2.17 & \textbf{2.80} / \textbf{4.27} & \textbf{1.47} / \textbf{2.13} & 7.66 / 13.50 & 3.03 / 3.74 \\
Broad Jump & n/a & n/a & 8.00 / 19.33 & \textbf{4.60} / \textbf{8.43} & \textbf{5.39} / \textbf{10.23} & 4.86 / 8.54 & n/a & n/a \\
\rowcolor{tablegroup}
\multicolumn{9}{l}{\textit{AMASS}} \\
Side Step & 7.18 / 8.65 & 2.46 / 4.25 & \textbf{5.21} / \textbf{7.03} & \textbf{2.18} / \textbf{4.03} & 7.18 / 8.65 & 2.46 / 4.25 & 19.13 / 40.38 & 2.24 / 3.29 \\
\bottomrule
\end{tabularx}%
}
\ifdefined\CurrentTableAudit
\typeout{CURRENTTABLE|s01_full_perturbed|\the\wd0|\the\dimexpr\ht0+\dp0\relax|\the\textwidth}
\fi
\usebox0
\par\vspace{2pt}
\begin{minipage}{\textwidth}\papertablenote Errors in cm, mean / time-P95; P95 is computed per rollout over time and then averaged, not pooled. All predetermined repeats are retained: six per native method and three for released SONIC per task/condition. n/a denotes an undefined aggregate when any required repeat lacks the stated window; no complete-case filtering. World and root-local errors use the common 30-body canonical representation. Bold marks the lowest defined native error per task, channel, and statistic, including ties.\end{minipage}
\par\nointerlineskip\vspace{2pt}\noindent\rule{\textwidth}{0.4pt}
\endgroup
\end{table*}

% Generated by update_arxiv_editorial_assets.py from the frozen seed data.
\begin{table*}[t]
\centering
\caption{\textbf{Native full-horizon tracking by training seed, nominal.}}
\label{tab:current-seeds-nominal}
\begingroup\papertablestyle
\begin{tabularx}{\textwidth}{@{}ll*{6}{>{\centering\arraybackslash}X}@{}}
\toprule\rowcolor{tablehead}
\textbf{Motion} & \textbf{Seed} & \multicolumn{2}{c}{\textbf{BeyondMimic (MjLab)}} & \multicolumn{2}{c}{\textbf{Static Search}} & \multicolumn{2}{c}{\cellcolor{oursbg}\textbf{MimicX}} \\
\cmidrule(lr){3-4}\cmidrule(lr){5-6}\cmidrule(lr){7-8}
\rowcolor{tablehead}
 & & World $\downarrow$ & Root-local $\downarrow$ & World $\downarrow$ & Root-local $\downarrow$ & World $\downarrow$ & Root-local $\downarrow$ \\
\midrule
\rowcolor{tablegroup}
Turn and Jog & 101 & \textbf{9.99} / \textbf{16.63} & \textbf{3.05} / \textbf{4.58} & \textbf{9.99} / \textbf{16.63} & \textbf{3.05} / \textbf{4.58} & \textbf{9.99} / \textbf{16.63} & \textbf{3.05} / \textbf{4.58} \\
\rowcolor{tablegroup}
 & 202 & \textbf{7.00} / \textbf{10.97} & 3.19 / \textbf{4.62} & 10.47 / 20.33 & 3.34 / 4.97 & 8.78 / 12.17 & \textbf{3.09} / 4.69 \\
Side Step & 101 & 7.18 / 8.50 & 2.38 / 4.47 & \textbf{3.47} / \textbf{6.82} & \textbf{2.01} / \textbf{3.82} & 7.18 / 8.50 & 2.38 / 4.47 \\
 & 202 & \textbf{7.05} / \textbf{8.23} & \textbf{2.23} / \textbf{4.02} & \textbf{7.05} / \textbf{8.23} & \textbf{2.23} / \textbf{4.02} & \textbf{7.05} / \textbf{8.23} & \textbf{2.23} / \textbf{4.02} \\
\rowcolor{tablegroup}
Forward Jog & 101 & 5.23 / 7.38 & \textbf{1.74} / \textbf{2.38} & \textbf{4.92} / \textbf{7.20} & 1.92 / 2.81 & 5.41 / 7.56 & 1.76 / 2.44 \\
\rowcolor{tablegroup}
 & 202 & 4.65 / 6.74 & 1.84 / \textbf{2.53} & \textbf{4.28} / \textbf{5.98} & 2.01 / 2.85 & 5.42 / 7.48 & \textbf{1.83} / 2.71 \\
Rubbing Eyes & 101 & \textbf{1.62} / \textbf{2.15} & \textbf{1.50} / 2.15 & \textbf{1.62} / \textbf{2.15} & \textbf{1.50} / 2.15 & \textbf{1.62} / \textbf{2.15} & \textbf{1.50} / \textbf{2.14} \\
 & 202 & \textbf{1.65} / \textbf{2.22} & \textbf{1.28} / \textbf{2.05} & \textbf{1.65} / \textbf{2.22} & \textbf{1.28} / \textbf{2.05} & \textbf{1.65} / \textbf{2.22} & \textbf{1.28} / \textbf{2.05} \\
\rowcolor{tablegroup}
Praying/Salute & 101 & 3.14 / 5.58 & \textbf{2.39} / \textbf{3.41} & \textbf{2.77} / \textbf{4.06} & 2.66 / 3.85 & 2.91 / 5.19 & 2.48 / 3.44 \\
\rowcolor{tablegroup}
 & 202 & 3.78 / 6.22 & \textbf{2.73} / \textbf{3.81} & 3.78 / \textbf{6.21} & \textbf{2.73} / 3.82 & \textbf{3.77} / 6.22 & \textbf{2.73} / \textbf{3.81} \\
Heel-and-Toe Dance & 101 & 7.68 / 14.46 & 3.26 / 4.73 & 8.53 / 19.38 & 3.43 / 4.77 & \textbf{6.12} / \textbf{11.41} & \textbf{3.24} / \textbf{4.37} \\
 & 202 & 6.93 / \textbf{10.55} & \textbf{3.09} / 4.44 & 7.29 / 14.03 & 3.37 / 4.86 & \textbf{6.78} / 10.57 & 3.21 / \textbf{4.37} \\
\rowcolor{tablegroup}
Crouch Transition & 101 & 2.64 / 3.64 & 3.18 / 3.56 & 2.55 / 3.83 & \textbf{2.60} / 3.22 & \textbf{2.01} / \textbf{3.14} & 2.73 / \textbf{3.14} \\
\rowcolor{tablegroup}
 & 202 & 2.69 / 3.60 & 3.43 / 3.98 & \textbf{1.67} / \textbf{2.24} & 2.77 / \textbf{3.39} & 2.10 / 2.66 & \textbf{2.75} / 3.79 \\
Pointing Down & 101 & \textbf{2.81} / \textbf{4.79} & \textbf{1.27} / 1.74 & 2.99 / 5.74 & 1.34 / 1.71 & 3.03 / 5.75 & 1.35 / \textbf{1.70} \\
 & 202 & \textbf{3.11} / \textbf{5.10} & \textbf{1.33} / \textbf{1.77} & 3.12 / 5.19 & \textbf{1.33} / 1.78 & \textbf{3.11} / \textbf{5.10} & \textbf{1.33} / \textbf{1.77} \\
\rowcolor{tablegroup}
Dodge Backward & 101 & 4.97 / 8.71 & \textbf{2.18} / \textbf{4.34} & 8.41 / 13.17 & 2.58 / 4.78 & \textbf{4.70} / \textbf{8.27} & 2.26 / 4.46 \\
\rowcolor{tablegroup}
 & 202 & 6.89 / \textbf{9.92} & \textbf{2.43} / \textbf{4.15} & 6.90 / \textbf{9.92} & \textbf{2.43} / \textbf{4.15} & \textbf{6.86} / \textbf{9.92} & 2.44 / \textbf{4.15} \\
Injured-Torso Walk & 101 & 3.19 / 4.92 & 1.53 / 2.16 & 2.98 / 4.51 & 1.64 / 2.22 & \textbf{2.95} / \textbf{4.29} & \textbf{1.46} / \textbf{2.08} \\
 & 202 & 2.45 / \textbf{3.86} & \textbf{1.60} / \textbf{2.23} & \textbf{2.44} / \textbf{3.86} & \textbf{1.60} / \textbf{2.23} & 2.45 / \textbf{3.86} & \textbf{1.60} / \textbf{2.23} \\
\rowcolor{tablegroup}
Broad Jump & 101 & \multicolumn{2}{c}{\textsc{NC}} & 4.55 / \textbf{8.54} & \textbf{4.40} / \textbf{8.25} & \textbf{4.48} / 9.35 & 4.73 / 8.98 \\
\rowcolor{tablegroup}
 & 202 & \multicolumn{2}{c}{\textsc{NC}} & 11.52 / 35.03 & \textbf{4.87} / 8.99 & \textbf{6.64} / \textbf{12.91} & \textbf{4.87} / \textbf{8.47} \\
\bottomrule
\end{tabularx}
\par\vspace{2pt}
\begin{minipage}{\textwidth}\papertablenote Mean / temporal P95 in cm, averaged over the
three held-out repeats per seed. NC marks an incomplete cohort; completion
patterns are summarized in Section~\ref{sec:current-protocol}. Bold marks the lowest
defined value within each task and seed at the shown precision.\end{minipage}
\par\nointerlineskip\vspace{2pt}\noindent\rule{\textwidth}{0.4pt}
\endgroup
\end{table*}

% Generated by update_arxiv_editorial_assets.py from the frozen seed data.
\begin{table*}[t]
\centering
\caption{\textbf{Native full-horizon tracking by training seed, perturbed.}}
\label{tab:current-seeds-perturbed}
\begingroup\papertablestyle
\begin{tabularx}{\textwidth}{@{}ll*{6}{>{\centering\arraybackslash}X}@{}}
\toprule\rowcolor{tablehead}
\textbf{Motion} & \textbf{Seed} & \multicolumn{2}{c}{\textbf{BeyondMimic (MjLab)}} & \multicolumn{2}{c}{\textbf{Static Search}} & \multicolumn{2}{c}{\cellcolor{oursbg}\textbf{MimicX}} \\
\cmidrule(lr){3-4}\cmidrule(lr){5-6}\cmidrule(lr){7-8}
\rowcolor{tablehead}
 & & World $\downarrow$ & Root-local $\downarrow$ & World $\downarrow$ & Root-local $\downarrow$ & World $\downarrow$ & Root-local $\downarrow$ \\
\midrule
\rowcolor{tablegroup}
Turn and Jog & 101 & \textbf{5.45} / \textbf{10.74} & \textbf{3.11} / \textbf{4.51} & \textbf{5.45} / \textbf{10.74} & \textbf{3.11} / \textbf{4.51} & \textbf{5.45} / \textbf{10.74} & \textbf{3.11} / \textbf{4.51} \\
\rowcolor{tablegroup}
 & 202 & \textbf{4.32} / \textbf{7.55} & 3.29 / \textbf{4.65} & 6.19 / 13.49 & 3.39 / 5.07 & 5.07 / 7.89 & \textbf{3.20} / 4.77 \\
Side Step & 101 & 7.22 / 8.79 & 2.55 / 4.49 & \textbf{3.28} / \textbf{5.56} & \textbf{1.99} / \textbf{4.06} & 7.22 / 8.79 & 2.55 / 4.49 \\
 & 202 & \textbf{7.15} / \textbf{8.51} & \textbf{2.36} / \textbf{4.01} & \textbf{7.15} / \textbf{8.51} & \textbf{2.36} / \textbf{4.01} & \textbf{7.15} / \textbf{8.51} & \textbf{2.36} / \textbf{4.01} \\
\rowcolor{tablegroup}
Forward Jog & 101 & 5.46 / 7.97 & \textbf{1.68} / 2.37 & 6.49 / 9.99 & 1.82 / 2.65 & \textbf{5.02} / \textbf{7.32} & 1.70 / \textbf{2.35} \\
\rowcolor{tablegroup}
 & 202 & 4.72 / 6.50 & \textbf{1.74} / \textbf{2.36} & 5.07 / 7.50 & 1.89 / 2.58 & \textbf{4.43} / \textbf{5.97} & 1.77 / 2.61 \\
Rubbing Eyes & 101 & \textbf{1.44} / \textbf{1.88} & \textbf{1.53} / \textbf{2.16} & \textbf{1.44} / \textbf{1.88} & \textbf{1.53} / \textbf{2.16} & \textbf{1.44} / \textbf{1.88} & \textbf{1.53} / 2.17 \\
 & 202 & \textbf{1.41} / \textbf{1.77} & \textbf{1.33} / \textbf{2.08} & \textbf{1.41} / \textbf{1.77} & \textbf{1.33} / \textbf{2.08} & \textbf{1.41} / \textbf{1.77} & \textbf{1.33} / \textbf{2.08} \\
\rowcolor{tablegroup}
Praying/Salute & 101 & 3.13 / 5.34 & \textbf{2.37} / \textbf{3.34} & \textbf{2.65} / \textbf{3.97} & 2.69 / 3.96 & 2.75 / 4.82 & 2.49 / 3.49 \\
\rowcolor{tablegroup}
 & 202 & 3.99 / 5.96 & \textbf{2.72} / \textbf{3.75} & 3.98 / \textbf{5.95} & \textbf{2.72} / \textbf{3.75} & \textbf{3.96} / 5.96 & \textbf{2.72} / 3.76 \\
Heel-and-Toe Dance & 101 & 7.11 / 10.81 & \textbf{3.24} / 4.66 & 8.02 / 13.91 & 3.42 / 4.77 & \textbf{5.54} / \textbf{9.41} & 3.27 / \textbf{4.35} \\
 & 202 & 6.04 / 8.92 & \textbf{3.08} / 4.40 & 7.30 / 12.51 & 3.30 / 4.75 & \textbf{5.57} / \textbf{8.28} & 3.23 / \textbf{4.32} \\
\rowcolor{tablegroup}
Crouch Transition & 101 & 2.95 / 3.94 & 3.30 / 3.71 & 2.59 / 3.79 & \textbf{2.70} / \textbf{3.25} & \textbf{2.00} / \textbf{3.15} & 2.87 / 3.32 \\
\rowcolor{tablegroup}
 & 202 & 2.88 / 3.71 & 3.55 / 4.11 & \textbf{1.69} / \textbf{2.13} & \textbf{2.90} / \textbf{3.56} & 2.16 / 2.80 & \textbf{2.90} / 4.00 \\
Pointing Down & 101 & \textbf{2.72} / \textbf{4.38} & \textbf{1.38} / \textbf{1.80} & 2.91 / 5.70 & 1.42 / 1.81 & 2.89 / 5.70 & 1.42 / \textbf{1.80} \\
 & 202 & 3.39 / \textbf{6.11} & \textbf{1.38} / \textbf{1.72} & \textbf{3.38} / \textbf{6.11} & \textbf{1.38} / \textbf{1.72} & 3.39 / \textbf{6.11} & \textbf{1.38} / \textbf{1.72} \\
\rowcolor{tablegroup}
Dodge Backward & 101 & 4.60 / 7.98 & \textbf{2.16} / 4.48 & 7.68 / 11.82 & 2.55 / 4.96 & \textbf{4.09} / \textbf{7.57} & 2.20 / \textbf{4.35} \\
\rowcolor{tablegroup}
 & 202 & 5.75 / \textbf{10.04} & \textbf{2.55} / \textbf{4.15} & \textbf{5.71} / \textbf{10.04} & \textbf{2.55} / \textbf{4.15} & \textbf{5.71} / \textbf{10.04} & \textbf{2.55} / \textbf{4.15} \\
Injured-Torso Walk & 101 & 3.24 / 5.01 & 1.48 / 2.13 & 3.14 / 4.74 & 1.60 / 2.17 & \textbf{2.81} / \textbf{4.30} & \textbf{1.42} / \textbf{2.10} \\
 & 202 & \textbf{2.79} / \textbf{4.24} & \textbf{1.52} / \textbf{2.16} & \textbf{2.79} / \textbf{4.24} & \textbf{1.52} / \textbf{2.16} & \textbf{2.79} / \textbf{4.24} & \textbf{1.52} / \textbf{2.16} \\
\rowcolor{tablegroup}
Broad Jump & 101 & \multicolumn{2}{c}{\textsc{NC}} & 4.57 / \textbf{8.14} & \textbf{4.42} / \textbf{7.78} & \textbf{4.52} / 8.52 & 4.79 / 8.34 \\
\rowcolor{tablegroup}
 & 202 & \multicolumn{2}{c}{\textsc{NC}} & 11.42 / 30.52 & \textbf{4.77} / 9.09 & \textbf{6.27} / \textbf{11.94} & 4.93 / \textbf{8.74} \\
\bottomrule
\end{tabularx}
\par\vspace{2pt}
\begin{minipage}{\textwidth}\papertablenote Mean / temporal P95 in cm, averaged over the
three held-out repeats per seed. NC marks an incomplete cohort; completion
patterns are summarized in Section~\ref{sec:current-protocol}. Bold marks the lowest
defined value within each task and seed at the shown precision.\end{minipage}
\par\nointerlineskip\vspace{2pt}\noindent\rule{\textwidth}{0.4pt}
\endgroup
\end{table*}

The three native methods each evaluate 144 checkpoints. Accepted
updates total 10 for BeyondMimic, 9 for Static Search and 13 for MimicX; 10, 7
and 12 of their 24 endpoints, respectively, contain an accepted update.
The remaining endpoints retain the shared warmstart. Requiring completion
and the additional native precision limits across selection, nominal and
perturbed evaluations, the methods qualify 17/24, 13/24 and 15/24
endpoints. Selection records measure the gate's operation; the completion and
tracking tables use only the independent final evaluations.

\section{Earlier Video-Study Protocols}
\label{sec:setup}

This section records the earlier video-study configurations used by the
retained component, learning-curve and qualitative evidence. The twelve-motion
matched-budget study follows Appendix~\ref{sec:current-protocol}. Each result
below is interpreted under its own reference, training budget and metric.

\subsection{Tasks, Robot, and Learning Backend}

The earlier four-video study comprises Tennis Swing, Football Juggling, Dance Sequence, and
Kung Fu Sequence, with evaluation horizons of 800, 455, 784, and 1,062 control
steps, respectively. The corresponding reference lengths are 519, 455, 784,
and 1,062 frames at 50\,Hz. All use a 29-DoF Unitree G1 on flat terrain in
MuJoCo/MjLab~\citep{mujoco2012}. Simulation advances at 200\,Hz and the policy
at 50\,Hz. Fourteen bodies are tracked, with the torso as the anchor. The
actor has 160 input dimensions and 29 joint-position outputs; actor and critic
use multilayer perceptrons with hidden widths $(512,256,128)$, ELU activations,
and observation normalization.

PPO~\citep{ppo2017} collects 24 steps from each of 1,024 environments per
iteration, followed by five optimization epochs and four minibatches. Its
clipping parameter is $0.2$, discount $0.99$, advantage decay $0.95$, entropy
coefficient $0.005$, and target KL $0.01$. Task-ordered continuation learning-rate
settings are $(10^{-6},7\times10^{-7},9\times10^{-7},5\times10^{-7})$, with an
adaptive schedule. The tracking configuration disables observation
corruption, push perturbations, joint initialization noise, and the configured
friction, encoder-bias, and center-of-mass randomizations.

\subsection{Video Supervision Configurations}

Each method uses continuation seeds $101$, $202$, and $303$ on every core task. Each task has one registered
warmstart checkpoint shared by all methods and continuation seeds.

\paragraph{Fixed Reference.}
We resume PPO for 250 iterations using the original retargeted motion and
fixed tracking objective. The backend retains adaptive motion-start sampling.

\paragraph{Failure Curriculum.}
We use the same reference and reward profile, replacing start sampling by a
$50\%$ mixture of full-start and window initialization. Previously diagnosed
windows are $[340,430)$ for Tennis Swing, $[380,430)$ for Football Juggling,
$[500,545)$ for Dance Sequence, and $[165,225)$ for Kung Fu Sequence.

\paragraph{Task-Aware Refinement.}
We retain this curriculum and add task-priority position, height, and
linear-velocity rewards, position and angular-velocity rewards for the
pelvis and torso,
stronger anchor tracking, and adjusted action smoothing. Priority groups are
the arms for tennis, both legs for football, pelvis, torso and ankles for dance,
and pelvis, torso and both legs for kung fu.

\paragraph{\mimicx{}.}
We run one automated refinement iteration with three heuristic candidates,
each receiving 250 PPO iterations from the shared warmstart. Its full
supervision profile includes inherited wrist, core, and stance-slip terms
before diagnostic updates are added. Thus the first three methods each use
$250\times24\times1{,}024=6{,}144{,}000$ new transitions; \mimicx{} uses three
such candidate budgets, or 18,432,000 transitions, plus verification.
Dance Sequence and Kung Fu Sequence use registered repaired motion files for
\mimicx{}, including its current-policy evaluation; the other methods use the
original references. Tennis Swing and Football Juggling share the original
reference across methods. Body errors are evaluated against each method's
selected reference, not re-evaluated against a common original target.

\subsection{Verification and Metric Reduction}

The four-video study ranks candidates by the lexicographic key
\begin{equation}
 K_j=(Z_j,H_j,-N_j,-b_j,-z_j,\widetilde r_j),
 \label{eq:verification-key}
\end{equation}
where $Z$ counts successful repeats, $H$ is the worst first-failure step,
$N$ counts terminations, $b,z$ are medians of last-record body and ankle-height
errors, and $\widetilde r$ is median evaluation reward. Its body guard is
\begin{equation}
 b_j\le(1+0.20)b_k.
 \label{eq:body-guard}
\end{equation}
This is the four-video selection criterion used in
Appendix~\ref{sec:results}; the matched-budget motion study uses
joint verification in Appendix~\ref{sec:exact-joint-gate}.

All reported core checkpoints are evaluated with seeds $1001$, $2002$, and
$3003$, yielding nine rollouts per task--method cell. Each rollout uses one
environment, deterministic policy inference, full-start initialization, and
active termination tests. These tests flag anchor-height error above
$0.35$\,m, either ankle's reference-relative height error above $0.25$\,m,
or an absolute difference above $0.8$ between the vertical components of
reference and robot anchor-frame gravity. The orientation threshold is a
projected-gravity discrepancy. \mimicx{} reports the selected policy's
verification repeats; protected trials report the shared warmstart again.

Success and Robust Execution Horizon follow
the historical convention $Z=\sum_s\mathbf1[n_s=0]$ and
$H=\min_s h_s$, where $h_s=T+1$ on success: tennis success uses the sentinel
$800+1=801$, not the reference length. At motion end, the backend restarts the
reference and reinitializes robot state. Success thus means no strict
termination over the evaluation schedule. The evaluator also continues after
termination resets. The instantaneous whole-body scalar is
\begin{equation}
  e_t^{\mathrm{body}}=
    \max_{b\in\mathcal B}\|p_{t,b}-\widetilde p_{t,b}\|_2.
  \label{eq:reported-body-error}
\end{equation}
Mean, nearest-rank 95th percentile, and peak pool all steps across three
repeats; the mean is a temporal mean of worst-body error. Anchor error uses
world-position distance, velocity errors average
bodywise Euclidean discrepancies, and the end-effector height scalar takes
the maximum over the two ankles. These summaries differ
from the final-record medians used by the guard. Task summaries
average training-seed values, while success counts all nine repeats.

\subsection{Additional Videos and Supplied Motions}

Four additional videos cover Basketball, Soccer, Michael Jackson dance, and
Uniandes Dance. Their Fixed Reference/\mimicx{} comparisons use three training
seeds, 300 iterations per continuation, and three evaluation repeats;
\mimicx{} allows two iterations with one candidate each. Reference-input
breadth covers 11 LaFAN1 references~\citep{lafan2020} and three AMASS/ACCAD
references~\citep{amass2019}, including two retargetings of the Form~1
source motion, using training seed $101$, three repeats, and the
250-iteration candidate budget.

\subsection{Heterogeneous Execution Benchmark}

The \hloop{} benchmark freezes the four Fixed Reference checkpoints trained
with seed $101$. Two evaluation seeds per checkpoint produce eight GPU rollout
jobs, eight dependent CPU diagnosis jobs, and one selection job, with zero policy
updates. The selector retains diagnoses whose frontier reaches the requested horizon
and uses a deterministic ordering in every execution mode. Sequential,
bulk-synchronous, and dependency-driven execution share this 17-job graph;
parallel modes use two GPU slots and at most eight CPU workers. After one
warmup per mode, five repetitions rotate execution order. We report median
wall time, its ratio to sequential execution, and identical selection outcomes for this timed workload.

\section{Qualitative Evidence and Video Correspondence}
\label{sec:qualitative}

\subsection{Human reconstruction and learned execution}

Figure~\ref{fig:registered_stages} combines Tennis and Forest processing
stages with matched Stair Ascent states. Figures~\ref{fig:tennis_events} and
\ref{fig:main-qualitative} add comparisons at task-critical phases.
The sequences below connect original video frames to recovered human motion
and recorded policy states. Each upper row contains seven input frames;
the lower row progresses from three human states to four robot states.
Tennis and Football use the recovered source-phase mapping. Dance uses
nominal source and rollout timestamps over a selected interval.
Figure~\ref{fig:kungfu-video-sequence} shows Kung Fu;
Figure~\ref{fig:track-stairs-correspondence} adds Track Running and Stair Ascent.

\evidencefigure{\includegraphics[width=\linewidth]{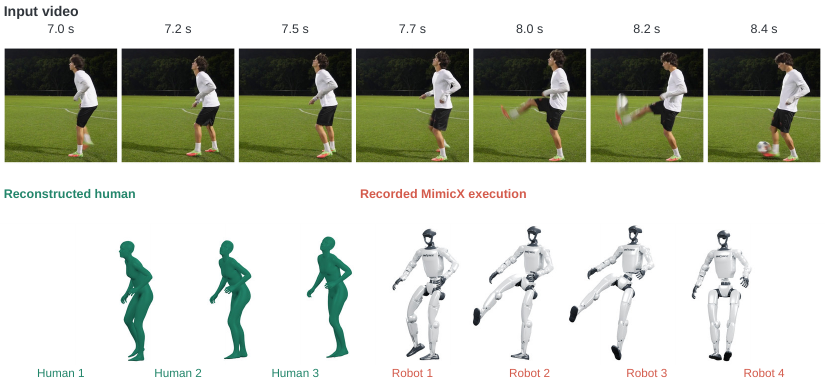}}{fig:procession_football}{\textbf{Football Juggling across time and embodiment.} Original-video frames provide the scene and action context for the reconstructed human motion and policy sequence below.}
\evidencefigure{\includegraphics[width=\linewidth]{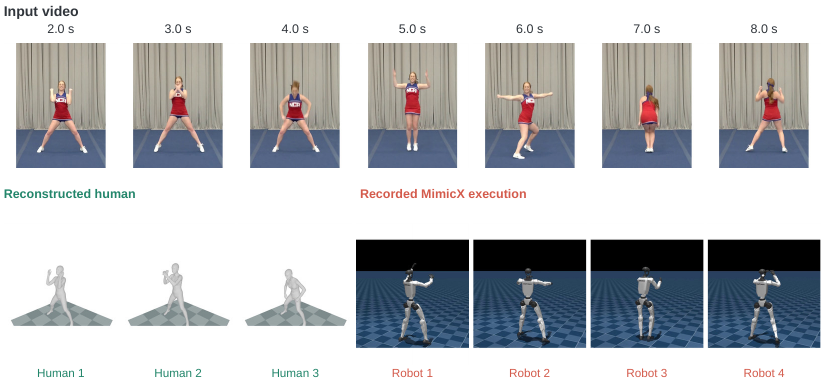}}{fig:procession_dance}{\textbf{Dance across time and embodiment.} Seven original-video frames show the selected interval; the lower row follows human reconstruction and policy execution.}

\subsection{Task-critical events reveal execution differences}

The task-event plates show Fixed Reference and \mimicx{} at selected
execution moments for Tennis, Dance and Kung Fu in the original component
study (Figures~\ref{fig:tennis_events}, \ref{fig:dance_events}
and~\ref{fig:kungfu_events}). In these
plates the solid robot is the policy execution and the translucent
method-colored robot is its reference. Fixed Reference uses gray and
\mimicx{} uses coral. The comparison focuses on visible pose and
reference alignment; the accompanying numerical tables summarize every
declared trial. Figure~\ref{fig:main-results} relates these examples to the
aggregate tracking and execution measurements of the four-video study.

\evidencefigure{\includegraphics[width=\linewidth]{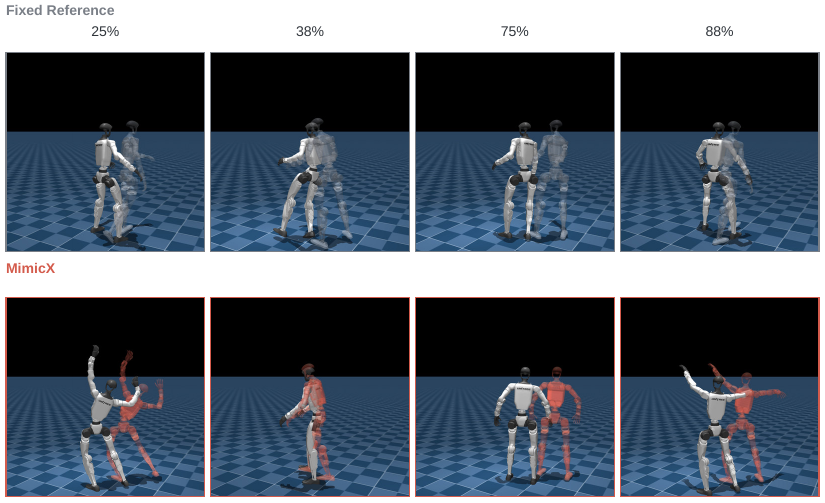}}{fig:tennis_events}{\textbf{Tennis task-critical execution.} Fixed Reference and \mimicx{} at four recorded event times. The neutral robot is the policy and the translucent method-colored robot is the reference overlay. Gray identifies Fixed Reference and coral identifies \mimicx{}.}
\begin{figure*}[tp]
\centering
\begin{subfigure}[t]{0.241\linewidth}\includegraphics[width=\linewidth]{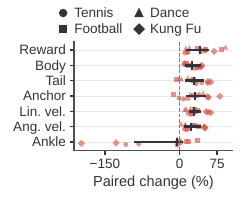}\caption{Paired effects}\end{subfigure}\hfill
\begin{subfigure}[t]{0.241\linewidth}\includegraphics[width=\linewidth]{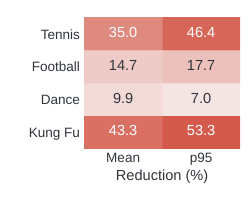}\caption{Mean and tail gains}\end{subfigure}\hfill
\begin{subfigure}[t]{0.241\linewidth}\includegraphics[width=\linewidth]{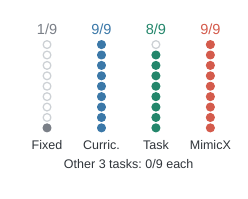}\caption{Tennis completion}\end{subfigure}\hfill
\begin{subfigure}[t]{0.241\linewidth}\includegraphics[width=\linewidth]{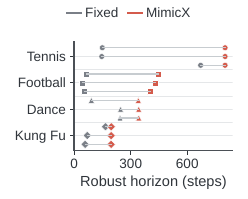}\caption{Execution horizon}\end{subfigure}
\caption{\textbf{Supervision configurations across four video-derived skills.}
(a) All twelve task-matched trials across seven reward and tracking channels.
(b) Task-wise reductions in mean and p95 tracking error, in percent. (c) Full-horizon completion across repeated evaluations. (d) Worst-case
execution horizon. Black marks in (a) denote median and interquartile range.
Fixed, Curriculum, and Task-aware abbreviate Fixed Reference, Failure
Curriculum, and Task-Aware Refinement; gray, blue, green, and coral identify
the four methods (Curric. and Task in (c)). Tail denotes body-error p95 and Ankle denotes ankle-height error.}
\label{fig:main-results}\label{fig:core_results}
\end{figure*}

\begin{figure*}[tp]
\centering
\begin{subfigure}[t]{\HalfPlate}\centering
\includegraphics[width=\linewidth]{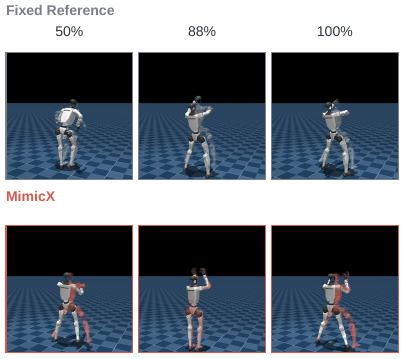}
\caption{Dance: three task-critical events}\label{fig:dance_events}
\end{subfigure}\hfill
\begin{subfigure}[t]{\HalfPlate}\centering
\includegraphics[width=\linewidth]{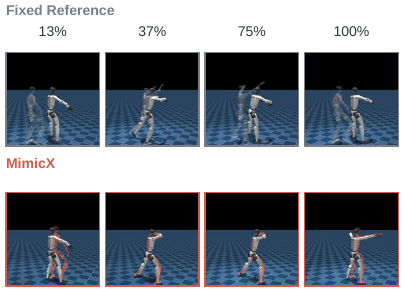}
\caption{Kung Fu: four task-critical events}\label{fig:kungfu_events}
\end{subfigure}
\caption{\textbf{Task-critical execution on articulated whole-body motions.}
Fixed Reference and MimicX are compared at matched events.
Solid robots show execution; translucent robots show method-colored references.}
\end{figure*}

\begin{figure*}[tp]
\centering
\begin{minipage}[t]{\HalfPlate}
\centering
\includegraphics[width=\linewidth]{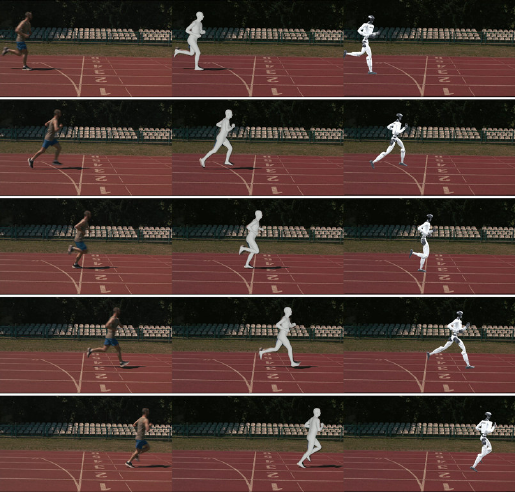}
\par\smallskip{\small (a) Track Running}
\end{minipage}\hfill
\begin{minipage}[t]{\HalfPlate}
\centering
\includegraphics[width=\linewidth]{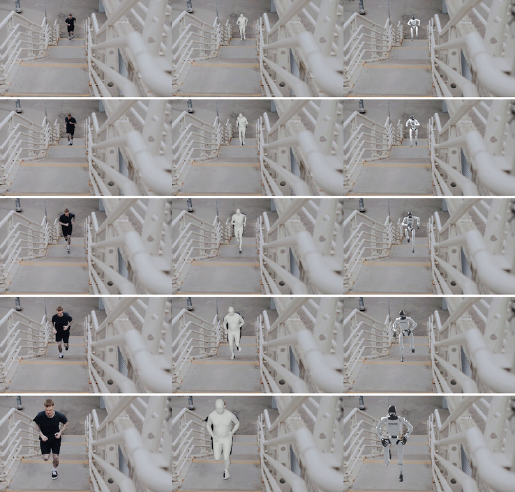}
\par\smallskip{\small (b) Stair Ascent}
\end{minipage}
\caption{\textbf{Source-view correspondence in Track Running and Stair Ascent.}
Columns show input video, reconstructed human, and G1 reference; rows follow
five synchronized instants.}
\label{fig:track-stairs-correspondence}
\end{figure*}

\subsection{Football execution in its task scene}

The Football panorama (Figure~\ref{fig:football_panorama}) uses the same
time-ordered composition as the Tennis panorama in Figure~\ref{fig:tennis_panorama}.
Ball placement follows the recovered source phases.
\evidencefigure{\includegraphics[width=\linewidth]{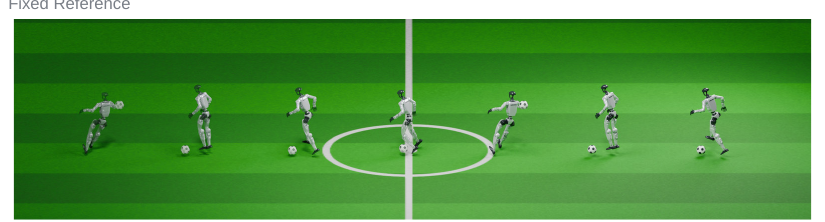}}{fig:football_panorama}{\textbf{Football Juggling in the registered scene.} Seven recorded Fixed Reference states, separated spatially and ordered by increasing opacity.}

\subsection{Scene reconstruction connects the rendered policy to the input}

Figure~\ref{fig:registered_stages} links the observed Tennis frame,
camera-space human reconstruction, policy in simulation, and the same
policy presented in a reconstructed scene. The stage panels use a common layout. The final presentation registers the recovered camera
to a Gaussian background while preserving the robot's recorded pose.
This gives the reader both an unobstructed simulation view and the
original scene context for the recorded tracking experiment.

The Tennis presentation video contains 260 frames at 25\,Hz. The foreground is
rendered at $1624\times1440$, and the source-view Gaussian background
at $2436\times2160$ from an $812\times720$ source image. Camera registration
has a median feature reprojection error of 0.647 source pixels, and all
foreground contours remain inside the output canvas. These measurements
quantify camera-registration accuracy and foreground coverage in the
reconstructed-scene presentation.

\subsection{Visual breadth beyond the four core tasks}

Figures~\ref{fig:soccer_events} and \ref{fig:uniandes_events} show two
additional video-derived tasks with positive tracking results. The
motion-reference plate in Figure~\ref{fig:lafan_events} illustrates the
reference-input pathway. Together
with the numerical breadth table, these examples show how the same
diagnosis, proposal and verification interfaces apply after different
reference-construction routes.

\evidencefigure{\includegraphics[width=\linewidth]{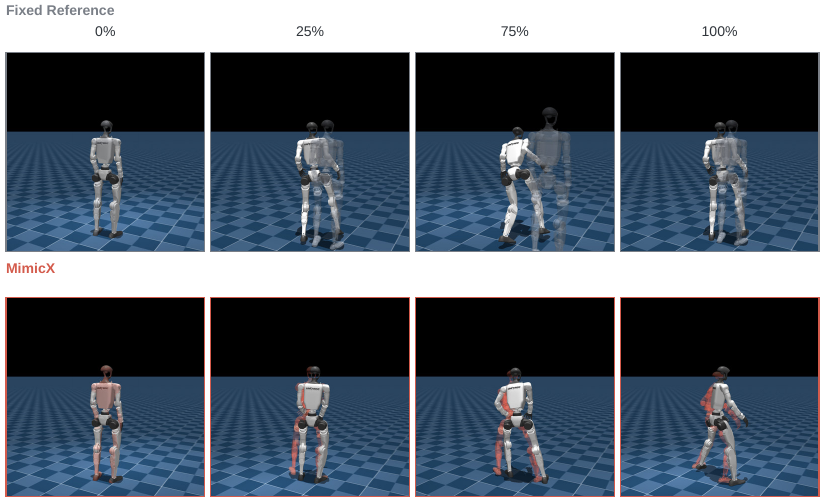}}{fig:lafan_events}{\textbf{AMASS/ACCAD Form 1: provided retargeting.} Fixed Reference and \mimicx{} on the provided robot-motion reference, with recorded policy states and method-colored reference overlays. Source attribution follows the motion audit in Table~\ref{tab:breadth}. This pathway starts from an existing motion reference.}

\section{Earlier Video Results and Component Studies}
\label{sec:results}

These results detail the four-video supervision configurations under
Appendix~\ref{sec:setup}. The matched-budget motion results use the
independent evaluations in Appendix~\ref{sec:current-protocol}.

\subsection{Tracking and Execution Under Video Supervision}

Table~\ref{tab:core} and Figure~\ref{fig:main-results} evaluate the four
component configurations on every core task. Relative to Fixed Reference,
\mimicx{} improves both reported body-tracking error and worst-case
execution horizon for all twelve task-matched continuation-trial pairs.
The task-macro effects are a 25.72\% reduction in body error and a
255.57\% increase in execution horizon, using the reference and metric
definitions in Appendix~\ref{sec:setup}. Table~\ref{tab:core} reports task-level
values and Table~\ref{tab:paired} reports paired effects across metrics.

% Table design v3
% Generated by scripts/build_tables.py; source hashes: tables/source_manifest.json
\begin{table*}[p]
\centering
\caption{Four-task continuation comparison. Entries show mean $\pm$ sample SD over three continuation-seed trials, each aggregating three evaluation rollouts. Bold marks each task--metric best mean, including ties.}
\label{tab:core}
\begingroup
\papertablestyle
\begin{tabularx}{\linewidth}{>{\raggedright\arraybackslash}p{.18\linewidth}*{7}{>{\centering\arraybackslash}X}}
\toprule
\rowcolor{tablehead}
\textbf{Method} & \textbf{\makecell{Horizon\\(steps, $\uparrow$)}} & \textbf{\makecell{Reward\\($\uparrow$)}} & \textbf{\makecell{Body mean\\(m, $\downarrow$)}} & \textbf{\makecell{Body P95\\(m, $\downarrow$)}} & \textbf{\makecell{Anchor error\\(m, $\downarrow$)}} & \textbf{\makecell{Linear vel.\\(m/s, $\downarrow$)}} & \textbf{\makecell{Angular vel.\\(rad/s, $\downarrow$)}}\\
\midrule
\rowcolor{tablegroup}
\multicolumn{8}{l}{\textit{Tennis Swing ($H=800$ steps)}} \\
\rowcolor{basebg}
\makecell[l]{Fixed\\Reference} & \makecell[c]{322.0\\{\color{FixedGray}$\pm$ 302.2}} & \makecell[c]{0.0657\\{\color{FixedGray}$\pm$ 0.0023}} & \makecell[c]{0.241\\{\color{FixedGray}$\pm$ 0.020}} & \makecell[c]{0.518\\{\color{FixedGray}$\pm$ 0.125}} & \makecell[c]{0.301\\{\color{FixedGray}$\pm$ 0.023}} & \makecell[c]{0.566\\{\color{FixedGray}$\pm$ 0.053}} & \makecell[c]{2.287\\{\color{FixedGray}$\pm$ 0.031}} \\
\rowcolor{curriculumbg}
\makecell[l]{Failure\\Curriculum} & \makecell[c]{\textbf{801.0}\\{\color{FixedGray}$\pm$ 0.0}} & \makecell[c]{\textbf{0.0816}\\{\color{FixedGray}$\pm$ 0.0008}} & \makecell[c]{0.184\\{\color{FixedGray}$\pm$ 0.002}} & \makecell[c]{0.282\\{\color{FixedGray}$\pm$ 0.007}} & \makecell[c]{0.208\\{\color{FixedGray}$\pm$ 0.014}} & \makecell[c]{\textbf{0.426}\\{\color{FixedGray}$\pm$ 0.004}} & \makecell[c]{\textbf{1.798}\\{\color{FixedGray}$\pm$ 0.012}} \\
\rowcolor{taskbg}
\makecell[l]{Task-Aware\\Refinement} & \makecell[c]{768.0\\{\color{FixedGray}$\pm$ 57.2}} & \makecell[c]{0.0788\\{\color{FixedGray}$\pm$ 0.0014}} & \makecell[c]{0.178\\{\color{FixedGray}$\pm$ 0.011}} & \makecell[c]{0.293\\{\color{FixedGray}$\pm$ 0.012}} & \makecell[c]{\textbf{0.177}\\{\color{FixedGray}$\pm$ 0.023}} & \makecell[c]{0.454\\{\color{FixedGray}$\pm$ 0.003}} & \makecell[c]{2.074\\{\color{FixedGray}$\pm$ 0.027}} \\
\rowcolor{oursbg}
\mimicx & \makecell[c]{\textbf{801.0}\\{\color{FixedGray}$\pm$ 0.0}} & \makecell[c]{0.0750\\{\color{FixedGray}$\pm$ 0.0012}} & \makecell[c]{\textbf{0.157}\\{\color{FixedGray}$\pm$ 0.008}} & \makecell[c]{\textbf{0.278}\\{\color{FixedGray}$\pm$ 0.009}} & \makecell[c]{0.276\\{\color{FixedGray}$\pm$ 0.006}} & \makecell[c]{0.439\\{\color{FixedGray}$\pm$ 0.017}} & \makecell[c]{1.979\\{\color{FixedGray}$\pm$ 0.064}} \\
\midrule
\rowcolor{tablegroup}
\multicolumn{8}{l}{\textit{Football Juggling ($H=455$ steps)}} \\
\rowcolor{basebg}
\makecell[l]{Fixed\\Reference} & \makecell[c]{53.7\\{\color{FixedGray}$\pm$ 10.5}} & \makecell[c]{0.0437\\{\color{FixedGray}$\pm$ 0.0043}} & \makecell[c]{0.286\\{\color{FixedGray}$\pm$ 0.017}} & \makecell[c]{0.612\\{\color{FixedGray}$\pm$ 0.096}} & \makecell[c]{0.130\\{\color{FixedGray}$\pm$ 0.009}} & \makecell[c]{0.840\\{\color{FixedGray}$\pm$ 0.056}} & \makecell[c]{3.669\\{\color{FixedGray}$\pm$ 0.241}} \\
\rowcolor{curriculumbg}
\makecell[l]{Failure\\Curriculum} & \makecell[c]{264.3\\{\color{FixedGray}$\pm$ 2.1}} & \makecell[c]{\textbf{0.0785}\\{\color{FixedGray}$\pm$ 0.0007}} & \makecell[c]{\textbf{0.237}\\{\color{FixedGray}$\pm$ 0.007}} & \makecell[c]{\textbf{0.395}\\{\color{FixedGray}$\pm$ 0.005}} & \makecell[c]{0.093\\{\color{FixedGray}$\pm$ 0.006}} & \makecell[c]{\textbf{0.535}\\{\color{FixedGray}$\pm$ 0.006}} & \makecell[c]{\textbf{2.051}\\{\color{FixedGray}$\pm$ 0.025}} \\
\rowcolor{taskbg}
\makecell[l]{Task-Aware\\Refinement} & \makecell[c]{360.0\\{\color{FixedGray}$\pm$ 39.8}} & \makecell[c]{0.0672\\{\color{FixedGray}$\pm$ 0.0011}} & \makecell[c]{0.246\\{\color{FixedGray}$\pm$ 0.007}} & \makecell[c]{0.448\\{\color{FixedGray}$\pm$ 0.002}} & \makecell[c]{\textbf{0.083}\\{\color{FixedGray}$\pm$ 0.014}} & \makecell[c]{0.599\\{\color{FixedGray}$\pm$ 0.019}} & \makecell[c]{2.765\\{\color{FixedGray}$\pm$ 0.076}} \\
\rowcolor{oursbg}
\mimicx & \makecell[c]{\textbf{427.3}\\{\color{FixedGray}$\pm$ 22.4}} & \makecell[c]{0.0683\\{\color{FixedGray}$\pm$ 0.0041}} & \makecell[c]{0.244\\{\color{FixedGray}$\pm$ 0.017}} & \makecell[c]{0.504\\{\color{FixedGray}$\pm$ 0.049}} & \makecell[c]{0.108\\{\color{FixedGray}$\pm$ 0.026}} & \makecell[c]{0.568\\{\color{FixedGray}$\pm$ 0.021}} & \makecell[c]{2.489\\{\color{FixedGray}$\pm$ 0.032}} \\
\midrule
\rowcolor{tablegroup}
\multicolumn{8}{l}{\textit{Dance Sequence ($H=784$ steps)}} \\
\rowcolor{basebg}
\makecell[l]{Fixed\\Reference} & \makecell[c]{193.3\\{\color{FixedGray}$\pm$ 88.6}} & \makecell[c]{0.0435\\{\color{FixedGray}$\pm$ 0.0115}} & \makecell[c]{0.270\\{\color{FixedGray}$\pm$ 0.008}} & \makecell[c]{0.502\\{\color{FixedGray}$\pm$ 0.051}} & \makecell[c]{0.215\\{\color{FixedGray}$\pm$ 0.042}} & \makecell[c]{0.811\\{\color{FixedGray}$\pm$ 0.056}} & \makecell[c]{3.779\\{\color{FixedGray}$\pm$ 0.216}} \\
\rowcolor{curriculumbg}
\makecell[l]{Failure\\Curriculum} & \makecell[c]{506.3\\{\color{FixedGray}$\pm$ 14.6}} & \makecell[c]{\textbf{0.0679}\\{\color{FixedGray}$\pm$ 0.0007}} & \makecell[c]{\textbf{0.243}\\{\color{FixedGray}$\pm$ 0.006}} & \makecell[c]{0.488\\{\color{FixedGray}$\pm$ 0.035}} & \makecell[c]{0.100\\{\color{FixedGray}$\pm$ 0.010}} & \makecell[c]{\textbf{0.622}\\{\color{FixedGray}$\pm$ 0.007}} & \makecell[c]{\textbf{2.892}\\{\color{FixedGray}$\pm$ 0.025}} \\
\rowcolor{taskbg}
\makecell[l]{Task-Aware\\Refinement} & \makecell[c]{\textbf{522.3}\\{\color{FixedGray}$\pm$ 0.6}} & \makecell[c]{0.0614\\{\color{FixedGray}$\pm$ 0.0004}} & \makecell[c]{0.256\\{\color{FixedGray}$\pm$ 0.008}} & \makecell[c]{0.482\\{\color{FixedGray}$\pm$ 0.016}} & \makecell[c]{\textbf{0.077}\\{\color{FixedGray}$\pm$ 0.003}} & \makecell[c]{0.631\\{\color{FixedGray}$\pm$ 0.003}} & \makecell[c]{3.484\\{\color{FixedGray}$\pm$ 0.032}} \\
\rowcolor{oursbg}
\mimicx & \makecell[c]{341.3\\{\color{FixedGray}$\pm$ 1.2}} & \makecell[c]{0.0581\\{\color{FixedGray}$\pm$ 0.0001}} & \makecell[c]{0.244\\{\color{FixedGray}$\pm$ 0.001}} & \makecell[c]{\textbf{0.467}\\{\color{FixedGray}$\pm$ 0.000}} & \makecell[c]{0.133\\{\color{FixedGray}$\pm$ 0.000}} & \makecell[c]{0.686\\{\color{FixedGray}$\pm$ 0.004}} & \makecell[c]{3.447\\{\color{FixedGray}$\pm$ 0.010}} \\
\midrule
\rowcolor{tablegroup}
\multicolumn{8}{l}{\textit{Kung Fu Sequence ($H=1062$ steps)}} \\
\rowcolor{basebg}
\makecell[l]{Fixed\\Reference} & \makecell[c]{97.7\\{\color{FixedGray}$\pm$ 59.4}} & \makecell[c]{0.0565\\{\color{FixedGray}$\pm$ 0.0046}} & \makecell[c]{0.249\\{\color{FixedGray}$\pm$ 0.005}} & \makecell[c]{0.634\\{\color{FixedGray}$\pm$ 0.189}} & \makecell[c]{0.229\\{\color{FixedGray}$\pm$ 0.115}} & \makecell[c]{0.548\\{\color{FixedGray}$\pm$ 0.062}} & \makecell[c]{2.084\\{\color{FixedGray}$\pm$ 0.038}} \\
\rowcolor{curriculumbg}
\makecell[l]{Failure\\Curriculum} & \makecell[c]{298.7\\{\color{FixedGray}$\pm$ 144.1}} & \makecell[c]{0.0832\\{\color{FixedGray}$\pm$ 0.0070}} & \makecell[c]{0.194\\{\color{FixedGray}$\pm$ 0.044}} & \makecell[c]{0.362\\{\color{FixedGray}$\pm$ 0.084}} & \makecell[c]{0.140\\{\color{FixedGray}$\pm$ 0.037}} & \makecell[c]{\textbf{0.220}\\{\color{FixedGray}$\pm$ 0.041}} & \makecell[c]{\textbf{0.998}\\{\color{FixedGray}$\pm$ 0.155}} \\
\rowcolor{taskbg}
\makecell[l]{Task-Aware\\Refinement} & \makecell[c]{\textbf{415.0}\\{\color{FixedGray}$\pm$ 94.0}} & \makecell[c]{0.0780\\{\color{FixedGray}$\pm$ 0.0020}} & \makecell[c]{0.209\\{\color{FixedGray}$\pm$ 0.011}} & \makecell[c]{0.358\\{\color{FixedGray}$\pm$ 0.029}} & \makecell[c]{0.108\\{\color{FixedGray}$\pm$ 0.007}} & \makecell[c]{0.255\\{\color{FixedGray}$\pm$ 0.006}} & \makecell[c]{1.249\\{\color{FixedGray}$\pm$ 0.093}} \\
\rowcolor{oursbg}
\mimicx & \makecell[c]{196.0\\{\color{FixedGray}$\pm$ 0.0}} & \makecell[c]{\textbf{0.0871}\\{\color{FixedGray}$\pm$ 0.0000}} & \makecell[c]{\textbf{0.141}\\{\color{FixedGray}$\pm$ 0.000}} & \makecell[c]{\textbf{0.296}\\{\color{FixedGray}$\pm$ 0.000}} & \makecell[c]{\textbf{0.070}\\{\color{FixedGray}$\pm$ 0.000}} & \makecell[c]{0.230\\{\color{FixedGray}$\pm$ 0.000}} & \makecell[c]{1.038\\{\color{FixedGray}$\pm$ 0.000}} \\

\bottomrule

\end{tabularx}
\par\smallskip\noindent\begin{minipage}{\linewidth}\papertablenote Body error is the temporal mean of the maximum aligned tracked-body error, pooled over rollout samples. P95 is the temporal 95th percentile of the same aligned tracked-body maximum, pooled across three rollouts per continuation seed. Anchor denotes world-coordinate torso-anchor error. Horizon is the earliest first-failure step across three rollouts ($H+1$ if none). Reward is dimensionless. Continuation seeds share a task-specific warmstart; protected Dance/Kung Fu trials retain the same checkpoint (Section~\ref{sec:setup}).\par\smallskip\hrule height 0.4pt\relax\end{minipage}
\endgroup
\end{table*}

\mimicx{} completes all nine Tennis evaluations. On Football, Dance and
Kung Fu it extends the worst-repeat horizon and reduces mean body error
relative to Fixed Reference.

Table~\ref{tab:core} combines robust execution horizon with mean and
95th-percentile tracking errors, capturing both sustained execution and
the magnitude of large tracking deviations.
Improvements extend beyond mean body error. Both velocity metrics improve
in all twelve pairs, while tail body and root-and-anchor errors improve in
ten (Table~\ref{tab:paired}).

% Table design v3
% Generated by scripts/build_tables.py; source hashes: tables/source_manifest.json
\begin{table*}[tp]
\centering
\caption{Paired improvements of \mimicx{} over Fixed Reference. All eight recorded metrics are retained. Relative changes are percentages, with positive values denoting reward/horizon increases or error reductions; wins count strictly positive task--continuation-seed pairs.}
\label{tab:paired}
\begingroup
\papertablestyle
\begin{tabularx}{\linewidth}{>{\raggedright\arraybackslash}p{.29\linewidth}*{6}{>{\centering\arraybackslash}X}}\toprule
\rowcolor{tablehead}
\textbf{Metric} & \textbf{\makecell{Task macro\\(\%, $\uparrow$)}} & \textbf{\makecell{Paired mean\\(\%, $\uparrow$)}} & \textbf{Wins} & \textbf{\makecell{Paired\\median\\(\%, $\uparrow$)}} & \textbf{\makecell{Minimum\\(\%)}} & \textbf{\makecell{Maximum\\(\%)}}\\
\midrule
Reward & +39.44 & +41.96 & 12/12 & +39.97 & +10.58 & +91.90 \\
Robust execution horizon & +255.57 & +320.70 & 12/12 & +255.78 & +18.07 & +904.65 \\
Mean aligned worst-body error & +25.72 & +25.65 & 12/12 & +24.26 & +7.35 & +44.58 \\
95th-percentile aligned worst-body error & +31.13 & +29.25 & 10/12 & +28.54 & -5.51 & +62.52 \\
Root-and-anchor error & +33.28 & +31.38 & 10/12 & +30.66 & -11.79 & +80.65 \\
Body linear-velocity error & +32.07 & +31.81 & 12/12 & +28.49 & +10.17 & +62.60 \\
Body angular-velocity error & +26.14 & +26.03 & 12/12 & +22.34 & +2.72 & +51.20 \\
End-effector height error (ankles) & -24.47 & -38.27 & 5/12 & -4.66 & -197.04 & +34.99 \\

\bottomrule
\end{tabularx}
\par\smallskip\noindent\begin{minipage}{\linewidth}\papertablenote For baseline $b$ and refined value $a$, let $r(b,a)=100s(a-b)/b$, with $s=+1$ for reward/horizon and $s=-1$ for errors. Task macro is $\frac14\sum_t r(\bar b_t,\bar a_t)$; paired mean is $\frac1{12}\sum_{t,k}r(b_{tk},a_{tk})$. Neither is the ratio of pooled means. Body error is the temporal mean of the maximum aligned tracked-body error, pooled over rollout samples. Body p95 is the temporal 95th percentile of the same spatial maximum. Shared-warmstart and protected-policy interpretation: Section~\ref{sec:setup}.\par\smallskip\hrule height 0.4pt\relax\end{minipage}
\endgroup
\end{table*}

\subsection{Cumulative components address different execution errors}

The component comparison in Table~\ref{tab:components} identifies complementary effects of the supervision components. Failure-conditioned sampling makes
difficult segments more influential in learning. It already brings Tennis
to full-horizon execution and substantially extends the other tasks'
execution horizons. Task-aware supervision further changes the allocation
of tracking effort: on Football, for example, it further extends execution horizon while
reducing root-and-anchor error.

The full loop selects a different operating point by combining proposals
with repeated execution guards. It provides the lowest mean body error
among the four methods on Tennis and Kung Fu, and the longest Football
execution horizon. Intermediate variants lead on some other channels,
including Dance execution horizon and several root-tracking values.
These results show how coordinated refinement and explicit acceptance
criteria balance complementary execution and tracking objectives. Table~\ref{tab:channels}
keeps local height, global tracking, and velocity channels visible together.

\begin{table*}[tp]
\begin{minipage}[t]{\columnwidth}\vspace{0pt}
% Generated by scripts/rebuild_table_layouts.py
\centering\caption{\textbf{Complementary components.} Task-balanced means over four motions.}
\label{tab:components}
\begingroup\papertablestyle
\begin{tabularx}{\linewidth}{>{\raggedright\arraybackslash}p{.30\linewidth}*{4}{>{\centering\arraybackslash}X}}
\toprule
\rowcolor{tablehead}
Metric & \makecell{Fixed\\Ref.} & FC & TA & \mimicx \\
\midrule
\rowcolor{curriculumbg}
Success (\%) & 2.8 & \textbf{25.0} & 22.2 & \cellcolor{oursbg}\textbf{25.0} \\
Horizon (\%) & 21.5 & 62.7 & \textbf{70.2} & \cellcolor{oursbg}64.0 \\
\rowcolor{curriculumbg}
Body (m) & 0.262 & 0.215 & 0.222 & \cellcolor{oursbg}\textbf{0.196} \\
Anchor (m) & 0.219 & 0.135 & \textbf{0.111} & \cellcolor{oursbg}0.147 \\
\rowcolor{curriculumbg}
Linear (m/s) & 0.691 & \textbf{0.451} & 0.485 & \cellcolor{oursbg}0.481 \\
Angular (rad/s) & 2.955 & \textbf{1.935} & 2.393 & \cellcolor{oursbg}2.238 \\
\bottomrule\end{tabularx}
\par\smallskip{\papertablenote\raggedright FC: Failure Curriculum; TA: Task-Aware Refinement. Higher success/horizon and lower errors are better. Bold marks metric bests. Horizon is $100\min(h,H)/H$.\par}
\par\smallskip\hrule height .4pt
\endgroup

\end{minipage}\hfill
\begin{minipage}[t]{\columnwidth}\vspace{0pt}
% Generated by scripts/rebuild_table_layouts.py
\centering\caption{\textbf{Task-critical channels.} Strongest selected event for training seed 101, evaluation seed 1001.}
\label{tab:channels}
\begingroup\papertablestyle
\begin{tabularx}{\linewidth}{l*{4}{>{\centering\arraybackslash}X}}
\toprule\rowcolor{tablehead}
Channel & Tennis & Football & Dance & \makecell{Kung\\Fu} \\
\midrule
\rowcolor{curriculumbg}
Phase (\%) & 25.1 & 100.0 & 87.6 & 37.4 \\
Local & \textbf{0.432} & \textbf{0.740} & 0.678 & 0.706 \\
\rowcolor{curriculumbg}
Global & 0.248 & 0.630 & \textbf{0.765} & \textbf{0.770} \\
Dynamics & 0.225 & 0.193 & 0.527 & 0.671 \\
\rowcolor{curriculumbg}
Dominant & Local & Local & Global & Global \\
Gate & \tabletag{oursbg}{Accept} & \tabletag{oursbg}{Accept} & \tabletag{basebg}{Protect} & \tabletag{basebg}{Protect} \\
\bottomrule

\end{tabularx}
\par\smallskip{\papertablenote\raggedright Bold identifies the dominant channel. Tennis/Football accept refinement; Dance/Kung Fu protect the current policy. Scores are dimensionless improvements; phase locates the selected event.\par}
\par\smallskip\hrule height .4pt
\endgroup

\end{minipage}
\end{table*}

% Original table definitions retained in the appendix.
Body error is the temporal mean of the maximum aligned tracked-body error, pooled over rollout samples. Anchor is world-coordinate torso-anchor error; velocities average world-coordinate errors across bodies and time. Normalized horizon is $100\min(h,H)/H$ per trial. Continuation, reference, and retained-checkpoint protocols are in Section~\ref{sec:setup}.

Scores are dimensionless ($\uparrow$): each channel averages $\mathrm{clip}((b-a)/(|b|+10^{-6}),-2,2)$ across its error terms. Local: wrists, ankles, and maximum ankle-height error. Global: world anchor position plus aligned bodywise mean and maximum errors. Dynamics: anchor/body linear and angular velocity. Progress is event step divided by task horizon.

\subsection{Learning dynamics under selected supervision}

Figure~\ref{fig:main-learning} and Table~\ref{tab:learning} summarize twelve
250-update Tennis continuations.
These traces resolve the initial body-error transient, subsequent tracking
improvement, reward development, and episode-length evolution across the
full continuation.

% Table design v3
% Generated by scripts/build_tables.py; source hashes: tables/source_manifest.json
\begin{table*}[tp]
\centering
\caption{Dense Tennis learning-log summaries from all 3,000 records: four methods, three continuation seeds, and 250 recorded updates per trajectory. Entries show mean $\pm$ sample SD, computed across seeds after obtaining each trajectory's endpoint or integral. Training reward retains its logged scale and differs from evaluation reward.}
\label{tab:learning}
\begingroup
\papertablestyle
\begin{tabularx}{\linewidth}{>{\raggedright\arraybackslash}p{.16\linewidth}*{9}{>{\centering\arraybackslash}X}}
\toprule
\rowcolor{tablegroup}
Method & \multicolumn{3}{c}{\textbf{Training reward}} & \multicolumn{3}{c}{\textbf{Body error (m)}} & \multicolumn{3}{c}{\textbf{Anchor error (m)}} \\
\rowcolor{tablehead}
 & Initial & Final & AUC & Initial & Final & AUC & Initial & Final & AUC \\
\midrule
\rowcolor{basebg}
\makecell[l]{Fixed\\Reference} & \makecell[c]{0.29\\{\color{FixedGray}$\pm$ 0.03}} & \makecell[c]{19.30\\{\color{FixedGray}$\pm$ 1.29}} & \makecell[c]{15.96\\{\color{FixedGray}$\pm$ 0.55}} & \makecell[c]{1.032\\{\color{FixedGray}$\pm$ 0.021}} & \makecell[c]{0.189\\{\color{FixedGray}$\pm$ 0.008}} & \makecell[c]{0.205\\{\color{FixedGray}$\pm$ 0.002}} & \makecell[c]{0.080\\{\color{FixedGray}$\pm$ 0.015}} & \makecell[c]{0.351\\{\color{FixedGray}$\pm$ 0.015}} & \makecell[c]{0.388\\{\color{FixedGray}$\pm$ 0.004}} \\
\rowcolor{curriculumbg}
\makecell[l]{Failure\\Curriculum} & \makecell[c]{0.23\\{\color{FixedGray}$\pm$ 0.01}} & \makecell[c]{23.70\\{\color{FixedGray}$\pm$ 1.91}} & \makecell[c]{18.45\\{\color{FixedGray}$\pm$ 0.18}} & \makecell[c]{1.159\\{\color{FixedGray}$\pm$ 0.355}} & \makecell[c]{0.152\\{\color{FixedGray}$\pm$ 0.013}} & \makecell[c]{0.188\\{\color{FixedGray}$\pm$ 0.001}} & \makecell[c]{0.095\\{\color{FixedGray}$\pm$ 0.004}} & \makecell[c]{0.261\\{\color{FixedGray}$\pm$ 0.023}} & \makecell[c]{0.377\\{\color{FixedGray}$\pm$ 0.002}} \\
\rowcolor{taskbg}
\makecell[l]{Task-Aware\\Refinement} & \makecell[c]{1.12\\{\color{FixedGray}$\pm$ 0.03}} & \makecell[c]{56.33\\{\color{FixedGray}$\pm$ 1.07}} & \makecell[c]{45.89\\{\color{FixedGray}$\pm$ 0.58}} & \makecell[c]{1.159\\{\color{FixedGray}$\pm$ 0.355}} & \makecell[c]{0.163\\{\color{FixedGray}$\pm$ 0.007}} & \makecell[c]{0.197\\{\color{FixedGray}$\pm$ 0.005}} & \makecell[c]{0.095\\{\color{FixedGray}$\pm$ 0.004}} & \makecell[c]{0.332\\{\color{FixedGray}$\pm$ 0.036}} & \makecell[c]{0.436\\{\color{FixedGray}$\pm$ 0.011}} \\
\rowcolor{oursbg}
\mimicx & \makecell[c]{4.36\\{\color{FixedGray}$\pm$ 0.35}} & \makecell[c]{126.25\\{\color{FixedGray}$\pm$ 1.59}} & \makecell[c]{103.74\\{\color{FixedGray}$\pm$ 0.88}} & \makecell[c]{0.800\\{\color{FixedGray}$\pm$ 0.741}} & \makecell[c]{0.143\\{\color{FixedGray}$\pm$ 0.018}} & \makecell[c]{0.181\\{\color{FixedGray}$\pm$ 0.019}} & \makecell[c]{0.056\\{\color{FixedGray}$\pm$ 0.009}} & \makecell[c]{0.359\\{\color{FixedGray}$\pm$ 0.035}} & \makecell[c]{0.457\\{\color{FixedGray}$\pm$ 0.023}} \\
\bottomrule
\end{tabularx}
\par\smallskip\noindent\begin{minipage}{\linewidth}\papertablenote The logged aligned bodywise mean is distinct from the evaluation maximum: it averages aligned position errors across tracked bodies. Initial/final refer to updates 0/249. Normalized AUC uses trapezoidal integration divided by 249, preserving units. Shared-warmstart continuation: Section~\ref{sec:setup}.\par\smallskip\hrule height 0.4pt\relax\end{minipage}
\endgroup
\end{table*}

For \mimicx{}, the curves follow the continuation selected by the execution
gate. All curves share a 250-update continuation axis; the total candidate
budgets are specified in Appendix~\ref{sec:setup}. Jointly reporting reward,
body error, and episode length reveals how fidelity and sustained execution
develop under each supervision configuration.

\subsection{Evaluation milestones across all core tasks}

Table~\ref{tab:milestones} follows 48 selected policy trajectories through
240 milestone evaluations. Task-Aware Refinement provides the highest macro horizon AUC, while
MimicX provides the lowest final body error. This complements the dense
training logs with checkpoint-level evidence on actual execution.

% Table design v3
\begin{table*}[tp]
\centering
\caption{\textbf{Execution quality throughout policy continuation.} Every cell averages three continuation seeds. The normalized horizon AUC integrates five checkpoints from 20\% to 100\% of the recorded continuation; the checkpoint columns report horizon fractions, followed by final body error. The four-task, four-method, three-seed design gives 48 curves and 240 evaluations.}
\label{tab:milestones}
\begingroup\papertablestyle
\begin{tabularx}{\linewidth}{l>{\raggedright\arraybackslash}p{.24\linewidth}*{7}{>{\centering\arraybackslash}X}}\toprule\rowcolor{tablehead}
\textbf{Task} & \textbf{Method} & \textbf{AUC} & \textbf{20\%} & \textbf{40\%} & \textbf{60\%} & \textbf{80\%} & \textbf{100\%} & \textbf{Body (m)}\\
\midrule
Tennis & Fixed Reference & 0.580 & 0.463 & 0.674 & 0.725 & 0.462 & 0.460 & 0.239 \\
\rowcolor{curriculumbg}
 & Failure Curriculum & \textbf{0.963} & 0.705 & 1.000 & 1.000 & 1.000 & \textbf{1.000} & 0.183 \\
\rowcolor{taskbg}
 & Task-Aware Refinement & 0.941 & 0.845 & 0.861 & 1.000 & 1.000 & 0.959 & 0.183 \\
\rowcolor{oursbg} & \mimicx & 0.921 & 1.000 & 1.000 & 0.842 & 0.842 & \textbf{1.000} & \textbf{0.156} \\
\midrule
Football & Fixed Reference & 0.170 & 0.566 & 0.113 & 0.097 & 0.130 & 0.118 & 0.286 \\
\rowcolor{curriculumbg}
 & Failure Curriculum & 0.574 & 0.566 & 0.570 & 0.576 & 0.578 & 0.581 & \textbf{0.237} \\
\rowcolor{taskbg}
 & Task-Aware Refinement & 0.630 & 0.568 & 0.572 & 0.577 & 0.692 & 0.791 & 0.246 \\
\rowcolor{oursbg} & \mimicx & \textbf{0.744} & 0.569 & 0.599 & 0.739 & 0.886 & \textbf{0.939} & 0.244 \\
\midrule
Dance & Fixed Reference & 0.329 & 0.307 & 0.398 & 0.318 & 0.323 & 0.247 & 0.268 \\
\rowcolor{curriculumbg}
 & Failure Curriculum & 0.550 & 0.479 & 0.478 & 0.545 & 0.613 & 0.647 & 0.243 \\
\rowcolor{taskbg}
 & Task-Aware Refinement & \textbf{0.622} & 0.471 & 0.585 & 0.667 & 0.666 & \textbf{0.666} & 0.256 \\
\rowcolor{oursbg} & \mimicx & 0.435 & 0.435 & 0.435 & 0.435 & 0.435 & 0.435 & \textbf{0.241} \\
\midrule
Kung Fu & Fixed Reference & 0.114 & 0.176 & 0.167 & 0.065 & 0.090 & 0.092 & 0.249 \\
\rowcolor{curriculumbg}
 & Failure Curriculum & 0.257 & 0.187 & 0.214 & 0.270 & 0.308 & 0.288 & 0.194 \\
\rowcolor{taskbg}
 & Task-Aware Refinement & \textbf{0.344} & 0.238 & 0.297 & 0.425 & 0.340 & \textbf{0.392} & 0.211 \\
\rowcolor{oursbg} & \mimicx & 0.185 & 0.185 & 0.185 & 0.185 & 0.185 & 0.185 & \textbf{0.141} \\

\bottomrule
\end{tabularx}
\par\smallskip{\papertablenote Milestones use evaluation seed 1001. The three retained Kung Fu policies remain constant across their milestones. The continuation axis follows the selected trajectory; total candidate-search budgets are specified in Section~\ref{sec:setup}. Final core results separately aggregate three evaluation seeds.}
\par\smallskip\hrule height .4pt
\endgroup\end{table*}

\subsection{Temporal diagnostics connect fidelity to dynamics}

\begin{figure*}[tp]
\centering
\includegraphics[width=\linewidth]{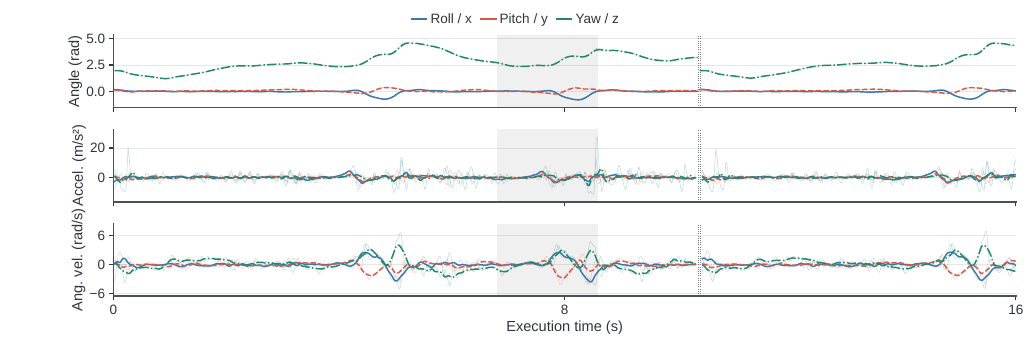}
\caption{\textbf{Aligned motion-state traces.} Orientation, acceleration and
angular velocity share a clock across 800 recorded Tennis states. Acceleration
uses position differences; angular velocity uses quaternion increments.
Derivative stencils crossing resets are excluded. Raw values and 11-sample
means distinguish short transients from the motion trend.}
\label{fig:state_traces}\label{fig:execution_dynamics}
\end{figure*}

\begin{table}[t]
\centering
\caption{\textbf{Complementary supervision components.} Four-video study, task-balanced means.}\label{tab:components-main-dense}
\begingroup\papertablestyle
\begin{tabularx}{\columnwidth}{>{\raggedright\arraybackslash}Xcccc}\toprule
\rowcolor{tablehead}
\textbf{Supervision} & \makecell{Body\\(m)} & \makecell{Anchor\\(m)} & \makecell{Linear\\(m/s)} & \makecell{Angular\\(rad/s)} \\
\midrule
\rowcolor{basebg}
Fixed Reference & 0.262 & 0.219 & 0.691 & 2.955 \\
\rowcolor{curriculumbg}
+ Curriculum & 0.215 & 0.135 & \textbf{0.451} & \textbf{1.935} \\
\rowcolor{taskbg}
+ Task-aware & 0.222 & \textbf{0.111} & 0.485 & 2.393 \\
\rowcolor{oursbg}
+ Verification & \textbf{0.196} & 0.147 & 0.481 & 2.238 \\
\bottomrule\end{tabularx}
\par\smallskip\begin{minipage}{\columnwidth}\papertablenote Ordered additions; final row is MimicX. Four tasks, three continuation seeds and three evaluation repeats. Lower is better; bold marks column best. Body: aligned tracked-body maximum averaged in time; anchor/velocities: world-frame errors. Protocol: Appendix~\ref{sec:setup}.\par\smallskip\hrule height .4pt\end{minipage}
\endgroup\end{table}

Figure~\ref{fig:execution_dynamics} aligns root orientation, acceleration,
and angular velocity on a common execution clock. These time-resolved signals
complement the episode-level gate statistics by locating rapid changes in
orientation and motion dynamics; restart-crossing derivatives are excluded.

\section{Gate and Rejection Case Studies}
\label{sec:gate-results}

Tables~\ref{tab:gate} and~\ref{tab:rejection} report two verification
protocols. In the original four-video study, the body guard uses the last
recorded aligned worst-body error and the execution ordering uses the earliest
termination across repeats. Continuation seeds start from the task-specific
warmstart; protected Dance and Kung Fu trials retain that checkpoint.
The separate Kung Fu check evaluates the two proposals and incumbent over
five repeats per policy and reports their worst failure horizons.

\section{Video and Motion-Reference Evaluation}
\label{sec:breadth-results}

\subsection{Additional videos and motion-reference inputs}

The additional video study extends the framework to Basketball, Soccer,
Michael Jackson Dance, and Uniandes Dance (Table~\ref{tab:additional}).
Soccer improves both body tracking and full-horizon success, while Michael
Jackson Dance and Uniandes Dance also reduce body error. Basketball retains
the same success rate with higher body error; Table~\ref{tab:additional}
reports the task-level values.

Fourteen motion-reference cases provide a complementary test after the
reference-construction interface (Table~\ref{tab:breadth}): eleven LaFAN1
references and three AMASS/ACCAD references, including two retargetings of
the same Form~1 source motion. They include
locomotion, dance, fighting and recovery motions. The results show sizeable
body-error reductions for several skills, including jumps, walking,
fight-and-sports and the ACCAD form sequence.
The motion-input cases test the policy-side loop with an existing reference;
eight video tasks exercise reconstruction through policy learning.

\subsection{Collision-scene video extension}
\label{sec:collision-extension}

Track Running, Stair Ascent, Forest Traversal, Platform Jump, and Parkour
pair video-derived motions with fixed task-equivalent collision scenes.
Each task uses a shared 3,000-update PPO checkpoint followed by equal
3,000-update Fixed Reference and Policy-in-the-Loop Refinement continuations,
with training seed 101 and 1,024 parallel environments. Evaluation seeds
1001, 2002, and 3003 yield 30 final rollouts, with terminations enabled.
The scene-study reset mixture assigns 50\% of resets to
the true start, 25\% to the failure window and 25\% to the full motion.
Execution and tracking guards govern
acceptance independently of reward.

\begin{figure*}[tp]
\centering
\begin{subfigure}[t]{\HalfPlate}
\includegraphics[width=\linewidth]{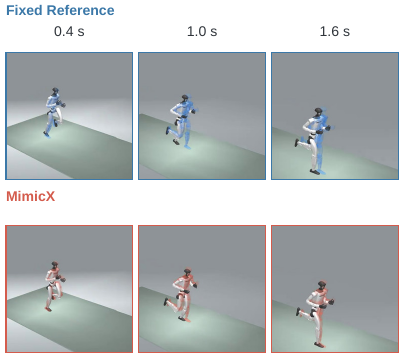}
\caption{Track Running}\end{subfigure}\hfill
\begin{subfigure}[t]{\HalfPlate}
\includegraphics[width=\linewidth]{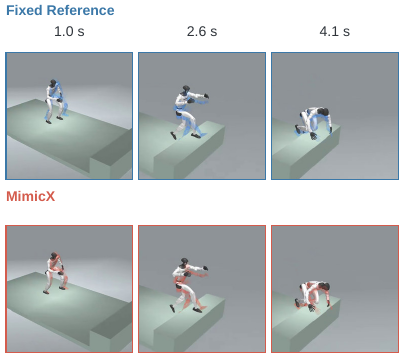}
\caption{Parkour transitions}\end{subfigure}
\caption{\textbf{Execution on additional task-equivalent collision scenes.}
Fixed Reference and MimicX at matched 20/50/80\% source phases, evaluation
seed 2002. Crops are identical between methods at each time. Track Running
completes all repeated evaluations with lower body and root errors after
refinement. Parkour panels depict the shared pre-failure interval; its
full-horizon outcome is reported in Table~\ref{tab:collision-extension}.}
\label{fig:collision-additional}
\end{figure*}

% Table design v3
% Generated by scripts/build_collision_extension.py from audited rollout CSVs.
\begin{table*}[tp]
\centering
\caption{Policy-in-the-loop refinement on five collision-scene video tasks. Equal-budget continuations share the same initial policy, scene, and reference. Completion rises from 9/15 to 12/15 full executions; both p95 errors improve on all five tasks. Steps report the minimum valid prefix divided by target length; reductions compare seed-mean errors over matched prefixes before failure.}
\label{tab:collision-extension}
\begingroup
\papertablestyle
\begin{tabularx}{\linewidth}{>{\raggedright\arraybackslash}p{.23\linewidth}*{6}{>{\centering\arraybackslash}X}}\toprule
\rowcolor{tablehead}  &  \multicolumn{2}{c}{Completion}  &  \multicolumn{2}{c}{Valid steps / $T$}  &  \multicolumn{2}{c}{p95 reduction} \\\rowcolor{tablehead}
\textbf{Task} & \textbf{Fixed} &  \textbf{MimicX}  & \textbf{Fixed} &  \textbf{MimicX}  & \textbf{Body} & \textbf{Root}\\
\midrule
Track Running & 3/3 & 3/3 & 97/97 & 97/97 & 53.80\% & 24.95\% \\
Stair Ascent & 3/3 & 3/3 & 449/449 & 449/449 & 11.86\% & 7.88\% \\
Forest Traversal & 0/3 & \textbf{3/3} & 39/404 & 404/404 & 79.15\% & \textbf{89.04\%} \\
Platform Jump & 3/3 & 3/3 & 58/58 & 58/58 & 5.78\% & 50.85\% \\
Parkour & 0/3 & 0/3 & 220/255 & 220/255 & 2.04\% & 0.59\% \\

\bottomrule
\end{tabularx}
\endgroup
\end{table*}

Strict completion requires the full scheduled interval without a flagged
termination. Table~\ref{tab:collision-extension} reports the minimum number
of valid pre-failure steps across evaluation repeats, divided by the target
length $T$; completed executions are recorded as $T/T$. Tracking reductions
use the common pre-first-failure prefix of each paired evaluation. We compute
the temporal p95 of maximum body-position error and world-root position error,
average each over the three repeats, and report the relative reduction.
Thus completion measures the full execution, while prefix errors compare
tracking at the same motion phases before either policy resets.

Refinement increases aggregate strict completion from 60\% to 80\%.
Forest Traversal changes from 0/3 to 3/3 completions while reducing body and
root prefix errors by 79.15\% and 89.04\%, respectively. Track Running,
Stair Ascent, and Platform Jump preserve 3/3 completion with improved tracking.
Parkour improves both prefix errors and reaches 220 of 255 steps in both
conditions. Figure~\ref{fig:collision-extension} complements these results with
matched close views of Stair Ascent under fixed and refined supervision.
Figure~\ref{fig:collision-additional} adds matched Track Running and Parkour
policy transitions.

\subsection{Developmental refinement cases}
\label{sec:historical}

The developmental runs illustrate how execution feedback targets different
stages of refinement (Table~\ref{tab:historical}). Initialization and root-focused
practice remove the observed Tennis terminations; subsequent articulation
refinement reduces wrist errors while preserving zero terminations.
Separate Football and Dance cases illustrate task-specific continuation and
local reference repair, respectively.

% Table design v3
\begin{table*}[tp]
\centering
\caption{\textbf{Execution-guided refinement in developmental case studies.} Source-verified before/after transitions show complementary improvements in stability, task-critical articulation and motion completion. These are selected single-rollout mechanism cases, separate from the controlled multi-seed results.}
\label{tab:historical}
\begingroup\papertablestyle
\begin{tabularx}{\linewidth}{>{\raggedright\arraybackslash}p{.23\linewidth}>{\raggedright\arraybackslash}Xrrl}\toprule\rowcolor{tablehead}
\textbf{Case} & \textbf{Measurement} & \textbf{Before} & \textbf{After} & \textbf{Change}\\
\midrule
Tennis initialization & Terminations in 800 steps & 7 & 0 & Stable rollout \\
Tennis articulation & Mean max-body error (m) & 0.1411 & 0.1202 & 14.81\% lower \\
Tennis articulation & Right-wrist error (m) & 0.0977 & 0.0497 & 49.13\% lower \\
Tennis articulation & Left-wrist error (m) & 0.0932 & 0.0453 & 51.39\% lower \\
\midrule
Football execution & Terminations in 455 steps & 1 & 0 & Complete interval \\
Dance execution & Terminations in 784 steps & 1 & 0 & Complete interval \\

\bottomrule
\end{tabularx}
\par\smallskip{\papertablenote The articulation pair preserves zero terminations. Dance combines local reference repair with continued policy learning; the rows characterize their respective recorded supervision configurations.}
\par\smallskip\hrule height .4pt
\endgroup\end{table*}

\section{HLoop Execution Analysis}
\label{sec:hloop-results}

\hloop{} reduces median wall time from 241.22\,s to 124.02\,s, a
1.945$\times$ speedup over sequential execution (Table~\ref{tab:hloop}),
while preserving the same selection outcome across repetitions.
Table~\ref{tab:hloop-raw} shows paired wall times for all three schedulers;
Figure~\ref{fig:current-efficiency}(c,d) separates per-job time from concurrency.

Table~\ref{tab:hloop-raw} gives each paired repetition. All execution modes use the same fixed-policy rollout--diagnosis--selection
workload and selection criterion. Bulk-synchronous
execution has a median time of 122.26\,s. Both parallel modes substantially
shorten the feedback workload relative to sequential processing.
Dependency-driven execution exposes each ready diagnosis as its rollout
finishes, providing a scheduling interface for the policy feedback stage.

\section{Extended Related Work}
\label{sec:extended-related-work}

\paragraph{Video-derived humanoid supervision.}
SMPL-X provides an expressive human representation~\citep{smplx2019}, GVHMR
recovers world-grounded motion from monocular video~\citep{gvhmr2024}, and GMR
shows that retargeting quality materially affects downstream humanoid tracking
\citep{gmr2025}.

\paragraph{Physics-based motion tracking.}
DeepMimic, AMP, PHC, and MaskedMimic establish powerful mechanisms for
reference imitation, reusable motion priors, recovery, and partial-motion
conditioning~\citep{deepmimic2018,amp2021,phc2023,maskedmimic2024}.
BeyondMimic prioritizes empirically difficult motion segments and extends a
tracker with guided diffusion~\citep{beyondmimic2025}; SONIC scales model, data,
and compute for generalist whole-body tracking~\citep{sonic2025}. In 2026,
OmniXtreme separates general motor learning from actuation-aware physical
refinement~\citep{omnixtreme2026}, and reference-and-goal multi-task RL uses
human motion as a prior for behavior that generalizes beyond the reference
distribution~\citep{referencegoal2026}.

\paragraph{Adaptive execution and long-horizon skills.}
KungfuBot adapts tracking tolerances and uses contact-aware
motions~\citep{kungfubot2025}; AdaMimic adapts timing and actions from sparse
keyframes~\citep{adamimic2025}; ASAP uses real trajectories to align simulation
and robot dynamics~\citep{asap2025}. Perceptive Humanoid Parkour composes
retargeted human skills into long-horizon trajectories before tracking and
distillation~\citep{parkour2026}, while HiWET couples world-frame end-effector
accuracy to whole-body stability~\citep{hiwet2026}.

\paragraph{Scene-aware video imitation.}
MeshMimic reconstructs human trajectories and scene geometry, then combines
kinematic consistency optimization with contact-invariant retargeting to
learn motion--terrain interactions from monocular video~\citep{meshmimic2026}.
MimicX complements geometric grounding by using policy execution feedback
to revise tracking objectives and practice distributions, with repeated
verification determining which continued policy is retained.

\section{Reproducibility}
\label{sec:reproducibility}

Versioned records support verified updates and resumption from retained
policies. Appendices~\ref{sec:current-protocol} and~\ref{sec:setup}
specify budgets and sampling units. Visuals distinguish input,
reconstruction, reference and execution; plots use logged trajectories.

\Needspace{20\baselineskip}
\section{Supervision Refinement with a World-Model Learner}
\label{sec:cross-learner}

We tested the supervision loop with a different policy optimizer using a
community implementation of SuperMimic~\citep{supermimic2026,supermimicreproduced}.
Policy updates backpropagate through short imagined rollouts in a shared,
pretrained dynamics model. Candidate selection and final evaluation use
physical simulation rather than imagined rollouts. All three arms share the
pretrained policy, world model and critic. The controls are continuous fixed
supervision and fixed supervision with verified candidate selection.

The three arms each use 3,000 collection steps across 1,024 environments.
The two candidate arms distribute these steps over two rounds of two
750-step continuations. Four motions and three training seeds per arm
produce 36 final policy outcomes. Each outcome is evaluated on five disjoint
held-out perturbations. Table~\ref{tab:cross-learner} averages the five repeats
within each training seed and then weights the three seeds equally.

All arms completed all 180 physical executions. The performance differences
are in tracking: feedback refinement lowered the mean world-body error for
Tennis and Kung Fu by 19.2\% and 23.6\% relative to fixed supervision.
Against fixed supervision with selection, the reductions were 10.3\% and
14.0\%, with paired gains on two and one of three seeds, respectively.

\begin{table}[H]
\centering
\caption{\textbf{Tracking with a shared world-model learner.} F: fixed supervision; S:
fixed supervision with verified selection; M: MimicX feedback refinement.
Each task and arm completed 15/15 physical rollouts.}
\label{tab:cross-learner}
\begingroup
\papertablestyle
\begin{tabularx}{\linewidth}{>{\raggedright\arraybackslash}X
  *{7}{>{\centering\arraybackslash}X}}
\toprule
\rowcolor{tablehead}
& \multicolumn{3}{c}{\textbf{World-body mean (cm)}}
& \multicolumn{3}{c}{\textbf{Joint MAE (rad)}} & \textbf{Paired wins} \\
\rowcolor{tablehead}
\textbf{Motion} & F & S & \cellcolor{oursbg}M & F & S & \cellcolor{oursbg}M
& M vs.\ F / S \\
\midrule
Tennis & 14.51 & 13.06 & \cellcolor{oursbg}\textbf{11.72}
& .08679 & \textbf{.08304} & \cellcolor{oursbg}.09735 & 3/3 / 2/3 \\
Kung Fu & 18.49 & 16.43 & \cellcolor{oursbg}\textbf{14.12}
& \textbf{.10496} & .11157 & \cellcolor{oursbg}.12078 & 2/3 / 1/3 \\
Football & \textbf{9.40} & 9.94 & \cellcolor{oursbg}18.10
& \textbf{.07871} & .08193 & \cellcolor{oursbg}.08302 & 0/3 / 0/3 \\
Dance & \textbf{9.77} & 10.98 & \cellcolor{oursbg}10.76
& .09499 & \textbf{.09406} & \cellcolor{oursbg}.09907 & 1/3 / 1/3 \\
\bottomrule
\end{tabularx}
\par\smallskip
\begin{minipage}{\linewidth}
\papertablenote Paired wins count training seeds with lower world-body mean.
Summed training/export worker time (F, S, M): 206.78, 212.03, 268.54 min.
This timing excludes physical evaluation and resource waiting.
\par\smallskip\hrule height .4pt
\end{minipage}
\endgroup
\end{table}

\end{document}